\documentclass{article}

\PassOptionsToPackage{numbers,compress}{natbib}
\usepackage[main,final]{neurips_2026}

\usepackage[utf8]{inputenc}
\usepackage[T1]{fontenc}
\usepackage{amsmath,amssymb,amsthm,mathtools}
\usepackage{booktabs}
\usepackage{nicefrac}
\usepackage{microtype}
\usepackage{xcolor}
\usepackage{enumitem}
\usepackage{graphicx}
\usepackage{float} 
\usepackage{subcaption}
\usepackage{bm}
\usepackage{wrapfig}
\usepackage{url}
\usepackage{hyperref}
\hypersetup{hidelinks} 

\newtheorem{theorem}{Theorem}[section]
\newtheorem{lemma}[theorem]{Lemma}
\newtheorem{proposition}[theorem]{Proposition}
\newtheorem{corollary}[theorem]{Corollary}
\theoremstyle{definition}
\newtheorem{definition}[theorem]{Definition}

\theoremstyle{remark}

\DeclareMathOperator{\arcosh}{arcosh}
\DeclareMathOperator{\grad}{grad}
\DeclareMathOperator{\argmin}{arg\,min}
\DeclareMathOperator{\argmax}{arg\,max}

\newcommand{\R}{\mathbb{R}}

\newcommand{\HnK}{\mathbb{H}^{n,K}}
\newcommand{\Hn}{\mathbb{H}^{n}}
\newcommand{\HH}{\mathbb{H}}
\newcommand{\Sn}{\mathbb{S}^{n-1}}
\newcommand{\buse}{b_{\bm{\xi}}}
\newcommand{\ip}[2]{\langle #1, #2 \rangle}
\newcommand{\ipL}[2]{\langle #1, #2 \rangle_{\mathcal{M}}}
\newcommand{\norm}[1]{\| #1 \|}
\newcommand{\abs}[1]{| #1 |}

\newcommand{\xhat}{\bm{\hat{\xi}}}
\newcommand{\wbar}{\bm{\bar{w}}}

\newcommand{\busek}{b_{\bm{\xi}}^K}

\title{The Implicit Bias of Hyperbolic Representation Learning for Multiclass Data:\\
A Busemann Risk Perspective}

\author{%
  Xingrun Li$^{1}$ \quad
  Sho Kuno$^{1}$ \quad
  Yusuke Mukuta$^{1,2}$ \quad
  Xin Yang$^{1}$ \quad
  Tatsuya Harada$^{1,2}$ \\
  $^{1}$The University of Tokyo \qquad
  $^{2}$RIKEN
}

\begin{document}

\maketitle

\vspace{-10pt}
\begin{abstract}
We study the implicit bias of Riemannian gradient flow for hyperbolic
multiclass classification with fixed class prototypes in hyperbolic
space~\(\mathbb H^n\). Our framework accommodates general
\textit{permutation invariant relative margin (PERM)}
losses~\cite{JMLR:v25:23-1599}, a class that includes cross entropy and
other standard multiclass losses. Our analysis is based on a decomposition:
at large radius, the distance to each prototype splits into a radial term
and a direction dependent term described by the Busemann function. This
yields two main results. First, we prove a radial dichotomy: the sign of a
drift coefficient~\(\mu\) determines whether the radius is pushed toward the
ideal boundary or back toward the interior; if the positive drift persists,
then \(r(t)=\tfrac12\log t+O(1)\), while persistent negative drift returns
the trajectory to the large radius threshold in finite time. Second, we show that large-radius directions can approach
critical points of the Busemann risk on
$\partial\mathbb H^n$, even during radial return. 
These results provide a rigorous large-radius perspective
on two phenomena we refer to as \textit{boundary saturation}
and \textit{near-boundary clustering} in hyperbolic
representation learning.

\end{abstract}

\vspace{-10pt}
\section{Introduction}

Over the past decade, deep learning has achieved remarkable success across diverse domains, from computer vision and natural language processing to scientific discovery~\cite{ImageNet,Devlin2019BERTPO,AlphaFold}. However, a central mystery underlying this success has captivated the research community: modern neural networks are typically overparameterized, and even in the absence of explicit regularization terms in the training objective, gradient based optimization nevertheless tends to find solutions that generalize well~\cite{Zhang2016UnderstandingDL,Belkin2018ReconcilingMM}. This phenomenon is known as \textbf{implicit regularization} (or implicit bias)~\cite{ImplicitRegularization}.
This perspective has led to a more precise understanding of several models. For example, in Euclidean settings, gradient descent on separable data with
cross-entropy loss converges directionally to max margin solutions
~\cite{GradientDescentonSeparableData}, while matrix factorized models
implicitly favor low rank structures~\cite{inmatrix,inmatrix2}.

In parallel, \textbf{hyperbolic representation learning} has become an important tool for modeling hierarchical, tree like, and long tailed data.
Due to its negative curvature and exponential volume growth, hyperbolic space can represent latent hierarchical structures with substantially lower distortion than Euclidean space~\cite{LowDistortion,Nickel2017PoincarEF,tradeoff}.
This geometric advantage has motivated a broad family of hyperbolic models, ranging from Poincar\'e embeddings to hyperbolic neural networks, graph neural networks, and Transformers~\cite{hnn,hgnn,Yang2024HypformerEE,survey1}, with applications across vision, graph learning, recommendation, reinforcement learning, and beyond~\cite{hyperbolicCV,hyperbolicRL,hnn+,hgformer,survey2}.
However, despite rapid progress in model design and empirical applications, the large radius training dynamics of hyperbolic models remain poorly understood.

Understanding implicit bias in hyperbolic space is not a direct extension of Euclidean theory.
Unlike Euclidean models, whose analyses of implicit bias often rely on norm direction decompositions and max margin characterizations~\cite{GradientDescentonSeparableData,GradientHomo,Ji2020DirectionalCA}, hyperbolic models are governed by Riemannian gradient dynamics on a noncompact curved manifold~\cite{RiemannianAdaptiveOptimizationMethods,SGDonRiemannian}.
This geometry permits representations to reach large radii while their angular directions become governed by the ideal boundary, even when radial motion is ultimately bounded.
Moreover, hyperbolic architectures commonly rely on geodesic distances, exponential and logarithmic maps, M\"obius operations, and other nonlinear geometric constructions~\cite{hnn,Lorentzmodelof}, making their losses fundamentally different from convex surrogates composed with linear predictors.
As a result, standard Euclidean tools do not directly apply.

Beyond the theoretical challenge, this gap leaves concrete empirical phenomena unexplained. For example, hyperbolic classifiers are often associated with two behaviors in the large-radius regime: \emph{\textbf{boundary saturation}}, where embeddings reach large hyperbolic radii near the Poincar\'e ball boundary and numerical instabilities may arise~\cite{NumericalStability,clipped}; and \emph{\textbf{near-boundary clustering}}, where class discrimination at large radii depends less on radial positions and more on boundary directions~\cite{fewshot,metriclearning}.

Motivated by these observations, we study the large radius dynamics of
hyperbolic representations under Riemannian gradient flow for multiclass
classification. Our central questions are: \emph{when do hyperbolic
representations undergo outward excursions, and what governs their
angular directions during radial return?}
Our analysis starts from a simple geometric observation: along a
large radius ray in \(\Hn\), the distance to any fixed prototype decomposes
into a common radial term and a direction dependent Busemann term. The
common radial term cancels in relative margin based multiclass losses,
or \emph{PERM losses}~\cite{JMLR:v25:23-1599}, leaving the leading order
classification behavior governed by Busemann functions on the ideal
boundary \(\partial\Hn\).
This decomposition yields a clean asymptotic picture. Persistent positive
drift conditionally yields the radial law
\(
r(t)=\frac12\log t+O(1),
\)
while angular alignment can reverse outward motion and induce a
slow return toward the interior. In the large radius regime, the angular
dynamics are governed approximately by a \emph{Busemann loss} on the ideal
boundary. Thus, at large radius, classification is governed primarily by
the boundary direction rather than the full finite radius hyperbolic position.
\vspace{-8pt}
\paragraph{Contributions.}
To the best of our knowledge, we provide the first rigorous asymptotic
analysis of Riemannian gradient flow for hyperbolic multiclass
classification under general PERM losses~\citep{JMLR:v25:23-1599},
including cross-entropy, multiclass exponential, and Gamma-Phi losses.
Our main contributions are:
\begin{itemize}
    \item We develop a Busemann asymptotic framework for hyperbolic
multiclass classification with PERM losses, showing that, at large radius,
each sample loss is governed to leading order by a Busemann loss on the
ideal boundary.
\item We establish a radial dichotomy governed by the drift
coefficient $\mu$: persistent positive drift conditionally yields
$r(t)=\frac12\log t+O(1)$, while angular alignment can reverse
outward motion and induce slow radial return, suggesting a possible
mechanism for boundary saturation.
\item We show escaping trajectories with ideal prototypes converge
to global minimizers of Busemann risk, while finite prototype
trajectories approach local minima arbitrarily closely during
three-phase return, suggesting a possible mechanism for
near-boundary clustering.
    \item We characterize the asymptotic classifier and its decision boundaries
on the ideal boundary. 
\end{itemize}
\vspace{-8pt}
\section{Related Work}
\vspace{-4pt}
Due to the limited space in the main text, we provide a brief overview
of related literature here and defer a more detailed discussion to
Appendix~[\ref{sec:related}].
For \textbf{implicit bias}:
Gradient based optimization has been shown to select max margin
solutions in separable linear classification
~\cite{GradientDescentonSeparableData}, with extensions to nonseparable
data~\cite{nonseparable}, homogeneous networks~\cite{GradientHomo}, deep directional
convergence~\cite{Ji2020DirectionalCA}, optimization geometry dependent
biases~\cite{geometry}, and low rank biases in
matrix factorization~\cite{inmatrix}. For \textbf{multiclass classification}, the PERM framework extends implicit bias theory beyond binary losses~\cite{implicitmulti}, while recent work also studies algorithm dependent margin biases in homogeneous networks~\cite{flavor}. These Euclidean results characterize margin or complexity biases, but not radial drift or directional alignment
in curved noncompact spaces.
For \textbf{hyperbolic representation learning}: Hyperbolic learning has developed a broad range of models, including
Poincar\'e embeddings~\cite{Nickel2017PoincarEF},
hyperbolic neural networks~\cite{hnn},
hyperbolic neural network variants~\cite{hnn+},
hyperbolic graph neural networks~\cite{hgnn},
and hyperbolic Transformers~\cite{Yang2024HypformerEE};
see~\cite{survey1} for a survey.
Empirical studies have also reported boundary related phenomena,
including numerical instability near the Poincar\'e boundary~\cite{NumericalStability},
the effectiveness of clipping~\cite{clipped},
common radius concentration~\cite{fewshot},
and metric learning induced class separation~\cite{metriclearning}.
These works highlight hyperbolic geometry, but focus mainly on model design
and empirical phenomena rather than the large radius dynamics governing
boundary approach and return.

\section{Hyperbolic Space}

In this section, we provide a brief review of the essential concepts and a theorem about hyperbolic geometry that underlie our proposed method.

\begin{figure}[htbp]
    \centering
    \begin{subfigure}{0.45\textwidth}
        \centering
        \includegraphics[width=\linewidth]{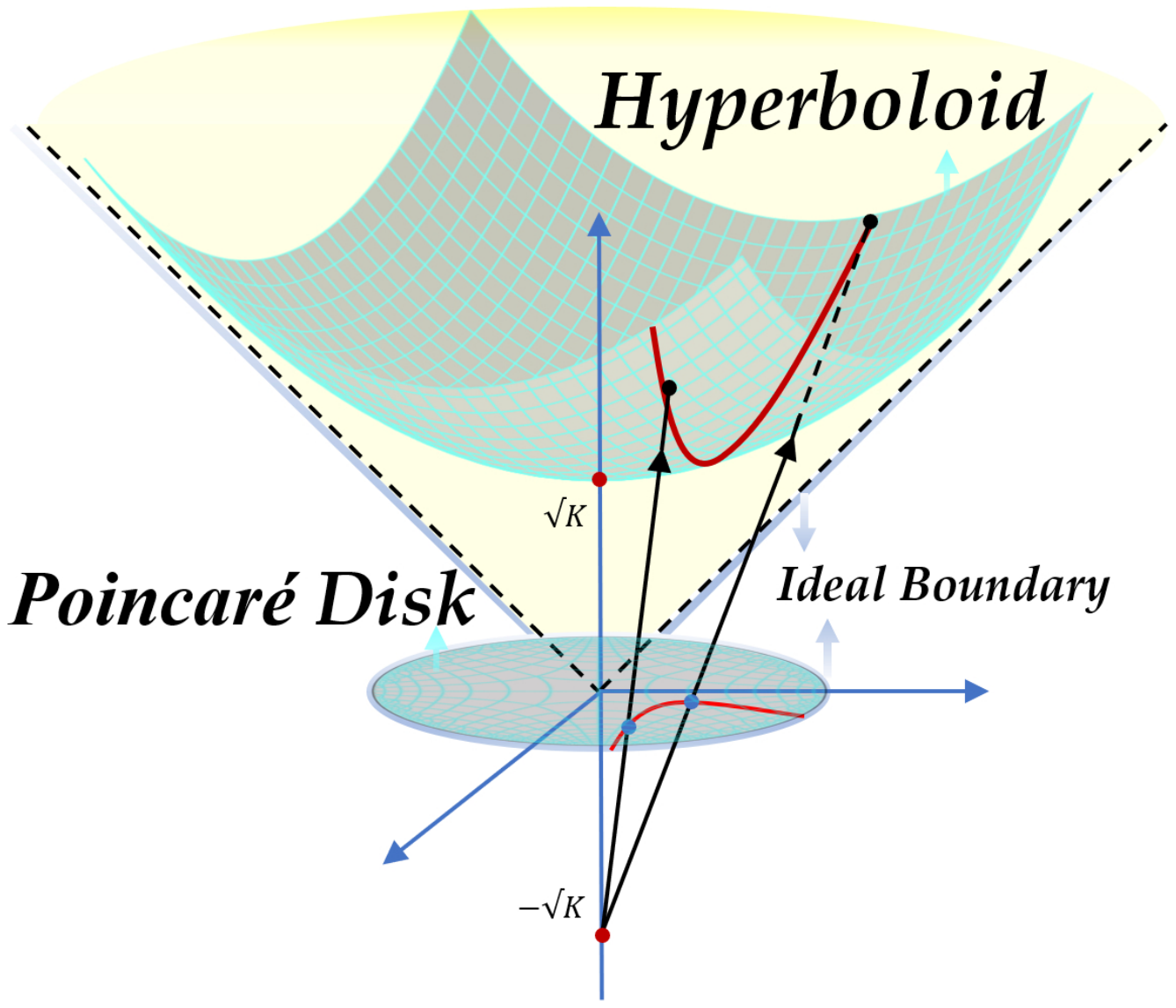}
        \caption{Hyperboloid model and  Poincar\'e Disk}
        \label{fig:a}
    \end{subfigure}
    \hfill
    \begin{subfigure}{0.45\textwidth}
        \centering
        \includegraphics[width=\linewidth]{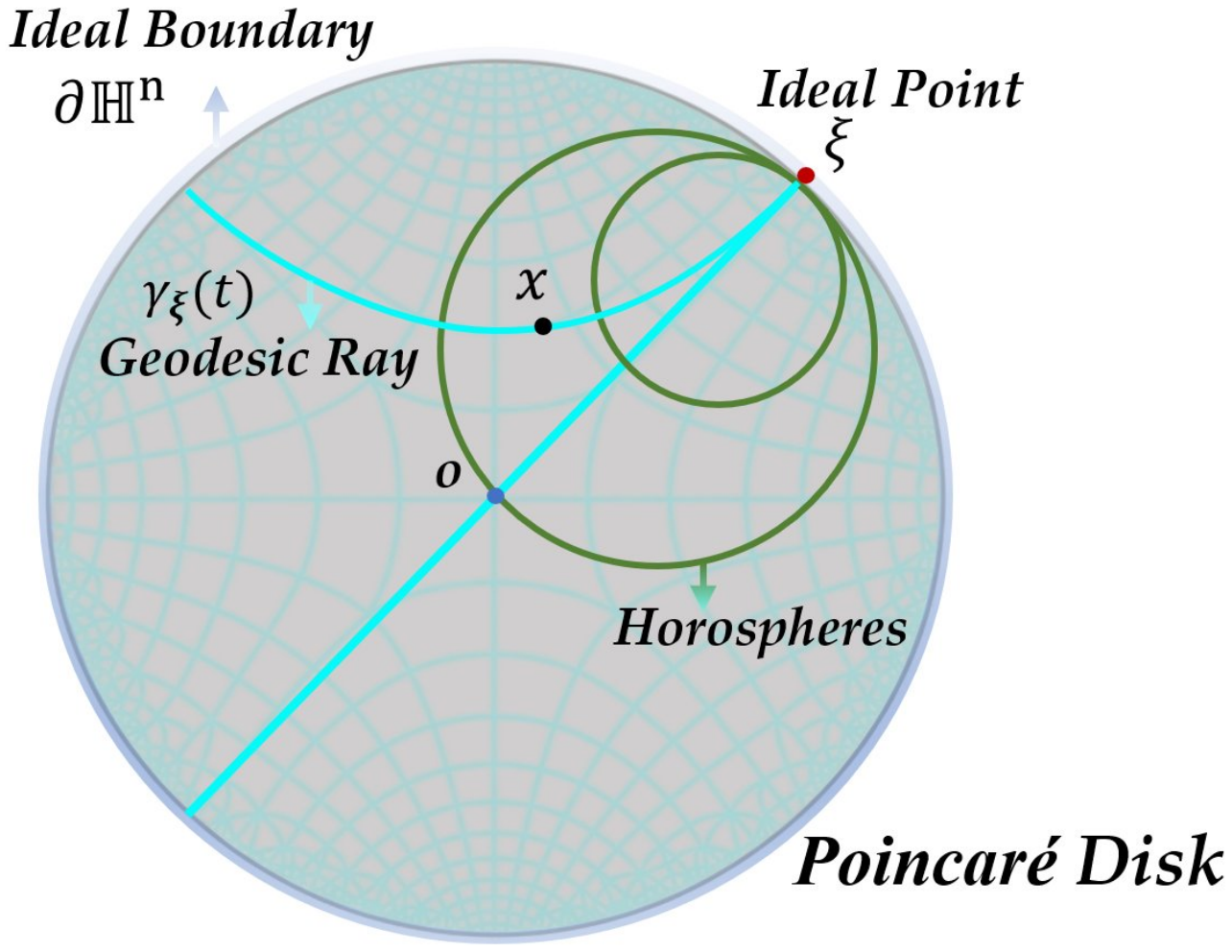}
        \caption{Geodesic rays, ideal boundary, and horospheres}
        \label{fig:b}
    \end{subfigure}
    
    \caption{Geometric illustration of hyperbolic representations and their boundary
asymptotics.}
    \label{fig:main}
\end{figure}

The \emph{Minkowski inner product}, defined for vectors $\mathbf{x} = (x_0, x_1, \ldots, x_n) = (x_0, \bar{\mathbf{x}}),\mathbf{y}=(y_0, y_1, \ldots, y_n) = (y_0, \bar{\mathbf{y}}) \in\mathbb{R}\times\mathbb{R}^{n}$, is given by:
$\langle \mathbf{x}, \mathbf{y}\rangle_{\mathcal{M}} = -x_0 y_0 + \bar{\mathbf{x}}^{\top}\bar{\mathbf{y}} = \mathbf{x}^{\top}\bm{\eta}\mathbf{y}$
where the \emph{Minkowski metric tensor} of hyperboloid model is \(\bm{\eta}=\mathrm{diag}(-1, 1, \ldots, 1)\).
The \(n\)-dimensional \emph{hyperbolic space} (hyperboloid model) is defined as:
\begin{equation}
\mathbb{H}^{n,K}=\{\mathbf{x}\in\mathbb{R}^{n+1}: \langle\mathbf{x},\mathbf{x}\rangle_{\mathcal{M}}=-K,\, x_0>0\},
\end{equation}
where \(K>0\) is the \emph{Minkowski constant}, which determines the constant negative \emph{curvature} \(\kappa=-1/K\).
The \emph{hyperbolic distance} of \(\mathbf{x},~\mathbf{y}\in\mathbb{H}^{n,K}\) is given by
\begin{equation}
d_{\mathbb{H}}^K(\mathbf{x},\mathbf{y})=\sqrt{K} \operatorname{arcosh}\left(-\langle \mathbf{x}, \mathbf{y}\rangle_{\mathcal{M}} /K\right).
\end{equation}

The \emph{north pole} is defined as:
$
\mathbf{o}=(\sqrt{K},0,\dots,0).
$
Every $\mathbf{x}=(x_0,\bar{\mathbf{x}})\in \HnK \setminus \{\mathbf{o}\}$ admits a unique polar decomposition of the form
\begin{equation}\label{eq:polar}
\mathbf{x}(r) = \left( \sqrt{K}\cosh \frac{r}{\sqrt{K}},  \sqrt{K}\sinh \frac{r}{\sqrt{K}}\xhat \right),
\end{equation}
where radius $r=d_{\mathbb{H}}^K(\mathbf{o},\mathbf{x})$ and direction $\xhat=\bar{\mathbf{x}}/\|\bar{\mathbf{x}}\|\in\Sn$.
The \emph{ideal boundary} is defined as the set
$\partial\HnK = \{ \boldsymbol{\xi} = (1, \xhat) : \xhat \in \Sn \}$,
where each element $\boldsymbol{\xi}$ represents an equivalence class
of geodesic rays.
\textbf{As $r \to \infty$, the ray $\mathbf{x}(r)$ in \eqref{eq:polar} converges to the \emph{ideal point} $\boldsymbol{\xi} = (1, \xhat)$ along the direction $\xhat$.}
For an ideal point \(\boldsymbol{\xi} = (1, \xhat) \in \partial\HnK\)
and \(\mathbf{w} \in \HnK\), the \emph{Busemann function} is defined as:
\begin{equation}\label{eq:busemann}
  \busek(\mathbf{w})
    := \lim_{t \to \infty}\bigl[d_{\mathbb{H}}^K(\mathbf{w},
       \gamma_{\boldsymbol{\xi}}(t)) - t\bigr]
     = \sqrt{K}\,\log\!\Bigl(
         \frac{-\ipL{\mathbf{w}}{\boldsymbol{\xi}}}{\sqrt{K}}\Bigr),
\end{equation}
where
\(
\gamma_{\boldsymbol{\xi}}(t)
  = \bigl(\sqrt{K}\cosh \tfrac{t}{\sqrt{K}},\;
     \sqrt{K}\sinh \tfrac{t}{\sqrt{K}}\;\xhat\bigr)
\)
is the unit-speed geodesic ray toward~\(\boldsymbol{\xi}\).

Intuitively, \(b_{\boldsymbol{\xi}}^K(\mathbf w)\) measures the asymptotic position of \(\mathbf w\) along the direction toward the ideal point \(\boldsymbol{\xi}\). The level sets
$\{\mathbf{w} : \busek(\mathbf{w}) = c\}$ are \emph{horospheres}
centred at~$\boldsymbol{\xi}$.

\paragraph{Clarification.}
In fact, the hyperboloid model and the Poincar\'e ball (disk) model are equivalent representations of the same hyperbolic space~\cite{Anderson1999HyperbolicG}, while the gyrovector formalism provides an algebraic language for operations on the hyperbolic space~\cite{Ungar2009gyro}. Since the quantities studied in this paper are intrinsic geometric objects, their asymptotic behavior is independent of the chosen representation. We therefore carry out the analysis in the hyperboloid model, where these quantities admit the simpler closed form expressions. And we present visualizations in the Poincar\'e ball model for better geometric intuition.
\textbf{Also, without loss of generality, we set Minkowski constant \(K=1\) throughout the main development to simplify the presentation, and thus omit \(K\) from the notation.} The general case follows by a straightforward rescaling argument, see Appendix [\ref{sec:discussion}].

\section{Hyperbolic Multi-classification Loss }
\subsection{Classification Setup}

We consider a $C$-class classification problem in $\Hn$ with $C \ge 2$. 
For each class $c \in [C] := \{1,\dots,C\}$, let \(\mathbf{w}_c \in \Hn\) 
denote its \emph{prototype}, a fixed representative point of class $c$ 
in the embedding space. We are given $N$ training samples 
$\{(\mathbf{z}_i,y_i)\}_{i=1}^N$, where \(\mathbf{z}_i\) is the input 
and $y_i \in [C]$ is its class label. An encoder 
$f_\theta$ maps each input \(\mathbf{z}_i\) to a point 
\(\mathbf{x}_i = f_\theta(\mathbf{z}_i) \in \Hn\), which we write 
in polar form as
\(
\mathbf{x}_i = (\cosh r_i,\; \sinh r_i\,\xhat_i),
\) and we omit parameters $\theta$ throughout the paper for brevity.

We define class score functions $v_c(\mathbf{x})$, $c \in [C]$, 
producing a score vector 
$\mathbf{v}(\mathbf{x}) = (v_1(\mathbf{x}),\dots,v_C(\mathbf{x}))^{\top} : 
\Hn \to \mathbb{R}^C$. Let 
$\boldsymbol{\phi}(\mathbf{v})= (\phi_1(\mathbf{v}),\dots,\phi_C(\mathbf{v}))^{\top} 
: \mathbb{R}^C \to \mathbb{R}^C$ be a vector-valued multiclass loss function, 
where $\phi_c(\mathbf{v})$ denotes the $c$-th component of 
$\boldsymbol{\phi}(\mathbf{v})$, corresponding to the loss when the true 
class is $c$. The loss for sample \((\mathbf{x}_i, y_i)\) takes the 
\emph{class-score form}:
\begin{equation}\label{eq:loss_general}
\ell(\mathbf{x}_i, y_i)
  = \phi_{y_i}\bigl(\mathbf{v}(\mathbf{x}_i)\bigr)
  = \phi_{y_i}\bigl(v_1(\mathbf{x}_i), \ldots, v_C(\mathbf{x}_i)\bigr).
\end{equation}
The total loss is \(L(\mathbf{X}) = \frac{1}{N}\sum_{i=1}^N \ell(\mathbf{x}_i, y_i)\), where
\(\mathbf{X}:=(\mathbf{x}_1,\ldots,\mathbf{x}_N)\)
collects all training embeddings.
The choice of class score function \(v_c(\mathbf{x})\) determines how
geometric information is used for classification. We adopt the distance to
prototype score, a common and geometrically direct choice based on proximity
in the intrinsic hyperbolic metric. Other hyperbolic scoring rules, such as
hyperplane based scores, hypersphere based scores, and Lorentzian inner
products, are also closely related to the underlying hyperbolic distance
geometry. Thus, throughout this work, we set
\(
  v_c(\mathbf{x}) = -d_{\HH}(\mathbf{x}, \mathbf{w}_c)
    = -\arcosh\!\bigl(-\langle \mathbf{x}, \mathbf{w}_c
       \rangle_{\mathcal{M}}\bigr).
\)
We define \(\mathbf{W}:=[\mathbf{w}_1,\ldots,\mathbf{w}_C]\)
to collect the prototypes column wise. The resulting
\emph{hyperbolic multiclass classification objective} is:
\begin{equation}
\vspace{-5pt}
  L_{\mathbf{W}}(\mathbf{X})
    = \frac{1}{N}\sum_{i=1}^{N}
        \phi_{y_i}\!\Bigl(
          -d_{\HH}(\mathbf{x}_i, \mathbf{w}_1),\;\ldots,\;
          -d_{\HH}(\mathbf{x}_i, \mathbf{w}_C)
        \Bigr).
\end{equation}

\subsection{PERM Losses Frame}
We employ the PERM framework of \citet{JMLR:v25:23-1599}.
We first recall some  central concepts.
Let $\mathfrak{S}_C$ denote the symmetric group on $[C]$, which is  the
group of all $C!$ permutations of $\{1,\dots,C\}$.
For $\sigma \in \mathfrak{S}_C$, the associated permutation matrix
$S_\sigma \in \{0,1\}^{C \times C}$ has entries
$(S_\sigma)_{ij} = \mathbf{1}[j = \sigma(i)]$,
so that $(S_\sigma \mathbf{v})_i = v_{\sigma(i)}$.
The matrix $D = [I_{C-1} \mid -\mathbf{1}_{C-1}] \in \R^{(C-1)\times C}$
extracts differences with respect to the last component:
\(
  (D\mathbf{v})_k = v_k - v_C,
  ~ k = 1,\dots,C{-}1.
\)
For each class $y \in [C]$, the matrix label code
$\Upsilon_y \in \{0,\pm1\}^{(C-1)\times(C-1)}$ is the unique
invertible matrix satisfying
\vspace{-2pt}
\begin{equation}\label{eq:label_code}
  \Upsilon_y D\mathbf{v}
    = \bigl(v_y - v_{j_1},\;\dots,\;v_y - v_{j_{C-1}}\bigr)^\top,
  \quad
  \{j_1,\dots,j_{C-1}\} = [C]\setminus\{y\}.
\end{equation}
That is, $\Upsilon_y$ re-expresses the
base differences $\{v_k - v_C\}$ as the class $y$ margins
$\{v_y - v_j\}_{j \neq y}$.
 
\begin{definition}[PERM loss]
A multiclass loss
$\boldsymbol{\phi}(\mathbf{v})= (\phi_1(\mathbf{v}),\dots,\phi_C(\mathbf{v}))^{\top} : \mathbb{R}^C \to \mathbb{R}^C$
is a \emph{PERM loss} if it is:
(1)~\emph{Permutation Equivariant}:
    $\boldsymbol{\phi}(S_\sigma \mathbf{v})
       = S_\sigma\,\boldsymbol{\phi}(\mathbf{v})$
    for all $\sigma \in \mathfrak{S}_C$; and
(2)~\emph{Relative Margin based}:
    $\phi_y(\mathbf{v})$ depends on $\mathbf{v}$ only through the
    margins $\{v_y - v_j\}_{j \neq y}$.
\end{definition}
 
By Theorem 2.5 in \cite{JMLR:v25:23-1599}, a PERM loss is uniquely
characterised by a symmetric function
$\psi : \R^{C-1}\!\to\R$, called its \emph{template function}, via
\begin{equation}\label{eq:template}
  \phi_y(\mathbf{v}) = \psi(\Upsilon_y D\mathbf{v}).
\end{equation}
Because $\psi$ is symmetric, the particular ordering of the margins in
\eqref{eq:label_code} does not affect $\phi_y$.
 
\begin{definition}[Totally regular PERM loss]
\label{def:totally_regular}
Let $\psi : \mathbb{R}^{C-1} \to \mathbb{R}$ be a PERM template. 
For each $n \in \{2,\dots,C-1\}$, define its $n$-ary truncation 
$\psi^{(n)} : \mathbb{R}^n \to \mathbb{R}$ by restricting $\psi$ 
to any $n$ coordinates.
The loss is \emph{totally regular} if, for all such $n$, 
$\psi^{(n)}$ is nonnegative, twice differentiable, strictly convex, 
semi-coercive, and satisfies $\nabla \psi^{(n)} \prec 0$ everywhere. 
\end{definition}
 Then we introduce our first assumption:
 \begin{enumerate}[label=(A\arabic*), ref=A\arabic*, series=assumptions]
  \item \label{assump:perm}
  \textbf{\textit{The multiclass loss function $\boldsymbol{\phi}$ is a totally regular PERM loss with template function $\psi$.}}
\end{enumerate}

(\ref{assump:perm})  are satisfied by cross-entropy ($\psi(\mathbf{u})
   = \log\!\bigl(1 + \textstyle\sum_{j=1}^{C-1} e^{-u_j}\bigr)$), multiclass exponential loss, and all Gamma-Phi losses with exponentially decaying $\phi$ (Proposition~4.6 in~\cite{JMLR:v25:23-1599}).

\section{Asymptotic Expansion and Busemann Risk}\label{Section4}

\subsection{The Radial Busemann Expansion}

The cornerstone of our analysis is the following large radius decomposition
of hyperbolic distance into a radial term and a direction dependent
Busemann term.

\begin{lemma}[Radial Busemann expansion]
\label{lem:dist_busemann}
For $\mathbf{x} = (\cosh r,\, \sinh r\,\xhat) \in \Hn 
\setminus \{\mathbf{o}\}$ and $\mathbf{w} = (w_0, w_1, \ldots, w_n)=(w_0, \wbar) 
\in\Hn$, let $\boldsymbol{\xi} = (1, \xhat) \in 
\partial\Hn$. There exists a constant 
$R_0(\mathbf{w}) > 0$ such that for all 
$r \ge R_0$,
\begin{equation}\label{eq:dist_bus_exp}
  d_{\HH}(\mathbf{x}, \mathbf{w}) 
    = r + \buse(\mathbf{w}) + O(e^{-2r}),
\end{equation}
where $\buse(\mathbf{w}) = \log(w_0 - \xhat \cdot \wbar)$ 
is the Busemann function.
\end{lemma}
\begin{proof}[Proof sketch]
Using the hyperboloid distance formula and the expansions of 
\(\cosh r\) and \(\sinh r\), we obtain
\(
  -\ipL{\mathbf{x}}{\mathbf{w}}
  =
  \frac{e^r}{2}
  (w_0-\xhat\cdot\wbar)
  \bigl(1+O(e^{-2r})\bigr).
\)
Since \(w_0-\xhat\cdot\wbar>0\), applying
\(\arcosh(u)=\log(2u)+O(u^{-2})\) gives
\(
  d_{\HH}(\mathbf{x},\mathbf{w})
  =
  r+\log(w_0-\xhat\cdot\wbar)+O(e^{-2r}),
\)
which is exactly the claimed expansion. The full and detailed proof is in 
Appendix [\ref{proof:4.1}].
\end{proof}

This decomposition allows us to expand the nonconvex loss \(\ell\) and
separate the radial and angular mechanisms behind the later convergence
analysis.

\subsection{Loss Decomposition}

Using Lemma~\ref{lem:dist_busemann}, we derive an asymptotic expansion of
the loss for each sample in~\eqref{eq:loss_general}, consisting of a leading
Busemann term and an exponentially small correction. Recall from
Assumption~(\ref{assump:perm}) that \(\bm{\phi}\) is a totally regular PERM
loss with template \(\psi\). Hence
\(\ell_i(\mathbf{x}_i,y_i)=\psi(\Upsilon_{y_i}D\mathbf v(\mathbf{x}_i))\),
where \(\mathbf v(\mathbf{x}_i)=(v_1(\mathbf{x}_i),\ldots,v_C(\mathbf{x}_i))
=(-d_{\HH}(\mathbf{x}_i,\mathbf w_1),\ldots,-d_{\HH}(\mathbf{x}_i,\mathbf w_C))\).

\begin{definition}[Busemann loss, risk  and correction coefficient]
\label{def:bus_loss}
For a direction $\bm{\xi}= (1, \xhat) \in \partial\Hn$ with label $y \in [C]$, and prototype $\{\mathbf{w}_c\}_{c=1}^C$, define:
\begin{align}
\mathcal{L}^{\mathrm{BS}}(\bm{\xi}, y) 
\;&:=\; \psi\bigl(\Upsilon_{y} D\,\mathbf{b}(\bm{\xi})\bigr), 
\qquad \quad\text{where } \mathbf{b}(\bm{\xi}) \;:=\; (-b_{\bm{\xi}}(\mathbf{w}_1), \ldots, -b_{\bm{\xi}}(\mathbf{w}_C))^\top,\\ 
\boldsymbol{\alpha}(\bm{\xi}) 
\;&:=\; \bigl(\alpha_{\bm{\xi}}(\mathbf{w}_1), \ldots, \alpha_{\bm{\xi}}(\mathbf{w}_C)\bigr)^\top,\qquad \text{where }\alpha_{\bm{\xi}}(\mathbf{w}_c)
  \;:=\;  \frac{\norm{\wbar_c^\perp}^2}
             {\ipL{\mathbf{w}_c}{\bm{\xi}}^2},
\end{align}
here $\wbar_c^\perp := \wbar_c - (\xhat \cdot \wbar_c)\,\xhat$ is 
the component of $\wbar_c$ orthogonal to the boundary
direction~$\xhat$.

We call $\mathcal{L}^{\mathrm{BS}} : \partial\Hn \times [C] \to \R$ the \emph{Busemann loss} function, and $\bm{\alpha}: \partial \Hn \to \R^C$  the \emph{correction coefficient}.
Moreover, the \emph{Busemann risk} is
\begin{equation}
\mathcal{R}_{\mathbf{W}}^{\mathrm{BS}}(\bm{\xi}_1,\ldots,\bm{\xi}_N)
  \;:=\; \frac{1}{N}\sum_{i=1}^N
         \mathcal{L}^{\mathrm{BS}}(\bm{\xi}_i, y_i)
\end{equation}
\end{definition}

\begin{proposition}[Loss expansion]
\label{prop:loss_expansion}
Under Assumption~(\ref{assump:perm}), there exists a 
constant $R_0(\mathbf{W}) > 0$ such that for each 
$i \in [N]$ and all $r_i \ge R_0$,
\begin{equation}\label{eq:loss_expansion1}
\ell_i(\mathbf{x}_i, y_i) 
\;=\; \mathcal{L}^{\mathrm{BS}}(\bm{\xi}_i, y_i) 
\;+\; \mu(\bm{\xi}_i, y_i)\,e^{-2r_i} 
\;+\; O(e^{-4r_i}),
\end{equation}
where
\begin{equation}\label{eq:beta_def_v22}
\mu(\bm{\xi}, y) \;:=\; -\bigl\langle 
\nabla\psi\bigl(\Upsilon_{y} D\,\mathbf{b}(\bm{\xi})\bigr),\; 
\Upsilon_{y} D\,\boldsymbol{\alpha}(\bm{\xi})\bigr\rangle,
\end{equation}
and we call $\mu : \partial\Hn \times [C] \to \mathbb{R}$ 
the radial drift coefficient.
\end{proposition}

\begin{proof}[Proof sketch]
By extending Lemma~\ref{lem:dist_busemann},
$v_c(\mathbf{x}_i) = -d_{\HH}(\mathbf{x}_i, \mathbf{w}_c) = -r_i - b_{\bm{\xi}_i}(\mathbf{w}_c) - \alpha_{\bm{\xi}_i}(\mathbf{w}_c)\,e^{-2r_i} + O(e^{-4r_i})$.
Since $\ell_i$ is a PERM loss by Assumption~(\ref{assump:perm}), it depends on
$\mathbf v$ only through the relative margin vector
$\Upsilon_{y_i}D\mathbf v$.
The common radial term $-r_i$ is annihilated by $D$, so:
\begin{equation}\label{eq:expansion}
\Upsilon_{y_i} D\, \mathbf{v}(\mathbf{x}_i) 
\;=\; \Upsilon_{y_i} D\, \mathbf{b}(\bm{\xi}_i) \;-\; \Upsilon_{y_i} D\,\boldsymbol{\alpha}(\bm{\xi}_i)\,e^{-2r_i} \;+\; O(e^{-4r_i}).
\end{equation}
Then applying Taylor expansion of $\psi$ around
$\Upsilon_{y_i}D\mathbf b(\bm{\xi}_i)$ to the relative margin expansion
in~\eqref{eq:expansion}, and using the definition of
$\mu(\bm{\xi}_i,y_i)$ in~\eqref{eq:beta_def_v22}, yields the loss
expansion~\eqref{eq:loss_expansion1}.

The full and detailed proof is in 
Appendix [\ref{proof:4.3}].
\end{proof}
The common radial term cancels in the relative margin vector
\(\Upsilon_{y_i}D\,\mathbf v\) by the relative margin structure in
Assumption~(\ref{assump:perm}). Therefore, the leading order loss depends only
on the boundary direction \(\bm{\xi}_i\), not on the radius \(r_i\). This yields
the radial angular decoupling in leading order, which drives our asymptotic analysis. Summing over \(i\), the empirical loss admits
$
\label{eq:total_loss_expansion}
L_{\mathbf W}(\mathbf X)
=
\frac1N\sum_{i=1}^N
\left[
  \mathcal L^{\mathrm{BS}}(\bm{\xi}_i,y_i)
  +
  \mu(\bm{\xi}_i,y_i)e^{-2r_i}
  +
  O(e^{-4r_i})
\right].
$
In particular, as \(\min_i r_i\to\infty\),
\begin{equation}
L_{\mathbf W}(\mathbf X)
=
\mathcal R_{\mathbf W}^{\mathrm{BS}}
(\bm{\xi}_1,\ldots,\bm{\xi}_N)
+
o(1).   
\end{equation}

This decoupling motivates our dynamical analysis below, where we study how
the radial and angular components of the embeddings evolve under gradient
flow. Throughout the analysis, we fix the class prototypes
\(\{\mathbf w_c\}_{c=1}^C\) and study the gradient flow dynamics of the
embeddings \(\{\mathbf x_i\}_{i=1}^N\) directly on \(\Hn\). The setting of
fixed prototypes isolates the geometric mechanism behind boundary
divergence in distance based hyperbolic classifiers, is natural in view of
the symmetry of hyperbolic distance, and is consistent with methods using
fixed prototypes in prototypical learning~\cite{fix3,fix1,fix2}. It also
parallels classical implicit bias analyses of linear classification, while
shifting the object of study from classifier weights to the radial and
angular dynamics of hyperbolic embeddings.

\section{Radial and Angular Behavior}
\label{section 5}
We study the Riemannian gradient flow on $(\Hn)^N$:
\begin{equation}\label{eq:grad_flow}
\dot{\mathbf{x}}_i(t) \;=\; -\grad_{\Hn}\, \ell_i(\mathbf{x}_i(t), y_i), \qquad i = 1, \ldots, N,
\end{equation}
with initial conditions $\mathbf{x}_i(0) \in \Hn$, where for any smooth function $f$ on a Riemannian manifold $(M, g)$ with metric matrix $G$, the Riemannian gradient is defined as
\(
\grad_{M} f = G^{-1}\,\nabla f,
\)
with $\nabla f$ denoting the vector of partial derivatives in local coordinates.
This is the continuous time analogue of Riemannian gradient descent $\mathbf{x} \leftarrow \exp_{\mathbf{x}}(-\eta\, \grad_{M} f)$.
Recall that each point on the hyperboloid is $\mathbf{x}_i = (\cosh r_i,\, \sinh r_i\,\xhat_i)$ with $r_i > 0$ and $\xhat_i \in \Sn$. In these polar coordinates $(r_i, \xhat_i)$, the induced metric takes the block diagonal form $ds^2 = dr^2 + \sinh^2(r)\,d\omega^2$, where $d\omega^2$ is the round metric on $\Sn$.
Since the metric is block diagonal between $r$ and $\xhat$, the gradient flow decouples into a radial and an angular equation:
\vspace{-4pt}
\begin{equation}\label{eq:polar_flow}
\dot{r}_i \;=\; -\frac{\partial \ell_i}{\partial r_i}, 
\qquad
\dot{\xhat}_i \;=\; -\frac{1}{\sinh^2(r_i)}\,\grad_{\Sn}\,\ell_i\bigl|_{\xhat_i},
\end{equation}
where $\grad_{\Sn} \ell_i = \nabla_{\xhat} \ell_i - (\nabla_{\xhat} \ell_i \cdot \xhat_i)\,\xhat_i$ is the Euclidean gradient with respect to $\xhat_i \in \R^n$ projected onto the tangent plane of the sphere at $\xhat_i$.
We first analyze a single trajectory that fix \(i\in[N]\) and suppress the subscript \(i\).
The extension to all samples follows by applying the same argument independently to each \(i\).

\subsection{Radial Dichotomy}

\begin{lemma}[Radial gradient]
\label{lem:radial_gradient}
Under Assumptions~(\ref{assump:perm}), there exists a 
constant $R_0(\mathbf{W}) > 0$ such that for all $r \ge R_0$,
\begin{equation}\label{eq:radial_grad}
\frac{\partial \ell}{\partial r} 
\;=\; -2\mu(\bm{\xi}, y)\,e^{-2r} 
\;+\; O(e^{-4r}),
\end{equation}
where $\mu$ is the radial drift coefficient defined in~\eqref{eq:beta_def_v22}.
\end{lemma}
\begin{proof}
By Lemma~\ref{lem:dist_busemann}, the expansion in
Proposition~\ref{prop:loss_expansion} holds with a $C^1$ remainder
$\rho$ satisfying $\rho=O(e^{-4r}),\partial_r\rho=O(e^{-4r})$ uniformly.
Differentiating
$\ell=\mathcal{L}^{\mathrm{BS}}(\bm{\xi},y)+
\mu(\bm{\xi},y)e^{-2r}+\rho$
in $r$, and noting that $\mathcal{L}^{\mathrm{BS}}$ and $\mu$ are
$r$-independent, yields
$\partial_r\ell=-2\mu(\bm{\xi},y)e^{-2r}+O(e^{-4r})$.
\end{proof}

\begin{theorem}[Radial dichotomy]
\label{thm:radial}
Under Assumption~(\ref{assump:perm}), for any drift lower bound
\(\underline{\mu}>0\), there exists
\(R_0=R_0(\mathbf{W},\underline{\mu})>0\) such that:
\begin{enumerate}[label=\textup{(\alph*)}]
\item \textbf{Boundary divergence.}\;
  If there exist \(T_2>T_1\geq0\) such that \(r(T_1)\ge R_0\) and
  \(\mu(\bm{\xi}(t),y)\ge\underline{\mu}\) on the time interval
  \(I=[T_1,T_2]\), then
  \(
    r(t)\ge\tfrac{1}{2}\log\!\bigl(
      e^{2r(T_1)}+2\underline{\mu}(t-T_1)
    \bigr)
  \)
  for all \(t\in I\).
  Moreover, if \(r(T_1)\ge R_0\) and
  \(\mu(\bm{\xi}(t),y)\ge\underline{\mu}\) for all \(t\ge T_1\), then
  \begin{equation}\label{eq:radial_rate_div}
    r(t)=\tfrac{1}{2}\log t+O(1),
    \qquad t\to\infty.
  \end{equation}

\item \textbf{Boundary repulsion.}\;
  If there exist \(T_2>T_1\geq0\) such that
  \(\mu(\bm{\xi}(t),y)\le-\underline{\mu}\) and \(r(t)\geq R_0\) for all
  \(t\in I=[T_1,T_2]\), then
  \(
    r(t)\le\tfrac{1}{2}\log\!\bigl(
      e^{2r(T_1)}-2\underline{\mu}(t-T_1)
    \bigr)
  \)
  for all \(t\in I\).
  Moreover, if \(r(T_1)\ge R_0\) and
  \(\mu(\bm{\xi}(t),y)\le-\underline{\mu}\)
  whenever \(r(t)\ge R_0\), then the first hitting time
  \(T^*:=\inf\{t\ge T_1:\, r(t)\le R_0\}\) satisfies
  \begin{equation}
    T^*\le T_1+
    \frac{e^{2r(T_1)}-e^{2R_0}}{2\underline{\mu}}
    <\infty.
  \end{equation}
\end{enumerate}
\end{theorem}
\noindent\textit{Proof sketch.}
Firstly, by Proposition~\ref{prop:loss_expansion}, differentiating the loss expansion with
respect to~$r$ in geodesic
polar coordinates gives
$\dot{r} = 2\mu(t)\,e^{-2r} + O(e^{-4r})$.
For $r \geq R_0$ sufficiently large, the remainder
is absorbed, yielding
$\underline{\mu}\,e^{-2r}
\leq \dot{r}
\leq 3\bar{\mu}\,e^{-2r}$
when $\mu \geq \underline{\mu} > 0$.
\begin{wrapfigure}{r}{0.4\textwidth}
  \centering
  \vspace{-10pt}
  \includegraphics[width=0.35\textwidth]{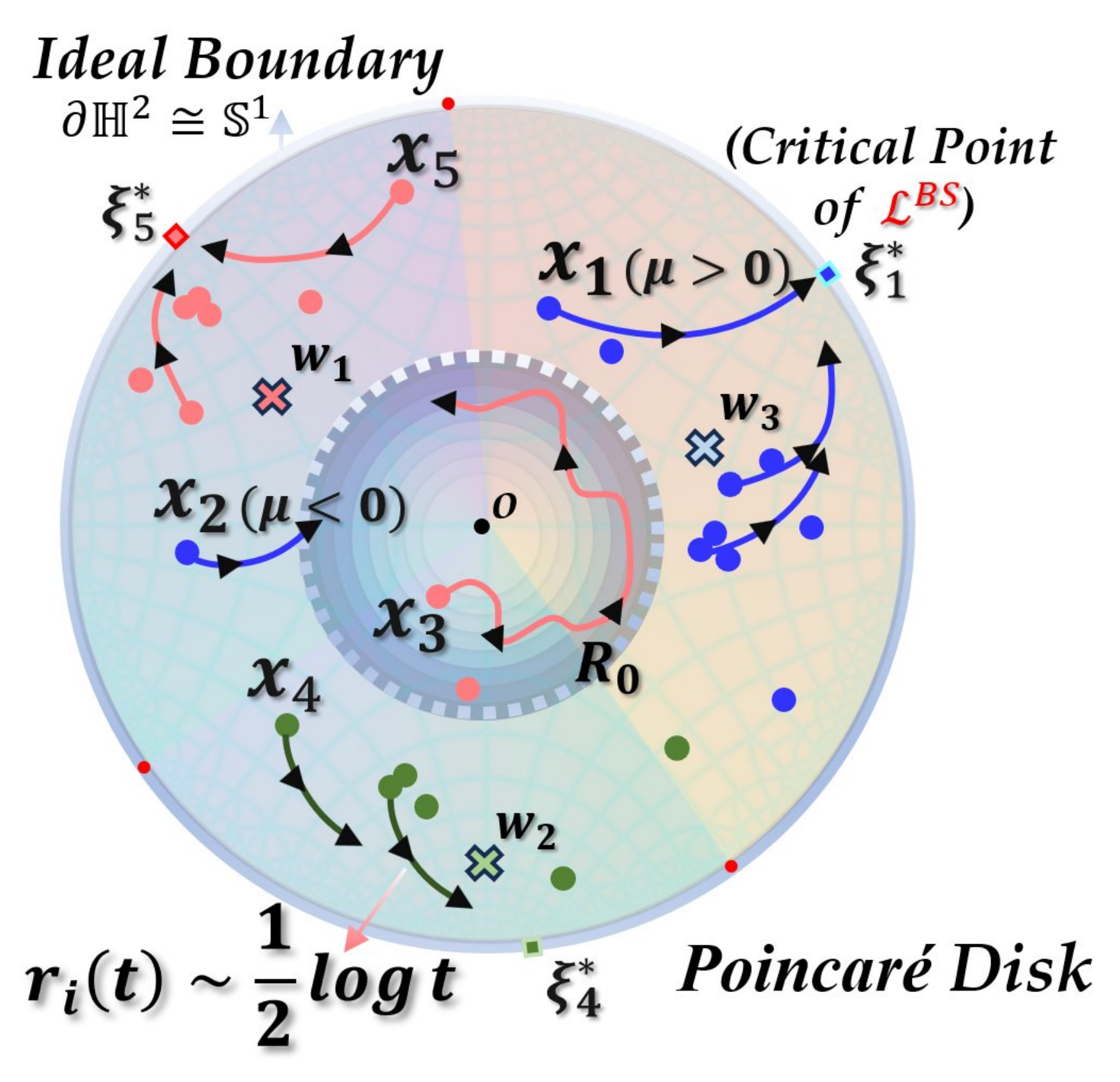}
  \caption{Radial dichotomy and boundary convergence.}
  \label{fig:asy}
  \vspace{-40pt}
\end{wrapfigure}
(i) Part~(a):
From the radial ODE, for $r \geq R_0$ large enough,
$\dot{r} \geq \underline{\mu}\,e^{-2r}$.
Multiplying both sides by $e^{2r}$ and recognizing
$\frac{d}{dt}[e^{2r}] = 2e^{2r}\dot{r}$, we obtain
$\frac{d}{dt}[e^{2r}] \geq 2\underline{\mu}$.
Integrating from $T_1$ to $t$ gives
$e^{2r(t)} \geq 2\underline{\mu}(t-T_1) + e^{2r(T_1)}$,
hence $r(t) \geq \frac{1}{2}\log t + O(1)$ if the condition holds for all $t\geq T_1$.
The matching upper bound follows from
$\dot{r} \leq 3\bar{\mu}\,e^{-2r}$ by the same
argument.

(ii) Part~(b): When $\mu \leq -\underline{\mu} < 0$ and
$r \geq R_0$, the same technique gives
$\frac{d}{dt}[e^{2r}] \leq -2\underline{\mu}$.
Integrating from $T_1$ to $t$ yields
$e^{2r(t)} \leq e^{2r(T_1)} - 2\underline{\mu}(t-T_1)$,
which reaches $e^{2R_0}$ no later than
$T^* = T_1+(e^{2r(T_1)} - e^{2R_0})/(2\underline{\mu})$.

The full and detailed proof is in 
Appendix [\ref{proof: 5.2}].
\hfill $\square$

Figure~\ref{fig:asy} illustrates the radial dichotomy and angular convergence
on the Poincar\'e disk. The dashed circle denotes the fixed threshold radius
\(R_0\), outside which the large radius expansion is valid. The crosses
\(w_1,w_2,w_3\) are fixed prototypes, and the colored points are embeddings.
For \(x_1\), positive drift \(\mu(\bm{\xi}_1,y_1)>0\) drives the trajectory
toward the ideal boundary, while for \(x_2\), negative drift
\(\mu(\bm{\xi}_2,y_2)<0\) pushes the trajectory back toward the interior,
as characterized by Theorem~\ref{thm:radial}. The point \(x_3\) starts inside
\(R_0\), where the large radius expansion does not control the dynamics and
the motion is instead governed by the full nonconvex loss \(\ell\).

\subsection{Angular Convergence}
\label{sec:angular}
We first consider the large-radius dynamics with fixed finite
prototypes. Although trajectories remain bounded, they may undergo
outward radial excursions before returning toward the interior.
During this process, angular alignment can reverse the radial drift,
while the directions approach local minima of
$\mathcal L^{\mathrm{BS}}$ to arbitrary accuracy at sufficiently large
initial radii. To formalize this behavior, we consider the following
configuration:

\begin{enumerate}[label=(A\arabic*), ref=A\arabic*, resume=assumptions]
  \item \label{assump:noncollinear}
  \textbf{\textit{Noncollinear prototypes: The fixed prototypes
  $\mathbf W=\{w_c\}_{c=1}^{C}\subset\Hn$, $C\ge3$,
  are pairwise distinct, and there exists no hyperbolic geodesic
  containing all $w_c$, $c\in[C]$.}}
\end{enumerate}

For sufficiently large initial radii, trajectories with initially
positive drift can undergo long time outward motion, allowing angular
alignment to reverse the sign of $\mu$ and induce radial return.
During this slow return, the angular deviation from a Busemann local
minimum decreases exponentially with the return radius $R$.
We formalize this behavior through the following aligned region.

\begin{definition}[Sufficiently aligned region]
\label{def:aligned}
Let $\bm\xi^*\in\partial\Hn$ be a nondegenerate local
minimum of $\mathcal L^{\mathrm{BS}}(\cdot,y)$.
Define
\begin{equation}
\mathcal A(\bm\xi^*)
:=
\left\{
\bm\xi\in\partial\Hn:
d_{\Sn}(\bm\xi,\bm\xi^*)\le\theta_*
\right\},
\end{equation}
where $d_{\Sn}(\bm\xi,\bm\xi^*)
=\arccos(\hat{\bm\xi}^{\top}\hat{\bm\xi}^*)$
is the spherical geodesic distance, and $\theta_*>0$
is sufficiently small that $\mu(\bm\xi,y)<0$
and the Busemann Hessian is positive definite
throughout $\mathcal A(\bm\xi^*)$.
\end{definition}

The angular width of this region depends on the prototype
geometry. For any number of distinct prototype directions with
common radius $\rho$, the maximal alignment angle near
the Busemann minimum approaching the true class direction
scales as $\Theta(e^{-2\rho})$ under cross-entropy.

\refstepcounter{theorem}
\label{thm:three_phase}
\noindent{\textbf{Theorem~\thetheorem}
(Three-phase Busemann dynamics).}
\textit{Under Assumptions~(\ref{assump:perm}) and
~(\ref{assump:noncollinear}), with $\psi$ real analytic,
suppose $\bm\xi(T_0)$ lies in the Busemann gradient-flow
basin of a nondegenerate local minimum $\bm\xi^*$,
with $\mu(\bm\xi(T_0),y)>0$.
For a fixed initial angular direction, for any $R\ge R_0$, if
$r(T_0)\ge 2R$, then there exist finite times
$T_0<T_1<T_2<T_3$ satisfying:}

\begin{wrapfigure}{r}{0.19\textwidth}
  \centering
  \vspace{-15pt}
  \includegraphics[width=0.19\textwidth]{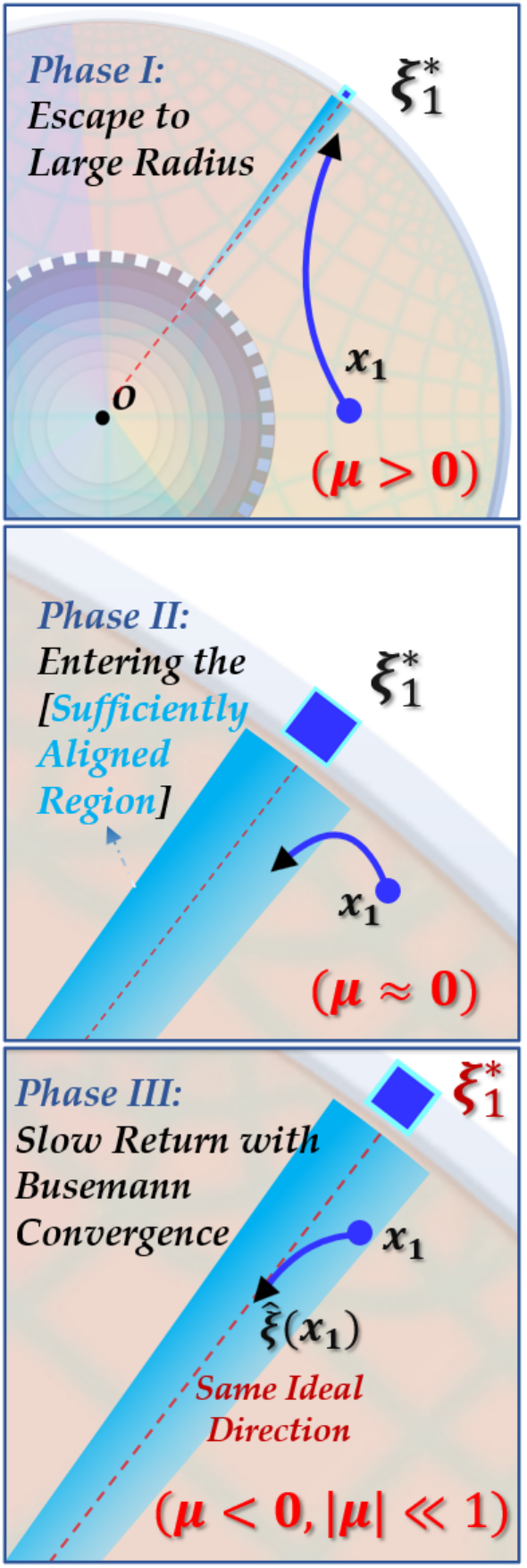}
  \caption{Three-phase Busemann dynamics.}
  \label{fig:three_phase}
  \vspace{-28pt}
\end{wrapfigure}

\noindent\textbf{(a) Large-radius excursion.}\;
\textit{The trajectory initially moves outward,
$\dot r(T_0)>0$, while
$\|\dot{\xhat}(t)\|=O(e^{-2r(t)})$
in the large-radius regime.}

\smallskip
\noindent\textbf{(b) Angular alignment and drift reversal.}\;
\textit{At $T_1$, $\mu(\bm\xi(T_1),y)=0$,
and $\mu(\bm\xi(t),y)<0$ for $t\in(T_1,T_3]$.
At $T_2$, the trajectory first enters the
sufficiently aligned region
$\mathcal A(\bm\xi^*)$.}

\smallskip
\noindent\textbf{(c) Slow radial return.}\;
\textit{For $t\in[T_2,T_3]$, $\dot r(t)<0$
and $|\dot r(t)|=O(e^{-2r(t)})$, with
$r(T_3)=R$.}
\textit{For any $\varepsilon>0$,
the returning trajectory reaches
\begin{equation}
    d_{\Sn}(\bm\xi(T_3),\bm\xi^*)\le\varepsilon
\end{equation}
in finite time by choosing $R$ sufficiently large.}

\noindent\textit{Proof sketch.}
For sufficiently large $r$, the radial and angular dynamics satisfy
$\dot r=2\mu(\bm\xi,y)e^{-2r}+O(e^{-4r})$ and
$\dot{\xhat}=-(\sinh^{-2}r)
(\grad_{\Sn}\mathcal L^{\mathrm{BS}}+O(e^{-2r}))$.
At a Busemann critical point,
$\mu(\bm\xi^*,y)=
-\sum_{c\ne y}g_c
\|\grad_{\Sn}b_{\bm\xi^*}(w_c)-
\grad_{\Sn}b_{\bm\xi^*}(w_y)\|^2<0$,
where $g_c>0$, by noncollinearity.
Under the time change $ds=dt/\sinh^2r$, the angular
dynamics is an $O(e^{-2r})$ perturbation of Busemann
gradient flow. Since the initial direction lies in
the basin of $\bm\xi^*$, finite-time approximation
ensures entry into $\mathcal A(\bm\xi^*)$ while
$r\ge 2R-O(1)$, forcing a preceding drift reversal.
The positive definite Busemann Hessian then gives
$\frac{d}{ds}d_{\Sn}(\bm\xi,\bm\xi^*)
\le-\lambda d_{\Sn}(\bm\xi,\bm\xi^*)+Ce^{-2r}$,
ensuring continued angular alignment and negative
radial drift. Since $|dr/ds|$ is bounded, the slow
return to $R$ provides at least $O(R)$ rescaled time
for angular contraction. Consequently,
$d_{\Sn}(\bm\xi(T_3),\bm\xi^*)
\le C'e^{-\kappa R}$, which is at most
$\varepsilon$ for sufficiently large $R$.

The full and detailed proof is in Appendix~[\ref{proof: 6.5}].
\hfill $\square$

Theorem~\ref{thm:three_phase} establishes that trajectories
with finite prototypes can undergo radial return while
approaching Busemann local minima to arbitrary accuracy.
We now consider ideal boundary prototypes, for which
permanent radial escape becomes possible. In this setting,
Busemann scores replace hyperbolic distance scores,
while preserving the PERM loss structure.

\begin{enumerate}[label=(A\arabic*), ref=A\arabic*, resume=assumptions]
  \item \label{assump:ideal-prototypes}
  \textbf{\textit{Ideal boundary prototypes: The fixed class
  prototypes $w_c^\infty=(1,\hat w_c)\in\partial\Hn$
  are pairwise distinct.}}
\end{enumerate}

In the Busemann risk, we use naturally continuous extension
$b_\xi(w_c^\infty)=\log(-\langle w_c^\infty,\xi\rangle_M)$
for $\xi\ne w_c^\infty$.

\begin{theorem}[Angular convergence to Busemann minimizers in the escaping regime]
\label{thm:angular}
Under
Assumptions~(\ref{assump:perm}) and (\ref{assump:ideal-prototypes}),
the Riemannian gradient flow satisfies
\begin{equation}
  r(t)\longrightarrow\infty,
  \qquad
  \bm{\xi}(t)\longrightarrow\bm{\xi}^*
  ,
  \quad \text{where}\quad
  \bm{\xi}^*\in
  \arg\min_{\bm{\xi}\in\partial\Hn}
  \mathcal{L}^{\mathrm{BS}}(\bm{\xi},y).
\end{equation}
 If Assumption~
(\ref{assump:ideal-prototypes}) holds for all $i\in[N]$,
then $(\bm{\xi}_1^*,\ldots,\bm{\xi}_N^*)$ is a
\textbf{global minimizer} of the Busemann risk
$\mathcal{R}^{\mathrm{BS}}_{\mathbf{W}^\infty}
  =\frac{1}{N}\sum_{i=1}^{N}
  \mathcal{L}^{\mathrm{BS}}(\bm{\xi}_i,y_i)$
on~$(\partial\Hn)^N$.
\end{theorem}

\noindent\textit{Proof sketch.}
Rotate coordinates so $w_y^\infty=(1,1,0,\ldots,0)$ and write
$x=(x_0,x_1,u)$, $s=x_0-x_1>0$, and
$a_c=\hat w_{c,\perp}/(1-\hat w_{c,1})$.
The relative margins are
$m_c=\log[(1+\|u-sa_c\|^2)/(s^2(1+\|a_c\|^2))]$.
With $q_c=-\partial_c\psi(m)>0$, the Riemannian flow yields
$\dot s=-2s\sum_{c\ne y}q_c/(1+\|u-sa_c\|^2)<0$;
$du/ds$ is a weighted average of the fixed $a_c$.
The margins remain bounded below, ensuring global existence,
while bounded $du/ds$ implies $u(t)$ converges.
If $s(t)$ approached a positive limit, the trajectory
would remain compact, contradicting a uniformly negative
$\dot s$. Thus $s(t)\to0$. Since
$x_0+x_1=(1+\|u\|^2)/s$, we obtain
$r(t)\to\infty$ and $\bm\xi(t)\to w_y^\infty$.
All relative Busemann margins then diverge to $+\infty$.
Nonnegativity and coordinatewise monotonicity of $\psi$
yield a global Busemann minimum under continuous extension.
Samplewise separability establishes the joint risk result.

The full and detailed proof is in Appendix~[\ref{proof: 5.6}]. \hfill $\square$

\subsection{Asymptotic Decision Boundary}

\begin{theorem}[Asymptotic classifier]
\label{thm:boundary}
Under Assumption~(\ref{assump:perm}), let
\(\{\mathbf w_c\}_{c=1}^C\) be fixed prototypes,
either all in \(\Hn\) or all in \(\partial\Hn\),
with distance or Busemann scores, respectively.
For a test point
\(\mathbf x=(\cosh r,\sinh r\,\xhat)\) with
\(\bm\xi=(1,\xhat)\in\partial\Hn\), define
\(
  \hat y(\mathbf x)
  :=
  \argmin_{y\in[C]}
  \psi\!\left(\Upsilon_yD\,\mathbf v(\mathbf x)\right).
\)
If \(\argmin_y\mathcal L^{\mathrm{BS}}(\bm\xi,y)\) is unique, then for all
sufficiently large \(r\),
\(
  \hat y(\mathbf x)
  =
  \argmin_{y\in[C]}\mathcal L^{\mathrm{BS}}(\bm\xi,y)
  =
  \argmin_{c\in[C]} b_{\bm\xi}(\mathbf w_c).
\)
For ideal prototypes, this equality holds for all \(r>0\).

Moreover, the decision boundary between
classes $c$ and $c'$ on $\partial\Hn\cong\Sn$ is
\begin{equation}\label{eq:boundary_condition}
\begin{aligned}
\mathcal B_{c,c'}=\bigl\{\xhat\in\Sn:\;&
\xhat^\top(\wbar_c-\wbar_{c'})
=w_{c,0}-w_{c',0},\quad b_{\bm\xi}(\mathbf w_c)
\le b_{\bm\xi}(\mathbf w_k),
\quad \forall k\in[C]\bigr\},
\end{aligned}
\end{equation}
which is a subset of the intersection of $\Sn$
with an affine hyperplane in $\R^n$.
For ideal prototypes, $w_{c,0}=1$, so the
corresponding spherical sections are great subspheres.
\end{theorem}

\noindent\textit{Proof sketch.}
By the radial Busemann expansion~\eqref{eq:dist_bus_exp},
the common term \(-r\) cancels in the relative margin vector. Hence, uniformly over \(y\),
\(\psi(\Upsilon_yD\mathbf v(\mathbf x))
=\mathcal L^{\mathrm{BS}}(\bm\xi,y)+O(e^{-2r})\).
\begin{wrapfigure}{r}{0.4\textwidth}
  \centering
  \vspace{-10pt}
\includegraphics[width=0.35\textwidth]{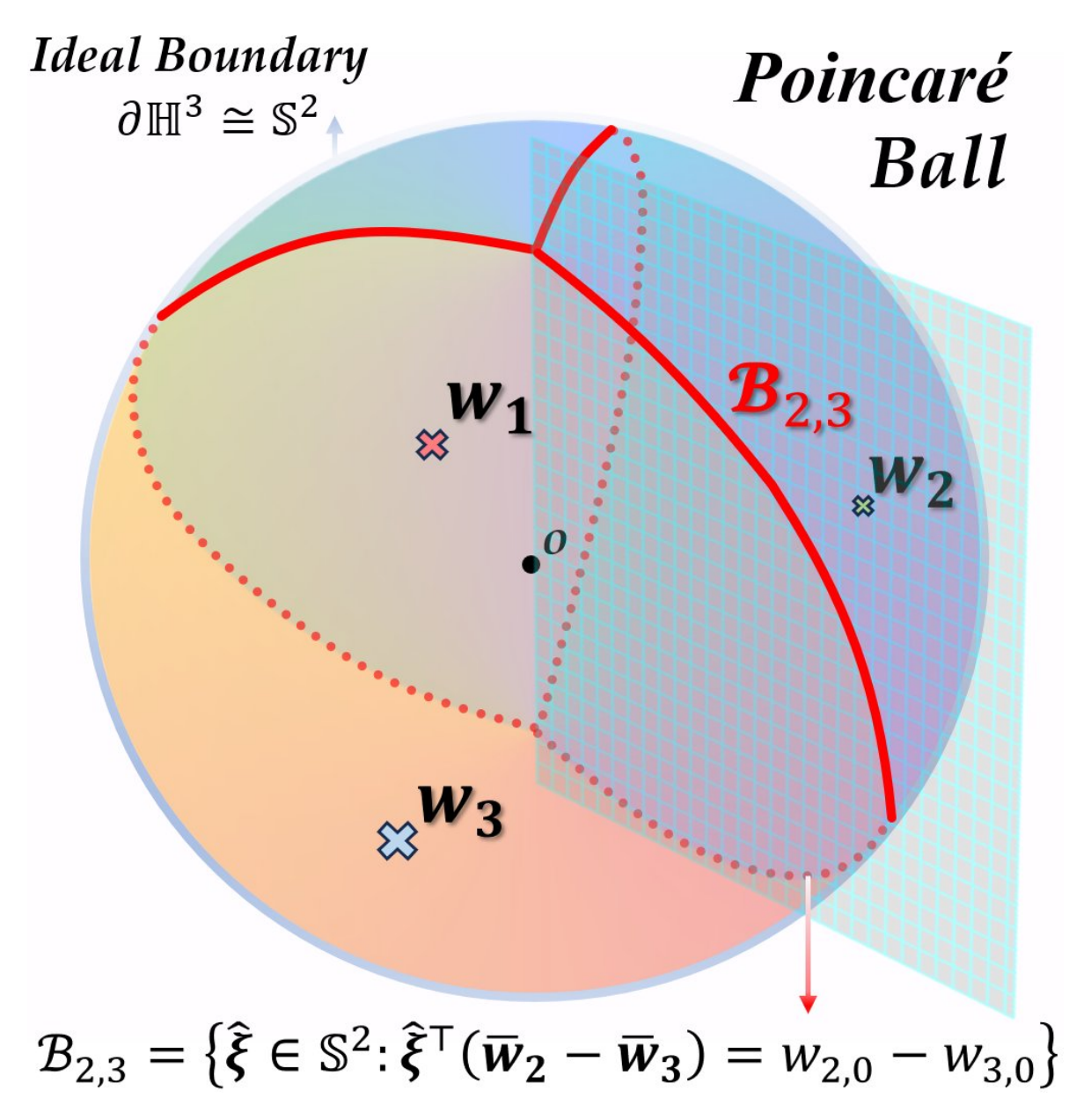}
  \caption{Asymptotic decision boundary on the ideal boundary.}
  \label{fig:boundary}
  \vspace{-20pt}
\end{wrapfigure}
If the minimizer of \(\mathcal L^{\mathrm{BS}}(\bm\xi,\cdot)\) is unique,
the \(O(e^{-2r})\) error cannot change the argmin for sufficiently large
\(r\). Therefore
\(\hat y(\mathbf x)=\argmin_y \mathcal L^{\mathrm{BS}}(\bm\xi,y)\).
It remains to identify this argmin. Writing
\(\beta_c=b_{\bm\xi}(\mathbf w_c)\), the Busemann loss depends on the
relative margins \((\beta_c-\beta_y)_{c\neq y}\). Since a totally regular
PERM template is permutation invariant and strictly decreasing in each
coordinate, the minimum is achieved exactly by the smallest Busemann value:
\(\argmin_y \mathcal L^{\mathrm{BS}}(\bm\xi,y)
=\argmin_c b_{\bm\xi}(\mathbf w_c)\).
Finally, \(b_{\bm\xi}(\mathbf w_c)=b_{\bm\xi}(\mathbf w_{c'})\) is equivalent
to \(\xhat^\top(\wbar_c-\wbar_{c'})=w_{c,0}-w_{c',0}\). 
The additional minimum condition ensures that the
pairwise bisector belongs to the multiclass decision
boundary rather than separating two nonwinning classes.
The full and detailed proof is in Appendix [\ref{proof: 5.7}].
\hfill $\square$

Figure~\ref{fig:boundary} illustrates the pairwise boundary
\(\mathcal B_{2,3}\) between prototypes \(w_2\) and \(w_3\), which is the
spherical section obtained by intersecting
\(\partial\mathbb H^3\simeq\mathbb S^2\) with the affine hyperplane
\(\hat\xi^\top(\bar w_2-\bar w_3)=w_{2,0}-w_{3,0}\). And the Busemann loss \(\mathcal{L}^{\mathrm{BS}}(\bm{\xi}_i,y_i)\) evaluates
the classification quality of the boundary direction \(\bm{\xi}_i\). This
can be interpreted through the following Busemann margin:

\begin{definition}[Busemann margin]
\label{def:horo_margin}
Let \(\bm{\xi}=(1,\xhat)\in\partial\Hn\). The \emph{Busemann margin}
of sample \(i\) at direction \(\bm{\xi}\) is
\begin{equation}\label{eq:horo_margin}
m_i(\bm{\xi})
:=
\min_{c\neq y_i}
\bigl[
b_{\bm{\xi}}(\mathbf w_c)
-
b_{\bm{\xi}}(\mathbf w_{y_i})
\bigr].
\end{equation}
Sample \(i\) is correctly classified at the ideal direction
\(\bm{\xi}\) if and only if \(m_i(\bm{\xi})>0\).
\end{definition}

\begin{corollary}[Correct classification under cross-entropy]
\label{lem:correct_min}
Under cross-entropy loss, for any class label $y_i$ and
any boundary direction $\bm\xi\in\partial\Hn$, if
\begin{equation}
\mathcal L^{\mathrm{BS}}(\bm\xi,y_i)<\log 2,
\end{equation}
then the Busemann margin is strictly positive,
$m_i(\bm\xi)>0$, and the asymptotic Busemann classifier
correctly predicts class $y_i$.
\end{corollary}

\noindent\textit{Proof.}
By the cross-entropy formulation of the Busemann loss,

\(
\mathcal L^{\mathrm{BS}}(\bm\xi,y_i)
=
\log\left(
1+\sum_{c\ne y_i}
e^{b_{\bm\xi}(w_{y_i})-b_{\bm\xi}(w_c)}
\right)
\ge
\log(1+e^{-m_i(\bm\xi)}).
\)
Rearranging yields the stated margin bound.
Since $\mathcal L^{\mathrm{BS}}(\bm\xi,y_i)<\log 2$,
the lower bound is strictly positive, implying
$b_{\bm\xi}(w_{y_i})<b_{\bm\xi}(w_c)$ for all
$c\ne y_i$. Thus the Busemann classifier assigns
the correct label.
\hfill $\square$

Together, these results characterize the asymptotic
decision boundaries and classification margins through
Busemann geometry. We next examine the predicted radial
and angular behavior through numerical experiments.



\section{Experiments}

We evaluate hyperboloid gradient dynamics using \(C=3\) prototypes on
\(\mathbb H^2\) at radius \(\rho=1\) with \(N=15\) embeddings, and ablate
\(C\in\{2,3,5\}\), PERM templates
\(\psi\in\{\psi_{\mathrm{CE}},\psi_{\mathrm{EXP}},\psi_{\mathrm{SS}}\}\). We integrate both the full coupled gradient flow and the
separate radial dynamics.

\begin{figure}[h]
    \centering
    \begin{subfigure}[b]{0.32\linewidth}
        \centering
        \includegraphics[width=\linewidth]{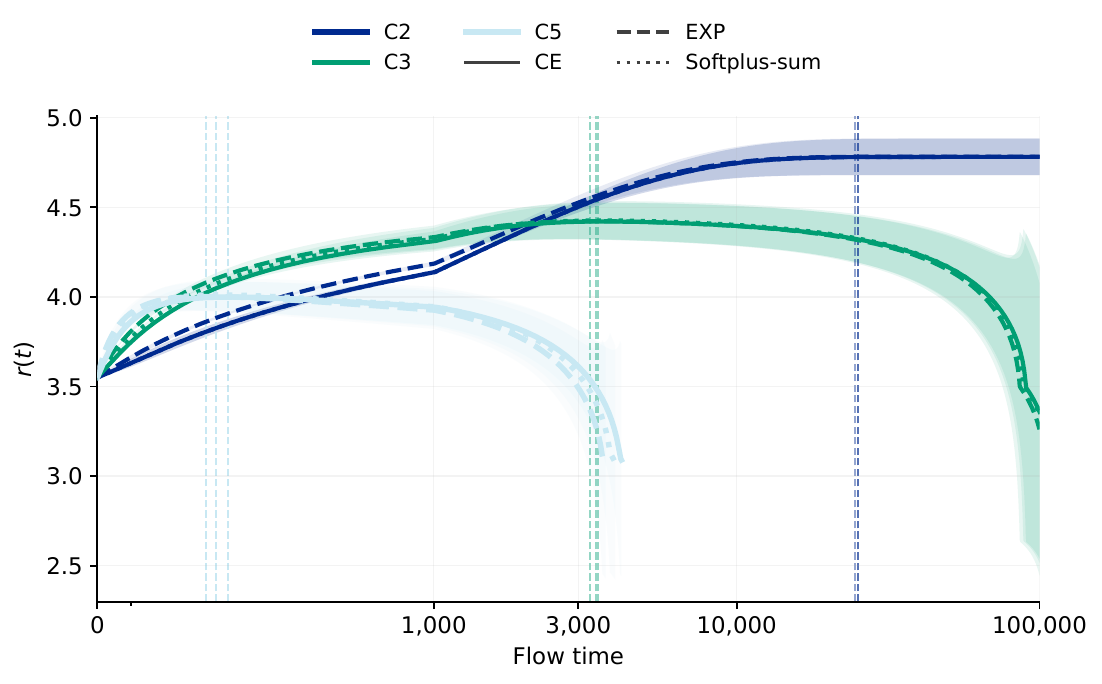}
        \caption{Radial evolution \(r(t)\).}
        \label{fig:coupled_radius}
    \end{subfigure}
    \hfill
    \begin{subfigure}[b]{0.32\linewidth}
        \centering
        \includegraphics[width=\linewidth]{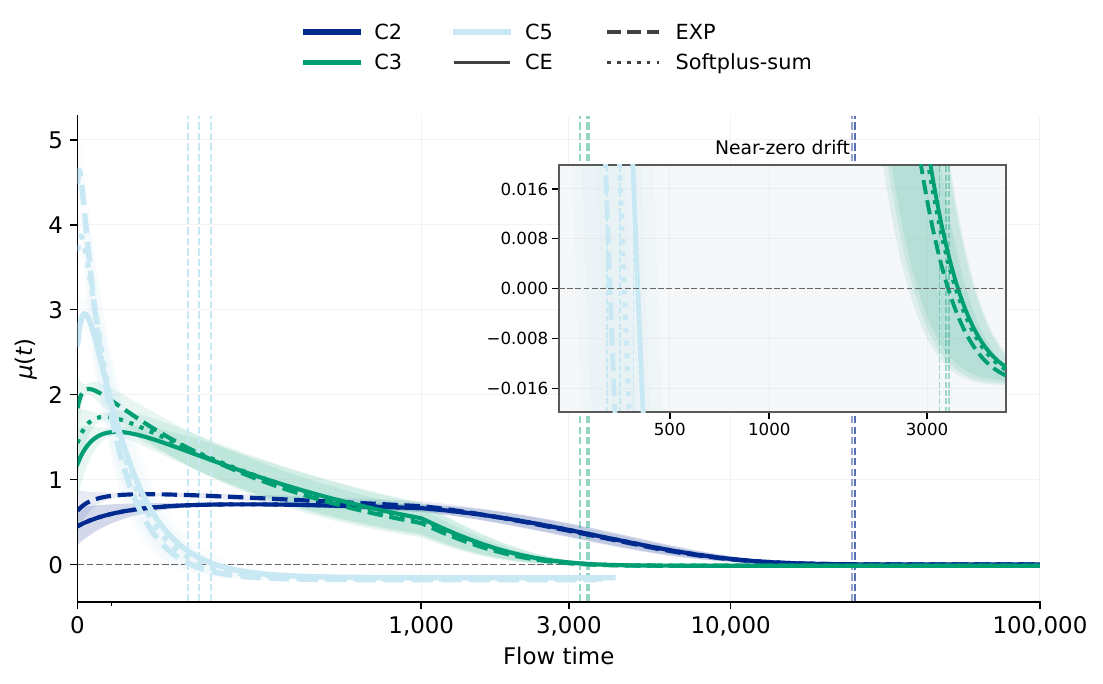}
        \caption{Drift coefficient \(\mu(\bm\xi(t),y)\).}
        \label{fig:coupled_mu}
    \end{subfigure}
    \hfill
    \begin{subfigure}[b]{0.32\linewidth}
        \centering
        \includegraphics[width=\linewidth]{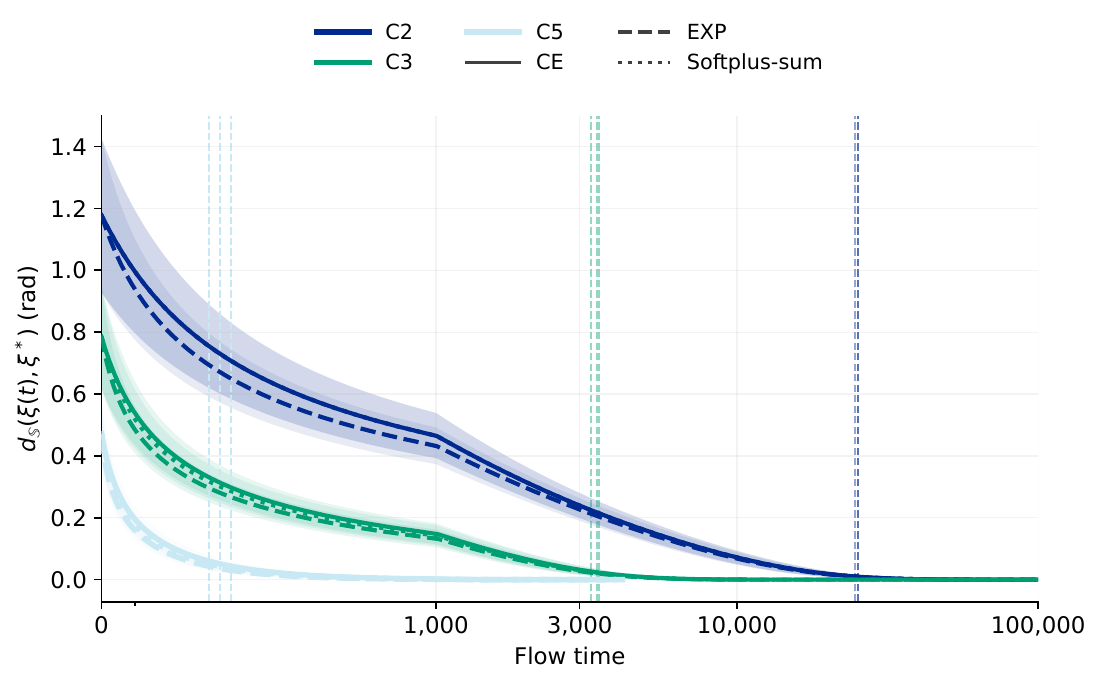}
        \caption{Angular \(d_{\Sn}(\bm\xi(t),\bm\xi^*)\).}
        \label{fig:coupled_angle}
    \end{subfigure}
    \captionsetup{justification=centering,singlelinecheck=false}
    \caption{Three-phase Busemann dynamics
    under coupled Riemannian gradient flow.}
    \vspace{-8pt}
    \label{fig:comparison}
\end{figure}

Figure~\ref{fig:comparison} confirms the three phases predicted
by Theorem~6.4 across configurations with \(C\ge3\).
The trajectories initially move outward, undergo angular alignment
with radial drift reversal, and subsequently return slowly toward
the interior. Across all configurations, the directions approach
Busemann local minima, supporting the predicted Busemann implicit
bias across prototype geometries and PERM losses, with weaker drift
reversal in degenerate cases. Shaded regions indicate standard
deviations across seeds. Vertical dashed lines indicate entry into
sufficiently aligned regions \(\mathcal A(\bm\xi^*)\).

\vspace{-7pt}

\begin{figure}[h]
  \centering
  \includegraphics[width=\linewidth]{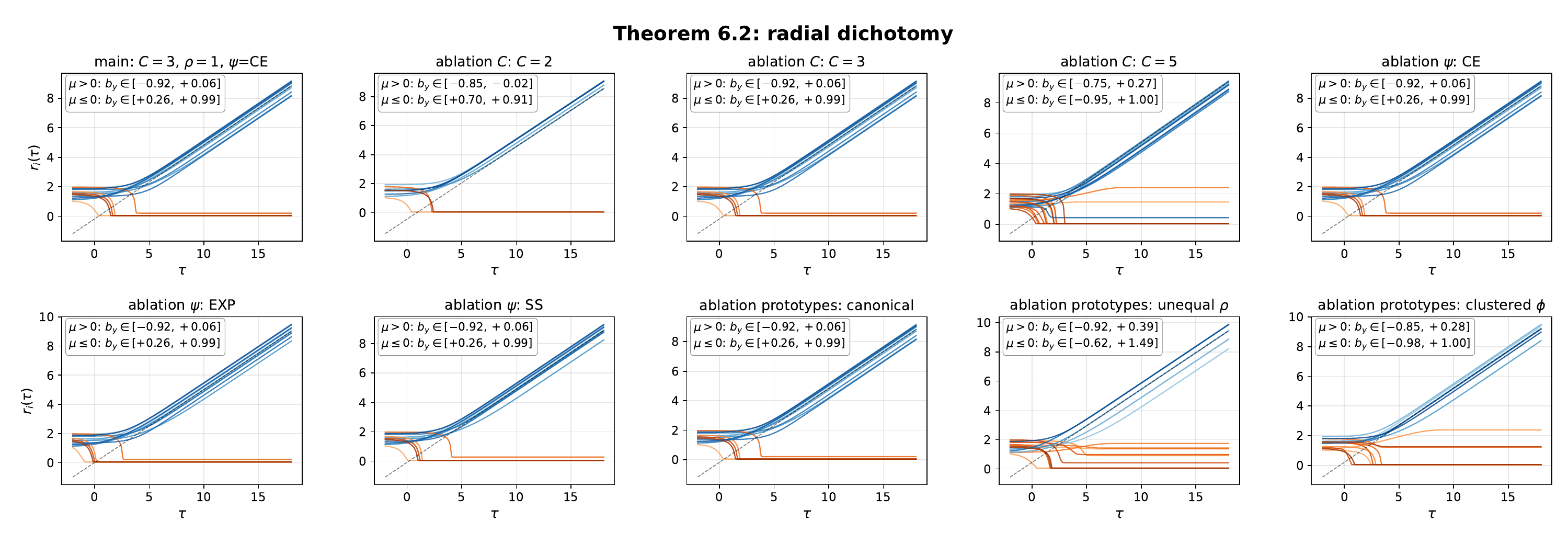}
  \caption{Numerical verification of radial dichotomy under
  fixed-direction gradient flow.}
  \label{fig:experiment_summary}
  \vspace{-7pt}
\end{figure}

Figure~\ref{fig:experiment_summary} confirms the predicted
\(r_i(t)=\frac12\log t+O(1)\) scaling under persistent positive drift,
while negative drift induces inward motion. These results support the
radial dichotomy and three-phase behavior across the tested settings.
Full experimental details, neural network experiments, and additional
diagnostics are provided in
Appendix~[\ref{sec:experiments(appendix)}].

\vspace{-7pt}
\section{Conclusions}

In this work, we analyzed the implicit bias of Riemannian gradient flow for
hyperbolic multiclass classification with fixed prototypes. Through a
Radial Busemann decomposition, we showed that large-radius dynamics are
governed by an interaction between radial drift and boundary direction
optimization. This yields a radial dichotomy and a three-phase regime of outward motion,
angular alignment, and slow return, during which the directions approach
local minima of the Busemann loss. We also established radial escape
and convergence to global Busemann risk minimizers.
We further characterized the asymptotic classifier and its spherical
decision boundaries.

Controlled experiments support the predicted radial and angular dynamics,
while coupled integration illustrates drift reversal and three-phase
behavior. Neural network experiments further validate the Busemann
asymptotic approximation. Together, these results provide a unified
perspective on boundary saturation and near-boundary clustering through
their shared Busemann directional bias. More broadly, our analysis
identifies the ideal boundary as the natural geometric object governing
optimization bias in hyperbolic representation learning and suggests
a boundary-based perspective for studying more general hyperbolic
learning dynamics.

\newpage

\begin{ack}

\textbf{Funding.}
This work was partially supported by JST Moonshot R\&D, Grant Number JPMJPS2011, the New Energy and Industrial Technology Development Organization (NEDO) under Project JPNP26007, JSPS KAKENHI, Grant Number JP25K21268 and JST SPRING, Grant Number JPMJSP2108.

\textbf{Competing interests.}
The authors declare no competing interests.

\end{ack}

\bibliographystyle{unsrtnat}
\bibliography{refs}

\newpage
\appendix

\section{Related Works, Discussion and Additional Results}
\subsection{Related Works}
\label{sec:related}

\paragraph{Implicit bias.}
Implicit bias studies how optimization algorithms select particular
solutions among many empirical risk minimizers, even without explicit
regularization. In separable linear classification, gradient descent with
exponential or cross entropy losses converges in direction to the maximum
margin classifier~\cite{GradientDescentonSeparableData}. This line of work
has been extended to refined convergence rates and margin dynamics
~\cite{Nacson2019Convergence,JiTelgarsky2019Risk}, nonseparable
classification~\cite{nonseparable}, homogeneous neural
networks~\cite{GradientHomo}, directional convergence in deep
models~\cite{Ji2020DirectionalCA}, general homogeneous or overparameterized
models~\cite{ChizatBach2020,Woodworth2020Kernel}, and optimization geometry
dependent biases such as steepest descent and mirror descent
~\cite{geometry,flavor}.
Beyond classification, implicit regularization has also been studied in
matrix factorization, deep matrix factorization, and deep linear networks,
where gradient based methods favor low rank or low complexity
solutions~\cite{inmatrix,inmatrix2}. These works mainly characterize Euclidean biases toward
maximum margin, low norm, low rank, or other complexity controlled
structures.
\paragraph{Multiclass implicit bias and PERM losses.}
For multiclass classification, recent work extends binary implicit bias
theory to permutation invariant relative margin losses through the PERM
framework~\cite{implicitmulti,JMLR:v25:23-1599}. This framework covers cross
entropy, multiclass exponential losses, and broader relative margin losses.
We use PERM losses as a general multiclass loss class, but our goal is
different from Euclidean multiclass margin theory: instead of characterizing
linear maximum margin behavior, we study radial escape and boundary
direction selection in hyperbolic space.

\paragraph{Hyperbolic representation learning and Busemann based methods.}
Hyperbolic representation learning has been widely used for hierarchical,
tree like, and long tailed data. Representative works include Poincar\'e
embeddings~\cite{Nickel2017PoincarEF}, Lorentz model
embeddings~\cite{pmlr-v80-nickel18a}, hyperbolic neural
networks~\cite{hnn}, hyperbolic neural network variants~\cite{hnn+},
hyperbolic graph neural networks~\cite{hgnn}, and hyperbolic
Transformers~\cite{Yang2024HypformerEE}; see~\cite{survey1} for a survey.
Applications include computer vision, few shot learning, metric learning,
graph learning, recommendation, and reinforcement learning
~\cite{metriclearning,fewshot,hgformer,hyperbolicRL}. Several empirical
studies also report boundary related phenomena, including numerical
instability near the Poincar\'e boundary~\cite{NumericalStability}, the
effectiveness of clipping~\cite{clipped}, common radius
concentration~\cite{fewshot}, and metric learning induced class
separation~\cite{metriclearning}.
Busemann functions have also been used in hyperbolic machine
learning. Hyperbolic Busemann Learning uses ideal prototypes and a Busemann
loss~\cite{Atigh2021HyperbolicBL}, and
horospherical decision boundary methods use level sets of Busemann functions
as classifiers~\cite{Fan2023Horospherical}. Recent works further study
Busemann based neural network
layers~\cite{Chen2026HyperbolicBN}, and related constructions on symmetric
spaces~\cite{Nguyen2026NeuralNO}. These works demonstrate the
importance of hyperbolic and Busemann geometry, but mainly focus on model
design, empirical behavior, or geometric constructions. Our work is
complementary: rather than proposing a new Busemann model, we show that
Busemann geometry emerges as the large radius limit of distance based
hyperbolic multiclass classification. Under PERM losses, the common radial
term cancels in the relative margins, yielding a Busemann risk on the ideal
boundary and leading to radial dichotomy, angular convergence, and an
asymptotic boundary classifier.

\subsection{Discussion and Limitation}
\label{sec:discussion}

\subsubsection{Discussion.}
Our results provide an asymptotic description of the boundary regime in
hyperbolic representation learning. This regime is different from the local
Euclidean behavior near the origin: near the ideal boundary, hyperbolic
distances admit an exponential expansion governed by Busemann functions.
Our analysis exploits this expansion to separate the loss into a common
radial term and a term that depends on direction through the Busemann
function. This separation is the geometric mechanism behind the radial
dichotomy, the Busemann risk on the ideal boundary, and the resulting
asymptotic classifier.

This perspective helps interpret several empirical phenomena in hyperbolic
learning. First, the radial dichotomy gives a mechanism for boundary
saturation: when the drift coefficient is positive, embeddings are pushed
toward the boundary at rate \(r(t)=\frac12\log t+O(1)\), consistent with
observations that embeddings concentrate near the Poincar\'e boundary and
that numerical instabilities arise in this region~\cite{NumericalStability}.
The coefficient \(\frac12\) is universal in our setting, in the sense that it
does not depend on the data geometry or margin. This contrasts with Euclidean separable linear classification, where the
logarithmic growth coefficient is determined by the data margin
\(\gamma^*\)~\cite{GradientDescentonSeparableData}. Second, the angular convergence theorem suggests a mechanism
for near boundary clustering: once embeddings have escaped sufficiently,
their discriminative information is encoded mainly in boundary directions
governed by the Busemann risk on \(\partial\Hn\), consistent with observations
of concentration at a common radius and directional class
separation~\cite{fewshot,metriclearning}.

Our theory also provides an interpretation of radius clipping
~\cite{clipped}. Since the leading classification behavior near the boundary
depends on Busemann values and boundary directions, extremely large radii
are not necessary once embeddings have entered the asymptotic regime.
Clipping can therefore be viewed as preventing excessive radial growth while
preserving the boundary direction structure that determines the classifier
in the late stage of training.

The discussion above was stated for the normalized case \(K=1\). For
completeness, we now record the corresponding curvature scaling, which
justifies the \(K=1\) convention used in the main text.

\paragraph{Curvature Dependence.}

For general Minkowski constant $K > 0$ (equivalently, curvature $\kappa = -1/K$), the model $\HnK$ and all preceding quantities rescale by $\sqrt{K}$.
Substituting $r \mapsto r/\sqrt{K}$, $d_{\HH}^K = \sqrt{K}\,d_{\HH}^{K=1}$, and $b_{\xhat}^K = \sqrt{K}\,b_{\xhat}^{K=1}$ throughout:

\begin{corollary}[Curvature scaling]
\label{cor:curvature}
For any $K > 0$, the rescaled properties are:
\begin{enumerate}[label=(\alph*)]
\item (Radial rate) Under the persistent positive drift condition of Theorem 6.2, $r_i(t) = \dfrac{\sqrt{K}}{2}\log t + O(1)$ as $t \to \infty$.
\item (Angular limit)
Under the assumptions of Theorem~6.5 for ideal prototypes,
\(
\hat{\bm\xi}_i(t)\to\hat{\bm\xi}_i^{*,K}
\in\arg\min_{\hat{\bm\xi}\in\mathbb S^{n-1}}
\mathcal L^{\mathrm{BS},K}(\bm\xi,y_i),
\)
where
\(\mathcal L^{\mathrm{BS},K}(\bm\xi,y_i)
:=\psi(\Upsilon_{y_i}D\mathbf b^K(\bm\xi))\).
\item (Decision boundary) The asymptotic decision boundaries on $\partial\HnK \cong \Sn$ are identical to the $K=1$ case.
\item (Horospherical margin) $m_i^K(\bm{\xi}) = \sqrt{K}\,m_i^{K=1}(\bm{\xi})$.
\end{enumerate}
\end{corollary}

\subsubsection{Limitations.}

Our analysis focuses on Riemannian gradient flow for embeddings with fixed
class prototypes. This setting follows a common strategy in implicit bias
theory, where part of the learning problem is fixed to isolate the
optimization mechanism of interest. For example, separable linear
classification fixes the data and studies classifier weight
dynamics~\cite{GradientDescentonSeparableData}, while matrix factorization
fixes observed entries and analyzes the implicit preference of learned
factors~\cite{inmatrix,inmatrix2}. In our case, fixing the prototypes
isolates the geometric mechanism governing radial drift and angular
alignment. We further investigate shared encoder parameterization
numerically. Extending the theoretical analysis to joint optimization
of embeddings, prototypes, and encoder parameters introduces additional
interactions through the encoder Jacobian and remains an important
direction for future work.

Our experiments combine controlled numerical integration with neural
network investigations. Coupled gradient flow experiments illustrate
the predicted three-phase dynamics across prototype configurations
and PERM losses. Shared encoder experiments reveal exponential
proximity between hyperbolic and Busemann stationary points, while
ResNet-18 experiments on CIFAR-10 provide complementary evidence
through parameter gradient comparisons. Together, these experiments
extend the numerical evidence beyond independent embedding dynamics,
while broader architectures and jointly learned prototypes remain
natural directions for further investigation.

\newpage

\section{Proof of Section~\ref{Section4}}

\subsection{Proof of Lemma~\ref{lem:dist_busemann}}\label{proof:4.1}

\textbf{Lemma~\ref{lem:dist_busemann}}~\ (Radial-Busemann expansion)
\textit{For $\mathbf{x} = (\cosh r,\, \sinh r\,\xhat) \in \Hn 
\setminus \{\mathbf{o}\}$ and $\mathbf{w} = (w_0, \wbar) 
\in \mathbb{H}\times\mathbb{H}^{n-1}$, let $\boldsymbol{\xi} = (1, \xhat) \in 
\partial\Hn$. There exists a constant 
$R_0(\mathbf{w}) > 0$ such that for all 
$r \ge R_0$,
\begin{equation*}
  d_{\HH}(\mathbf{x}, \mathbf{w}) 
    = r + \buse(\mathbf{w}) + O(e^{-2r}),
\end{equation*}
where $\buse(\mathbf{w}) = \log(w_0 - \xhat \cdot \wbar)$ 
is the Busemann function.}

\begin{proof}
Recall that the hyperboloid model is defined as $\Hn = \{\mathbf{x} = (x_0, x_1, \ldots, x_n)= (x_0, \bar{\mathbf{x}}) \in \R^{n+1} : \ipL{\mathbf{x}}{\mathbf{x}} = -1,\; x_0 > 0\}$, where the Minkowski inner product is
\(
  \ipL{\mathbf{x}}{\mathbf{y}} 
  = -x_0 y_0 + \sum_{i=1}^{n} x_i y_i= -x_0 y_0 + \bar{\mathbf{x}}^{\top}\bar{\mathbf{y}}.
\)
The geodesic distance between $\mathbf{x}, \mathbf{w} \in \Hn$ is

\begin{equation}\label{eq:dist_formula}
  d_{\HH}(\mathbf{x}, \mathbf{w}) 
  = \arcosh\bigl(-\ipL{\mathbf{x}}{\mathbf{w}}\bigr).
\end{equation}

We are given $\mathbf{x} = (\cosh r,\, \sinh r\,\xhat)$ with $r > 0$ and $\xhat \in \Sn$, and $\mathbf{w} = (w_0, w_1, \ldots, w_n) = (w_0, \wbar)$. Substituting into the Minkowski inner product:
\begin{align}
  -\ipL{\mathbf{x}}{\mathbf{w}}
  &= -\bigl(-\cosh r \cdot w_0 + \sinh r\,\xhat \cdot \wbar\bigr) \nonumber\\
  &= w_0 \cosh r - (\xhat \cdot \wbar)\sinh r. \label{eq:lorentz_ip}
\end{align}
Now we use the exponential forms of hyperbolic functions:
\[
  \cosh r = \frac{e^r + e^{-r}}{2}, \qquad
  \sinh r = \frac{e^r - e^{-r}}{2}.
\]
Substituting into \eqref{eq:lorentz_ip}:
\begin{align}
  -\ipL{\mathbf{x}}{\mathbf{w}}
  &= w_0 \cdot \frac{e^r + e^{-r}}{2} 
     - (\xhat \cdot \wbar) \cdot \frac{e^r - e^{-r}}{2} \nonumber\\
  &= \frac{e^r}{2}\bigl(w_0 - \xhat \cdot \wbar\bigr)
     + \frac{e^{-r}}{2}\bigl(w_0 + \xhat \cdot \wbar\bigr). \label{eq:ip_expanded}
\end{align}

Since $\mathbf{w} \in \Hn$, we have $\ipL{\mathbf{w}}{\mathbf{w}} = -w_0^2 + \|\wbar\|^2 = -1$, so $w_0 = \sqrt{1 +\|\wbar\|^2}$. In particular $w_0 \geq 1$ and $w_0 >\|\wbar\|$. By the Cauchy--Schwarz inequality in $\R^n$:
\[
  \|{\xhat \cdot \wbar}\| \leq \|{\xhat}\|\,\|\wbar\| 
  = \|\wbar\| < w_0,
\]
since $\xhat \in \Sn$ has $\|{\xhat} = 1$. Therefore $w_0 - \xhat \cdot \wbar > 0$. We define
\begin{equation}\label{eq:def_A}
  A \;:=\; w_0 - \xhat \cdot \wbar \;>\; 0.
\end{equation}
Similarly, we define $B := w_0 + \xhat \cdot \wbar > 0$. Then substitute into \eqref{eq:ip_expanded}: 
\begin{equation}\label{eq:ip_AB}
  -\ipL{\mathbf{x}}{\mathbf{w}} 
  = \frac{e^r}{2}\,A + \frac{e^{-r}}{2}\,B
  = \frac{e^r A}{2}\left(1 + \frac{B}{A}\,e^{-2r}\right).
\end{equation}

Since the distance function is $d_{\HH}(\mathbf{x}, \mathbf{w}) 
  = \arcosh\bigl(-\ipL{\mathbf{x}}{\mathbf{w}}\bigr)$, we also need to expand the $\arcosh$ function. By definition, $\arcosh(u) = \log(u + \sqrt{u^2 - 1})$. We have
\begin{equation}
\begin{aligned}
  u + \sqrt{u^2 - 1}
  &= u + u\sqrt{1 - 1/u^2} \\
  &= u\left(1 + \sqrt{1 - 1/u^2}\right) \\
  &= 2u \cdot \frac{1 + \sqrt{1 - 1/u^2}}{2}.
\end{aligned}
\end{equation}
Therefore
\begin{equation}
  \arcosh(u) = \log(2u) + \log\!\left(\frac{1 + \sqrt{1 - 1/u^2}}{2}\right).
\end{equation}

Let $u \geq 2$ and set $\delta = 1/u^2$, then $\delta \in (0, 1/4]$. 
Using the Taylor bound $\sqrt{1 - \delta} \geq 1 - \delta$ for $\delta \in [0,1]$:
\begin{equation}
  \frac{1 + \sqrt{1-\delta}}{2} \geq \frac{1 + 1 - \delta}{2} = 1 - \frac{\delta}{2}.
\end{equation}

Also $\sqrt{1-\delta} \leq 1$, so:
\begin{equation}
  \frac{1+\sqrt{1-\delta}}{2} \leq 1.
\end{equation}

Hence
\begin{equation}
  1 - \frac{\delta}{2} \leq \frac{1+\sqrt{1-\delta}}{2} \leq 1.
\end{equation}

Since $\log$ is monotone and $\abs{\log(1-s)} \leq 2s$ for $s \in [0, 1/2]$:
\begin{equation}
  \abs{\log\!\left(\frac{1 + \sqrt{1 - \delta}}{2}\right)} 
  \leq \abs{\log(1 - \delta/2)} 
  \leq 2 \cdot \frac{\delta}{2} = \delta = \frac{1}{u^2}.
\end{equation}

So we get that for $u \geq 2$, we have:
\begin{equation}\label{log(2u)}
  \arcosh(u) = \log(2u) + \varepsilon(u), 
  \qquad \text{where } \abs{\varepsilon(u)} \leq \frac{1}{u^2}.
\end{equation}

Then we need $u = -\ipL{\mathbf{x}}{\mathbf{w}} \geq 2$ for applying~\ref{log(2u)} .  From~\eqref{eq:ip_AB},
\(
  u 
  = \frac{e^r A}{2}\!\left(1 + \frac{B}{A}\,e^{-2r}\right),
\)
since $A > 0$ and $B > 0$, 
$1 + \frac{B}{A}\,e^{-2r} \geq 1$ for all $r > 0$, so
\begin{equation}
  u \;\geq\; \frac{e^r A}{2}.
\end{equation}

Since $\xhat \in \Sn$ is a unit vector, the Cauchy--Schwarz
inequality gives
\begin{equation}
  A \;=\; w_0 - \xhat \cdot \wbar
  \;\geq\; w_0 - \|\wbar\|.
\end{equation}
Using the hyperboloid constraint
$w_0^2 - \|\wbar\|^2 = 1$, we factor
\begin{equation}
  w_0 - \|\wbar\|
  \;=\; \frac{w_0^2 - \|\wbar\|^2}{w_0 + \|\wbar\|}
  \;=\; \frac{1}{w_0 + \|\wbar\|}
  \;>\; 0.
\end{equation}
Hence $A \geq \frac{1}{w_0 + \|\wbar\|}$ uniformly over
all directions $\xhat \in \Sn$. Then the condition
$u \geq \frac{e^r A}{2} \geq 2$ is guaranteed whenever
\begin{equation}\label{eq:R0}
  r \;\geq\; R_0(\mathbf{w})
  \;:=\; \log\bigl(4(w_0 + \|\wbar\|)\bigr),
\end{equation}
which depends only on the prototype~$\mathbf{w}$, not on
the direction~$\xhat$.

Finally, for $r$ large enough, all $r \geq R_0(\mathbf{w})$:
\begin{align}
  d_{\HH}(\mathbf{x}, \mathbf{w})
  &= \arcosh(u) \nonumber\\
  &= \log(2u) + O(u^{-2}). \label{eq:dist_arcosh}
\end{align}
We compute $\log(2u)$:
\begin{align}
  \log(2u) 
  &= \log\left(2 \cdot \frac{e^r A}{2}\left(1 + \frac{B}{A}e^{-2r}\right)\right) \nonumber\\
  &= \log\left(e^r A \left(1 + \frac{B}{A}e^{-2r}\right)\right) \nonumber\\
  &= \log(e^r) + \log(A) + \log\!\left(1 + \frac{B}{A}e^{-2r}\right) \nonumber\\
  &= r + \log A + \log\!\left(1 + \frac{B}{A}e^{-2r}\right). \label{eq:log2u}
\end{align}
The last logarithm is bounded by $e^{-2r}$ using $\log(1+s) \leq s$ for $s > 0$:
\begin{equation}\label{eq:log_remainder}
  0 \leq \log\!\left(1 + \frac{B}{A}e^{-2r}\right) 
  \leq \frac{B}{A}\,e^{-2r} \leq C_1(\textbf{w})e^{-2r},
\end{equation}
where $C_1(\textbf{w}):=(w_0+\|\wbar\|)^2$ since $\frac{B}{A}=\frac{w_0+\xhat\wbar}{w_0-\xhat\wbar}\leq(w_0+\|\wbar\|)^2$. Also, the $O(u^{-2})$ term in \eqref{eq:dist_arcosh} is also bounded by $e^{-2r}$, since $u \geq \frac{e^r A}{2}$ for $r \geq R_0$, we have $u^{-2} \leq \frac{4}{A^2}\,e^{-2r} \leq C_2(\textbf{w})e^{-2r}$, where $C_2(\textbf{w}):=4(w_0 + \|\wbar\|)^2.$
Combining \eqref{eq:dist_arcosh}, \eqref{eq:log2u}, and \eqref{eq:log_remainder}:
\begin{align*}
  d_{\HH}(\mathbf{x}, \mathbf{w})
  &= r + \log A + \underbrace{\log\!\left(1 + \frac{B}{A}e^{-2r}\right)}_{O(e^{-2r})} 
     + \underbrace{O(u^{-2})}_{O(e^{-2r})} \\
  &= r + \log A + O(e^{-2r}).
\end{align*}
Recalling $A = w_0 - \xhat \cdot \wbar$ from \eqref{eq:def_A}, we have
\(
  \log A = \log(w_0 - \xhat \cdot \wbar).
\)
This is precisely the Busemann function $\buse(\mathbf{w})$ centered at the ideal boundary point $\bm{\xi} = (1, \xhat) \in \partial\Hn$:
\[
  \buse(\mathbf{w}) = \log(w_0 - \xhat \cdot \wbar).
\]

So, we have shown that for $r \geq R_0(\mathbf{w})$,
\[
  d_{\HH}(\mathbf{x}, \mathbf{w}) 
  = r + \buse(\mathbf{w}) + O(e^{-2r}),
\]

which completed the proof.
\end{proof}
\newpage

\subsection{Proof of Proposition~\ref{prop:loss_expansion}}\label{proof:4.3}

\textbf{Proposition~\ref{prop:loss_expansion}}\quad(Loss expansion)
\textit{Under Assumption~(\ref{assump:perm}), there exists a 
constant $R_0(\mathbf{W}) > 0$ such that for each 
$i \in [N]$ and all $r_i \ge R_0$,
\begin{equation*}
\ell_i(\mathbf{x}_i, y_i) 
\;=\; \mathcal{L}^{\mathrm{BS}}(\bm{\xi}_i, y_i) 
\;+\; \mu(\bm{\xi}_i, y_i)\,e^{-2r_i} 
\;+\; O(e^{-4r_i}),
\end{equation*}
where
\begin{equation*}
\mu(\bm{\xi}, y) \;:=\; -\bigl\langle 
\nabla\psi\bigl(\Upsilon_{y} D\,\mathbf{b}(\bm{\xi})\bigr),\; 
\Upsilon_{y} D\,\boldsymbol{\alpha}(\bm{\xi})\bigr\rangle,
\end{equation*}
and we call $\mu : \partial\Hn \times [C] \to \mathbb{R}$ 
the radial drift coefficient.}

\begin{proof}

Recall from the proof of Lemma~\ref{lem:dist_busemann}
(Appendix~\ref{proof:4.1}) that for a sample point
$\mathbf{x}_i = (\cosh r_i,\, \sinh r_i\,\xhat_i)$ and a
prototype $\mathbf{w}_c = (w_{c,0},\, \wbar_c)$, we defined
\[
  A_c := w_{c,0} - \xhat_i \cdot \wbar_c > 0,
  \qquad
  B_c := w_{c,0} + \xhat_i \cdot \wbar_c > 0,
\]
and showed that
\begin{equation}\label{eq:uc_recall}
  u_c \;:=\; -\ipL{\mathbf{x}_i}{\mathbf{w}_c}
  \;=\; \frac{e^{r_i} A_c}{2}
        \bigl(1 + s_c\bigr),
  \qquad
  s_c := \frac{B_c}{A_c}\,e^{-2r_i}.
\end{equation}
Lemma~\ref{lem:dist_busemann} only kept the leading term,
here we need the exact $e^{-2r_i}$ coefficient, so we
expand more carefully. Here by
using the Cauchy inequality and the hyperboloid
constraint $w_{c,0}^2 - \norm{\wbar_c}^2 = 1$, we established
in Lemma~\ref{lem:dist_busemann} that
\begin{equation}\label{eq:Ac_lower}
  A_c \;\geq\; w_{c,0} - \norm{\wbar_c}
  \;=\; \frac{1}{w_{c,0} + \norm{\wbar_c}}
  \;>\; 0.
\end{equation}
Similarly, $B_c \leq w_{c,0} + \norm{\wbar_c}$, so the ratio
$B_c / A_c$ is bounded above by $(w_{c,0}+\norm{\wbar_c})^2$.
Define the global threshold
\begin{equation}\label{eq:R0_global}
  R_0(\mathbf{W})
  \;:=\; \max_{c \in [C]}
         \log\bigl(4(w_{c,0}+\norm{\wbar_c})\bigr).
\end{equation}
For any $r_i \geq R_0(\mathbf{W})$ and every $c \in [C]$,
two conditions hold simultaneously:
 \begin{equation}
     u_c \geq \frac{e^{r_i} A_c}{2} \geq 2,
 \end{equation} ensuring that
      the $\arcosh$ expansion
      (equation~\eqref{log(2u)} in
      Lemma~\ref{lem:dist_busemann}) is valid, and
 the perturbation $s_c $ satisfies
      \begin{equation}
     0 < s_c \leq
       (w_{c,0}+\norm{\wbar_c})^2 \cdot
       \frac{1}{16(w_{c,0}+\norm{\wbar_c})^2}
       = \frac{1}{16}, \end{equation}
      so all Taylor remainders for $s_c \to 0$ are uniformly
      $O(e^{-4r_i})$.
We now refine the two pieces of
$d_{\HH}(\mathbf{x}_i,\mathbf{w}_c) = \log(2u_c) + \varepsilon(u_c)$.
From~\eqref{eq:uc_recall}:
\begin{align}
  \log(2u_c)
  &= \log\!\bigl(e^{r_i}\,A_c\,(1+s_c)\bigr)
     \nonumber\\
  &= \log(e^{r_i}) + \log(A_c) + \log(1+s_c).
     \label{eq:log2uc_split}
\end{align}
Since $0 < s_c \leq 1/16$, we use the Taylor expansion
$\log(1+s) = s - \frac{s^2}{2} + \cdots$ with the bound
$\abs{\log(1+s) - s} \leq s^2$ for $s \in [0,1/2]$.
This gives
\begin{equation}\label{eq:log1sc}
  \log(1+s_c) = s_c + O(s_c^2)
  = \frac{B_c}{A_c}\,e^{-2r_i} + O(e^{-4r_i}).
\end{equation}
Substituting~\eqref{eq:log1sc} into~\eqref{eq:log2uc_split}:
\begin{equation}\label{eq:log2uc_final}
  \log(2u_c) = r_i + \log A_c
               + \frac{B_c}{A_c}\,e^{-2r_i}
               + O(e^{-4r_i}).
\end{equation}
Recall from the Appendix [\ref{proof:4.1}]
that
\(
  \varepsilon(u)
  = \log\!\left(\frac{1+\sqrt{1-1/u^2}}{2}\right).
\)
We first compute $1/u_c^2$.  From~\eqref{eq:uc_recall},
$u_c = \frac{e^{r_i}A_c}{2}(1+s_c)$, so
\begin{equation}\label{eq:uc_sq_inv}
  \frac{1}{u_c^2}
  = \frac{4}{e^{2r_i}\,A_c^2\,(1+s_c)^2}.
\end{equation}
Since $0 < s_c \leq 1/16$, we expand
$(1+s_c)^{-2} = 1 - 2s_c + O(s_c^2) = 1 + O(s_c)$,
giving
\begin{equation}\label{eq:delta_c}
  \delta_c \;:=\; \frac{1}{u_c^2}
  \;=\; \frac{4}{A_c^2}\,e^{-2r_i} + O(e^{-4r_i}).
\end{equation}
Now we expand $\varepsilon$ in powers of $\delta_c$.
First, $\sqrt{1-\delta} = 1 - \frac{\delta}{2} + O(\delta^2)$
for $\delta \to 0$:
\begin{equation}
  \sqrt{1 - \delta_c}
  = 1 - \frac{2}{A_c^2}\,e^{-2r_i} + O(e^{-4r_i}).
\end{equation}
Hence
\begin{equation}
  \frac{1+\sqrt{1-\delta_c}}{2}
  = 1 - \frac{1}{A_c^2}\,e^{-2r_i} + O(e^{-4r_i}).
\end{equation}
Next, $\log(1-s) = -s + O(s^2)$ for $s \to 0$, with
$s = \frac{1}{A_c^2}\,e^{-2r_i}$, we get:
\begin{equation}\label{eq:eps_final}
  \varepsilon(u_c)
  = -\frac{1}{A_c^2}\,e^{-2r_i} + O(e^{-4r_i}).
\end{equation}
Adding~\eqref{eq:log2uc_final} and~\eqref{eq:eps_final} we can get:
\begin{equation}\label{eq:dist_2nd}
  d_{\HH}(\mathbf{x}_i,\mathbf{w}_c)
  = r_i + \underbrace{\log A_c}_{= b_{\bm{\xi}_i}(\mathbf{w}_c)}
    + \underbrace{\left(\frac{B_c}{A_c}
      - \frac{1}{A_c^2}\right)}_{:=\alpha_{\bm{\xi}_i}(\mathbf{w}_c)}
      e^{-2r_i}
    + O(e^{-4r_i}).
\end{equation}

We combine the two fractions over a common denominator:
\begin{equation}\label{eq:alpha_simplify_step1}
  \alpha_{\bm{\xi}_i}(\mathbf{w}_c)
  = \frac{B_c}{A_c} - \frac{1}{A_c^2}
  = \frac{A_c\,B_c - 1}{A_c^2} = \frac{\norm{\wbar_c}^2
          - (\xhat_i \cdot \wbar_c)^2}
         {(w_{c,0} - \xhat_i\cdot\wbar_c)^2}
  \;\geq\; 0.
\end{equation}
Denoting by 
$\wbar_c^\perp := \wbar_c - (\xhat \cdot \wbar_c)\,\xhat$
the component of $\wbar_c$ orthogonal to the boundary
direction~$\xhat$, and recalling that
$A_c = -\ipL{\mathbf{w}_c}{\bm{\xi}}$ with
$\bm{\xi} = (1, \xhat)$, we obtain the compact form
\begin{equation}\label{eq:alpha_compact}
  \alpha_{\bm{\xi}}(\mathbf{w}_c)
  := \frac{\norm{\wbar_c^\perp}^2}
             {\ipL{\mathbf{w}_c}{\bm{\xi}}^2},
\end{equation}
with equality if and only if $\wbar_c \parallel \xhat$
, which means that the prototype points is in the same direction as the
boundary point or $\mathbf{w}_c = \mathbf{o}$.
We collect these coefficients into a vector:
\begin{equation}\label{eq:alpha_vec}
  \boldsymbol{\alpha}(\bm{\xi}_i)
  := \bigl(\alpha_{\bm{\xi}_i}(\mathbf{w}_1),\;
     \ldots,\;
     \alpha_{\bm{\xi}_i}(\mathbf{w}_C)\bigr)^\top
  \in \R^C.
\end{equation}

By the PERM loss definition in~\eqref{eq:template},
the per-sample loss is
\begin{equation}\label{eq:perm_recall}
  \ell_i(\mathbf{x}_i, y_i)
  = \psi\!\bigl(\Upsilon_{y_i}\,D\,\mathbf{v}_i\bigr),
\end{equation}
where the score vector
$\mathbf{v}_i = (v_{i,1},\ldots,v_{i,C})^\top \in \R^C$
has entries
\(
  v_{i,c} = -d_{\HH}(\mathbf{x}_i, \mathbf{w}_c).
\)
Substituting the second-order expansion~\eqref{eq:dist_2nd}
for each component:
\begin{align}
  v_{i,c}
  &= -d_{\HH}(\mathbf{x}_i, \mathbf{w}_c) \nonumber\\
  &= -r_i - b_{\bm{\xi}_i}(\mathbf{w}_c)
     - \alpha_{\bm{\xi}_i}(\mathbf{w}_c)\,e^{-2r_i}
     + O(e^{-4r_i}).
     \label{eq:logit_component}
\end{align}
Recall from Definition~\ref{def:bus_loss} that
$\mathbf{b}(\bm{\xi}_i)_c
= -b_{\bm{\xi}_i}(\mathbf{w}_c)$.
Therefore $-b_{\bm{\xi}_i}(\mathbf{w}_c)
= \mathbf{b}(\bm{\xi}_i)_c$, and we can rewrite
\eqref{eq:logit_component} in vector form:
\begin{equation}\label{eq:logit_vec}
  \mathbf{v}_i
  = -r_i\,\mathbf{1}
    + \mathbf{b}(\bm{\xi}_i)
    - \boldsymbol{\alpha}(\bm{\xi}_i)\,e^{-2r_i}
    + O(e^{-4r_i}),
\end{equation}
where $\mathbf{1} = (1,\ldots,1)^\top \in \R^C$.
Now we apply $\Upsilon_{y_i}\,D$ to both sides
of~\eqref{eq:logit_vec}. The key algebraic property
of the relative-margin matrix~$D$ is
\begin{equation}\label{eq:D1_zero}
  D\,\mathbf{1} = \mathbf{0}.
\end{equation}
This holds because $D$ extracts pairwise differences
$v_y - v_c$: adding any constant to all logits leaves
all differences unchanged.  Therefore:
\begin{align}
  \Upsilon_{y_i}\,D\,\mathbf{v}_i
  &= -r_i\,
     \underbrace{\Upsilon_{y_i}\,D\,\mathbf{1}}_{=\,\mathbf{0}}
     \;+\; \Upsilon_{y_i}\,D\,\mathbf{b}(\bm{\xi}_i)
     \;-\; \Upsilon_{y_i}\,D\,\boldsymbol{\alpha}(\bm{\xi}_i)
           \,e^{-2r_i}
     \;+\; O(e^{-4r_i})
     \nonumber\\[6pt]
  &= \underbrace{\Upsilon_{y_i}\,D\,
     \mathbf{b}(\bm{\xi}_i)}_{
       =:\;\mathbf{z}_0\;\in\;\R^{C-1}}
     \;-\; \underbrace{\Upsilon_{y_i}\,D\,
           \boldsymbol{\alpha}(\bm{\xi}_i)}_{
       =:\;\mathbf{h}\;\in\;\R^{C-1}}
     \,e^{-2r_i}
     \;+\; O(e^{-4r_i}).
     \label{eq:perm_arg}
\end{align}

By Assumption~\ref{assump:perm}, the template $\psi : \R^{C-1} \to \R$ is twice differentiable.  We take the Taylor expansion of $\psi$ around the
point $\mathbf{z}_0 \in \R^{C-1}$ and step
$\boldsymbol{\delta} \in \R^{C-1}$:
\begin{equation}\label{eq:taylor_general}
  \psi(\mathbf{z}_0 + \boldsymbol{\delta})
  = \psi(\mathbf{z}_0)
    + \ip{\nabla\psi(\mathbf{z}_0)}{\boldsymbol{\delta}}
    + O\!\bigl(\norm{\boldsymbol{\delta}}^2\bigr).
\end{equation}
From~\eqref{eq:perm_arg}, the argument of $\psi$ is
$\mathbf{z}_0 + \boldsymbol{\delta}$, where the
step is
\begin{equation}\label{eq:delta_def}
  \boldsymbol{\delta}
  := -\mathbf{h}\,e^{-2r_i} + O(e^{-4r_i}).
\end{equation}
Its norm satisfies
$\norm{\boldsymbol{\delta}} = O(e^{-2r_i})$, and hence
$\norm{\boldsymbol{\delta}}^2 = O(e^{-4r_i})$.

Substitute~\eqref{eq:delta_def}
into~\eqref{eq:taylor_general}:
\begin{align}
  \psi(\mathbf{z}_0 + \boldsymbol{\delta})
  &= \psi(\mathbf{z}_0)
     + \bigl\langle \nabla\psi(\mathbf{z}_0),\;
       -\mathbf{h}\,e^{-2r_i} + O(e^{-4r_i})
       \bigr\rangle
     + O(e^{-4r_i})
     \nonumber\\[4pt]
  &= \psi(\mathbf{z}_0)
     - \ip{\nabla\psi(\mathbf{z}_0)}{\mathbf{h}}
       \,e^{-2r_i}
     + O(e^{-4r_i}),
     \label{eq:psi_expanded}
\end{align}
where in the second line we used the fact that
$\ip{\nabla\psi(\mathbf{z}_0)}{O(e^{-4r_i})}
= O(e^{-4r_i})$  since $\nabla\psi(\mathbf{z}_0)$ is
a fixed vector, independent of~$r_i$.
Adding~\eqref{eq:perm_recall} and~\eqref{eq:psi_expanded} we get:
\begin{equation}\label{eq:loss_final}
  \ell_i(\mathbf{x}_i, y_i)
  = \psi(\mathbf{z}_0)
    - \ip{\nabla\psi(\mathbf{z}_0)}{\mathbf{h}}\,e^{-2r_i}
    + O(e^{-4r_i}).
\end{equation}
And we substitute the definition
$\mathbf{z}_0 = \Upsilon_{y_i}\,D\,\mathbf{b}(\bm{\xi}_i)$:
\begin{equation}
  \psi(\mathbf{z}_0)
  = \psi\!\bigl(\Upsilon_{y_i}\,D\,
    \mathbf{b}(\bm{\xi}_i)\bigr)
  = \mathcal{L}^{\mathrm{BS}}(\bm{\xi}_i, y_i),
\end{equation}
which by (Definition~\ref{def:bus_loss}), it is the Busemann loss. Then we 
substitute
$\mathbf{h} = \Upsilon_{y_i}\,D\,
\boldsymbol{\alpha}(\bm{\xi}_i)$ into the coefficient
of $e^{-2r_i}$:
\begin{equation}
  -\ip{\nabla\psi(\mathbf{z}_0)}{\mathbf{h}}
  = -\bigl\langle
    \nabla\psi\!\bigl(\Upsilon_{y_i}\,D\,
      \mathbf{b}(\bm{\xi}_i)\bigr),\;
    \Upsilon_{y_i}\,D\,\boldsymbol{\alpha}(\bm{\xi}_i)
    \bigr\rangle
  =: \mu(\bm{\xi}_i, y_i).
\end{equation}
This is the \emph{radial drift coefficient} defined
in~\eqref{eq:beta_def_v22}.

Then we obtain:
\[
  \ell_i(\mathbf{x}_i, y_i)
  = \mathcal{L}^{\mathrm{BS}}(\bm{\xi}_i, y_i)
    + \mu(\bm{\xi}_i, y_i)\,e^{-2r_i}
    + O(e^{-4r_i}),
\]
valid for all $r_i \geq R_0(\mathbf{W})
= \max_{c \in [C]}
\log\bigl(4(w_{c,0}+\norm{\wbar_c})\bigr)$.
This completed the proof.
\end{proof}

\newpage

\section{Proof of Section~\ref{section 5}}
\subsection{Proof of Theorem~\ref{thm:radial}}\label{proof: 5.2}

\textbf{Theorem~\ref{thm:radial}}\quad (Radial dichotomy)
\textit{Under Assumption~(\ref{assump:perm}), for any drift lower bound
\(\underline{\mu}>0\), there exists
\(R_0=R_0(\mathbf{W},\underline{\mu})>0\) such that:
\begin{enumerate}[label=\textup{(\alph*)}]
\item \textbf{Boundary divergence.}\;
  If there exist \(T_2>T_1\geq0\) such that \(r(T_1)\ge R_0\) and
  \(\mu(\bm{\xi}(t),y)\ge\underline{\mu}\) on the time interval
  \(I=[T_1,T_2]\), then
  \(
    r(t)\ge\tfrac{1}{2}\log\!\bigl(
      e^{2r(T_1)}+2\underline{\mu}(t-T_1)
    \bigr)
  \)
  for all \(t\in I\).
  Moreover, if \(r(T_1)\ge R_0\) and
  \(\mu(\bm{\xi}(t),y)\ge\underline{\mu}\) for all \(t\ge T_1\), then
  \begin{equation*}
    r(t)=\tfrac{1}{2}\log t+O(1),
    \qquad t\to\infty.
  \end{equation*}
\item \textbf{Boundary repulsion.}\;
  If there exist \(T_2>T_1\geq0\) such that
  \(\mu(\bm{\xi}(t),y)\le-\underline{\mu}\) and \(r(t)\geq R_0\) for all
  \(t\in I=[T_1,T_2]\), then
  \(
    r(t)\le\tfrac{1}{2}\log\!\bigl(
      e^{2r(T_1)}-2\underline{\mu}(t-T_1)
    \bigr)
  \)
  for all \(t\in I\).
  Moreover, if \(r(T_1)\ge R_0\) and
  \(\mu(\bm{\xi}(t),y)\le-\underline{\mu}\)
  whenever \(r(t)\ge R_0\), then the first hitting time
  \(T^*:=\inf\{t\ge T_1:\, r(t)\le R_0\}\) satisfies
  \begin{equation*}
    T^*\le T_1+
    \frac{e^{2r(T_1)}-e^{2R_0}}{2\underline{\mu}}
    <\infty.
  \end{equation*}
\end{enumerate}
}

\begin{proof}

Firstly recall that any point $\mathbf{x} \in \Hn \setminus \{\mathbf{o}\}$
can be written in geodesic polar coordinates as
$\mathbf{x} = (\cosh r,\, \sinh r\,\xhat)$ with $r > 0$
and $\xhat \in \Sn$.  In these coordinates, the
Riemannian metric of $\Hn$ takes the form
\begin{equation}\label{eq:polar_metric}
  ds^2 = dr^2 + \sinh^2\!(r)\; d\xhat^2,
\end{equation}
where $d\xhat^2$ is the round metric on $\Sn$.
For any smooth function \(f:\Hn\to\R\), the Riemannian
gradient in polar coordinates is
\begin{equation}\label{eq:grad_polar}
  \grad_{\Hn} f
  =
  \frac{\partial f}{\partial r}\mathbf e_r
  +
  \frac{1}{\sinh^2\!(r)}\,
  \grad_{\Sn} f,
\end{equation}
where \(\mathbf e_r\) denotes the unit radial vector field. Since the prototypes $\{\mathbf{w}_c\}$ are fixed, the
loss $\ell_i$ depends only on $\mathbf{x}_i$.  The
Riemannian gradient flow is
\begin{equation}\label{eq:grad_flow_recall}
  \dot{\mathbf{x}}_i(t)
  = -\grad_{\Hn}\,\ell_i(\mathbf{x}_i(t), y_i).
\end{equation}
Projecting onto the radial direction using
\eqref{eq:grad_polar}:
\begin{equation}\label{eq:rdot_general}
  \dot{r}(t) = -\frac{\partial \ell_i}{\partial r}\,
               \bigg|_{\xhat = \xhat(t)}.
\end{equation}
From Proposition~\ref{prop:loss_expansion}, for
$r \geq R_0(\mathbf{W})$:
\begin{equation}\label{eq:loss_exp_recall}
  \ell_i = \mathcal{L}^{\mathrm{BS}}(\bm{\xi}, y_i)
           + \mu(\bm{\xi}, y_i)\,e^{-2r}
           + O(e^{-4r}).
\end{equation}
When we differentiate with respect to $r$ while holding
$\xhat$ fixed. Since the first term
      $\mathcal{L}^{\mathrm{BS}}(\bm{\xi}, y_i)$ depends
      only on the direction $\bm{\xi} = (1, \xhat)$, not
      on~$r$, so
      $\frac{\partial}{\partial r}
       \mathcal{L}^{\mathrm{BS}} = 0$.
The second term $\mu = \mu(\bm{\xi}, y_i)$ also does not depend
      on~$r$:
      $\frac{\partial}{\partial r}
       \bigl[\mu\, e^{-2r}\bigr]
       = \mu \cdot (-2)\, e^{-2r}
       = -2\mu\, e^{-2r}$.
       The remainder:
      $\frac{\partial}{\partial r}
       O(e^{-4r}) = O(e^{-4r})$.
Substituting into~\eqref{eq:rdot_general}:
\begin{equation}\label{eq:rdot_explicit}
 \dot{r}(t) = 2\mu(t)\,e^{-2r(t)}
                      + O(e^{-4r(t)}).
\end{equation}
 The sign of $\mu(t)$
determines whether the sample is pushed toward the
boundary ($\mu > 0$) or pulled back toward the interior
($\mu < 0$).

\textbf{Proof of part~(a):
boundary divergence.}

Firstly, the $O(e^{-4r})$ in~\eqref{eq:rdot_explicit} can be
bounded that there exists a constant
$M = M(\mathbf{W}) > 0$, depending only on the
prototypes, such that for all $r \geq R_0$ and all
$\xhat \in \Sn$:
\begin{equation}\label{eq:remainder_bound}
  \abs{\dot{r}(t) - 2\mu(t)\,e^{-2r(t)}}
  \leq M\,e^{-4r(t)}.
\end{equation}
The constant $M$ can be extracted from the explicit
second-order expansion in
Proposition~\ref{prop:loss_expansion}.

Assume $r(T_1) \geq R_0$ and
$\mu(t) \geq \underline{\mu} > 0$ on some interval $I=[T_1,T_2]$.
From~\eqref{eq:rdot_explicit}
and~\eqref{eq:remainder_bound}:
\begin{equation}\label{eq:rdot_lower_raw}
  \dot{r}(t) \geq 2\underline{\mu}\,e^{-2r(t)}
                   - M\,e^{-4r(t)}
              = e^{-2r(t)}\bigl(
                2\underline{\mu} - M\,e^{-2r(t)}
                \bigr).
\end{equation}
We can choose $R_0 \geq \frac{1}{2}\log(M/\underline{\mu})$ so that
$M\,e^{-2R_0} \leq \underline{\mu}$.  
Then for all $r(t) \geq R_0$:
\begin{equation}\label{eq:rdot_lower}
  \dot{r}(t) \geq e^{-2r(t)}\bigl(
                  2\underline{\mu} - \underline{\mu}\bigr)
              = \underline{\mu}\,e^{-2r(t)}
              > 0.
\end{equation}
So $r(t)$ is strictly increasing whenever $r(t) \geq R_0$.
Since $r(T_1) \geq R_0$, we have $r(t) \geq R_0$ for all
$t \in I$, and $\dot{r}(t) > 0$ for all $t \in I$, which means once the trajectory is above $R_0$, it never
comes back down on this interval.

Then we try to prove $r(t) \to \infty$ under the global condition.
Suppose $\mu(t) \geq \underline{\mu} > 0$ for all $t\geq T_1$.
Suppose for contradiction that $r(t) \to r_B < \infty$
as $t \to \infty$.  Since $r(t)$ is increasing and $r(t) \to r_B$,
we have $r(t) \leq r_B$ for all $t \geq T_1$, and
hence $e^{-2r(t)} \geq e^{-2r_B}$.
Substituting into~\eqref{eq:rdot_lower} and integrating
from $T_1$ to $T$:
\begin{equation}
  r(T) - r(T_1)
  = \int_{T_1}^T \dot{r}(t)\,dt
  \geq \int_{T_1}^T \underline{\mu}\,e^{-2r(t)}\,dt
  \geq \underline{\mu}\,e^{-2r_B}\cdot (T-T_1).
\end{equation}
Sending $T \to \infty$, the right-hand side diverges,
so $r(T) \to \infty$, contradicting
$r(T) \to r_B < \infty$.

Therefore $r(t) \to \infty$.

Then we prove the rate of divergence. Firstly, we prove the 
lower bound that
$r(t) \geq \frac{1}{2}\log t + O(1)$.

We multiply both sides of~\eqref{eq:rdot_lower}
by $e^{2r(t)} > 0$:
\[
  e^{2r(t)}\,\dot{r}(t)
  \;\geq\; \underline{\mu}\,
           e^{-2r(t)}\,e^{2r(t)}
  \;=\; \underline{\mu}.
\]
The left-hand side is a total derivative.  By the
chain rule,
\(
  \frac{d}{dt}\bigl[e^{2r(t)}\bigr]
  = 2\,e^{2r(t)}\,\dot{r}(t),
\)
so the inequality becomes
\begin{equation}\label{eq:f_lower_ineq}
  \frac{d}{dt}\bigl[e^{2r(t)}\bigr]
  \;\geq\; 2\underline{\mu}.
\end{equation}
This says that $e^{2r(t)}$ grows at least as fast as a
linear function.  Integrating both sides from $T_1$ to~$T$:
\begin{equation}
  e^{2r(T)} - e^{2r(T_1)}
  \;=\; \int_{T_1}^T
        \frac{d}{dt}\bigl[e^{2r(t)}\bigr]\,dt
  \;\geq\; \int_{T_1}^T 2\underline{\mu}\,dt
  \;=\; 2\underline{\mu}\,(T-T_1).
\end{equation}
Rearranging:
\begin{equation}\label{eq:e2r_lower}
  e^{2r(T)} \;\geq\; 2\underline{\mu}\,(T-T_1) + e^{2r(T_1)}.
\end{equation}
Taking logarithms  and
dividing by~$2$:
\begin{equation}\label{eq:r_lower_v2}
  r(T) \;\geq\; \frac{1}{2}
        \log\!\bigl(2\underline{\mu}\,(T-T_1) + e^{2r(T_1)}\bigr).
\end{equation}
For $T \geq T_1+1$, since
$2\underline{\mu}\,(T-T_1) + e^{2r(T_1)} \geq 2\underline{\mu}\,(T-T_1)$:
\begin{equation}\label{eq:r_lower_asymp}
  r(T) \;\geq\; \frac{1}{2}\log(2\underline{\mu}\,(T-T_1))
  \;=\; \frac{1}{2}\log T
        + O(1).
\end{equation}

Then we prove the upper bound that
$r(t) \leq \frac{1}{2}\log t + O(1)$.

Since $\mu(t) \geq \underline{\mu} > 0$ and $\xhat(t)$
stays in the compact set $\Sn$, there exists a constant
$\bar{\mu} := \sup_{\xhat \in \Sn} \mu(\bm{\xi}, y_i)
< \infty$ (by continuity of $\mu$ on the compact
set~$\Sn$).  From~\eqref{eq:rdot_explicit}
and~\eqref{eq:remainder_bound}:
\begin{equation}\label{eq:rdot_upper_raw}
  \dot{r}(t) \leq 2\bar{\mu}\,e^{-2r(t)}
                   + M\,e^{-4r(t)}.
\end{equation}
For $r(t) \geq R_0\geq \frac{1}{2}\log(M/\bar{\mu})$, we have
$M\,e^{-4r} \leq M\,e^{-2R_0}\,e^{-2r}
\leq \bar{\mu}\,e^{-2r}$,
therefore:
\begin{equation}\label{eq:rdot_upper}
  \dot{r}(t) \leq 3\bar{\mu}\,e^{-2r(t)}.
\end{equation}

From~\eqref{eq:rdot_upper},
$\dot{r}(t) \leq 3\bar{\mu}\,e^{-2r(t)}$.
Similarly, multiplying both sides by $e^{2r(t)}$:
\begin{equation}
  e^{2r(t)}\,\dot{r}(t) \;\leq\; 3\bar{\mu},
\end{equation}
hence
for $T \geq T_1+1$:
\begin{equation}\label{eq:r_upper_asymp}
  r(T) \;\leq\; \frac{1}{2}
        \log\!\bigl(6\bar{\mu}\,(T-T_1) + e^{2r(T_1)}\bigr)
  \;\leq\; \frac{1}{2}\log T
           + O(1).
\end{equation}

Finally, from~\eqref{eq:r_lower_asymp}
and~\eqref{eq:r_upper_asymp}, there exist constants
$c_1, c_2$ (depending on $\underline{\mu}$, $\bar{\mu}$,
$r(T_1)$, and $T_1$) such that for all $t \geq T_1+1$:
\begin{equation}
  \frac{1}{2}\log t + c_1
  \;\leq\; r(t)
  \;\leq\; \frac{1}{2}\log t + c_2.
\end{equation}
This gives $r(t) = \frac{1}{2}\log t + O(1)$, proving
part~(a).

\textbf{Proof of part~(b):
boundary repulsion.}

Assume $r(T_1)\geq R_0$ and
$\mu(t) \leq -\underline{\mu} < 0$
whenever $r(t) \geq R_0$.

From~\eqref{eq:rdot_explicit}
and~\eqref{eq:remainder_bound}, whenever
$r(t) \geq R_0$:
\begin{equation}\label{eq:rdot_repel}
  \dot{r}(t) \leq -2\underline{\mu}\,e^{-2r(t)}
                   + M\,e^{-4r(t)}
              = e^{-2r(t)}\bigl(
                -2\underline{\mu} + M\,e^{-2r(t)}\bigr).
\end{equation}
Also, we can choose $R_0 \geq \frac{1}{2}\log(M/\underline{\mu})$ so that
$M\,e^{-2R_0} \leq \underline{\mu}$. Then for all
$r(t) \geq R_0$:
\begin{equation}\label{eq:rdot_repel_clean}
  \dot{r}(t) \leq -\underline{\mu}\,e^{-2r(t)} < 0.
\end{equation}
So $r(t)$ is \emph{strictly decreasing} whenever
$r(t) \geq R_0$.
Similarly, multiplying both sides by $e^{2r(t)}$:
\begin{equation}\label{eq:fdot}
  2\dot{r}(t)\,e^{2r(t)}
  \leq 2\bigl(-\underline{\mu}\,e^{-2r(t)}\bigr)
       \,e^{2r(t)}
  = -2\underline{\mu}.
\end{equation}
Integrating from $T_1$ to $t$:
\begin{equation}\label{eq:f_decrease}
  e^{2r(t)} \leq e^{2r(T_1)} - 2\underline{\mu}\,(t-T_1)
       .
\end{equation}
The condition $e^{2r(t)} \leq e^{2R_0}$ (i.e., $r(t) \leq R_0$)
is achieved when
\begin{equation}\label{eq:Tstar}
  e^{2r(T_1)} - 2\underline{\mu}\,(t-T_1) \leq e^{2R_0},
  \qquad
  t \geq T^* := T_1+
  \frac{e^{2r(T_1)} - e^{2R_0}}{2\underline{\mu}}.
\end{equation}
Since $r(T_1) \geq R_0$, we have $T^* \geq T_1$, and
$T^* < \infty$.  So $r(t)$ drops below $R_0$ no later
than time~$T^*$.

This completed the proof of part~(b), and with it, we completed the proof of the whole
theorem.
\end{proof}
\newpage

\subsection{Proof of Theorem~\ref{thm:three_phase}}
\label{proof: 6.5}

\textbf{Theorem~\ref{thm:three_phase}}
\textit{Under Assumptions~(\ref{assump:perm}) and
~(\ref{assump:noncollinear}), with $\psi$ real analytic,
suppose $\bm\xi(T_0)$ lies in the Busemann gradient-flow
basin of a nondegenerate local minimum $\bm\xi^*$,
with $\mu(\bm\xi(T_0),y)>0$.
For a fixed initial angular direction, there exist
$R_0>0$ and constants $C,\lambda,\kappa>0$
such that, for any $R\ge R_0$, if
$r(T_0)\ge 2R$, then there exist finite times
$T_0<T_1<T_2<T_3$ satisfying:}

\noindent\textbf{(a) Large-radius excursion.}\;
\textit{The trajectory initially moves outward,
$\dot r(T_0)>0$, while
$\|\dot{\xhat}(t)\|=O(e^{-2r(t)})$
in the large-radius regime.}

\smallskip
\noindent\textbf{(b) Angular alignment and drift reversal.}\;
\textit{At $T_1$, $\mu(\bm\xi(T_1),y)=0$,
and $\mu(\bm\xi(t),y)<0$ for $t\in(T_1,T_3]$.
At $T_2$, the trajectory first enters the
sufficiently aligned region
$\mathcal A(\bm\xi^*)$.}

\smallskip
\noindent\textbf{(c) Slow radial return.}\;
\textit{For $t\in[T_2,T_3]$, $\dot r(t)<0$
and $|\dot r(t)|=O(e^{-2r(t)})$, with
$r(T_3)=R$.}
\textit{For any $\varepsilon>0$,
the returning trajectory reaches
\begin{equation*}
    d_{\Sn}(\bm\xi(T_3),\bm\xi^*)\le\varepsilon
\end{equation*}
in finite time by choosing $R$ sufficiently large.}

\begin{proof}
Firstly recall that, in geodesic polar coordinates
$x=(\cosh r,\sinh r\,\xhat)$, the Riemannian gradient
flow satisfies
\begin{equation}
\dot r=-\partial_r\ell,
\qquad
\dot{\xhat}
=
-\frac{1}{\sinh^2r}\grad_{\Sn}\ell .
\label{eq:three_phase_polar}
\end{equation}
From Proposition~\ref{prop:loss_expansion},
$\ell(x,y)=\mathcal L^{\mathrm{BS}}(\bm\xi,y)
+\mu(\bm\xi,y)e^{-2r}+O(e^{-4r})$.
Differentiating gives
\begin{equation}
\dot r
=
2\mu(\bm\xi,y)e^{-2r}+O(e^{-4r}),
\qquad
\dot{\xhat}
=
-\frac{1}{\sinh^2r}
\left(
\grad_{\Sn}\mathcal L^{\mathrm{BS}}(\bm\xi,y)
+O(e^{-2r})
\right).
\label{eq:three_phase_flow}
\end{equation}

We first establish the sign of the radial drift at
a Busemann critical direction.
Write $g_c=-\partial_c\psi>0$ at the Busemann
margins and $b_c=b_{\bm\xi}(w_c)$.
By the definitions of $\mu$ and
$\mathcal L^{\mathrm{BS}}$,
\[
\mu=\sum_{c\ne y}g_c
\left(\|\grad_{\Sn}b_y\|^2-\|\grad_{\Sn}b_c\|^2\right),
~\
\grad_{\Sn}\mathcal L^{\mathrm{BS}}
=\sum_{c\ne y}g_c(\grad_{\Sn}b_y-\grad_{\Sn}b_c).
\]
Consequently, at any Busemann critical direction
$\bm\xi^*$,
\begin{equation}
\mu(\bm\xi^*,y)
=
-\sum_{c\ne y}g_c
\left\|
\grad_{\Sn}b_{\bm\xi^*}(w_c)
-
\grad_{\Sn}b_{\bm\xi^*}(w_y)
\right\|^2<0.
\label{eq:three_phase_mu_negative}
\end{equation}
Indeed, equality would require identical Busemann
gradients for all prototypes, placing them on a common
hyperbolic geodesic asymptotic to $\bm\xi^*$,
contrary to Assumption~(\ref{assump:noncollinear}).
Since $\bm\xi^*$ is a nondegenerate local minimum,
we may choose $\mathcal A(\bm\xi^*)$ sufficiently
small that $\mu(\bm\xi,y)\le-m<0$ and
$\operatorname{Hess}_{\Sn}
\mathcal L^{\mathrm{BS}}\succeq\lambda I$
throughout this region, for some $m,\lambda>0$.

We next compare the angular dynamics with Busemann
gradient flow. Introduce the rescaled time
$s(t)=\int_{T_0}^t\sinh^{-2}r(\tau)\,d\tau$.
Equation~\eqref{eq:three_phase_flow} becomes
\begin{equation}
\frac{dr}{ds}
=
\frac12\mu(\bm\xi,y)+O(e^{-2r}),
\qquad
\frac{d\xhat}{ds}
=
-\grad_{\Sn}\mathcal L^{\mathrm{BS}}(\bm\xi,y)
+O(e^{-2r}).
\label{eq:three_phase_rescaled}
\end{equation}
Let $\bm\zeta(s)$ solve the unperturbed Busemann
gradient flow starting from $\bm\xi(T_0)$.
By the basin-of-attraction hypothesis,
$\bm\zeta(s)\to\bm\xi^*$.
Hence there exists a finite $S>0$ such that
$d_{\Sn}(\bm\zeta(S),\bm\xi^*)<\theta_*/2$.

Since $\mu$ is uniformly bounded, the rescaled
radial velocity is bounded by some $M>0$ in the
large-radius regime. Thus $r(s)\ge r(T_0)-Ms$.
For sufficiently large $R$, the trajectory remains
above $R$ on $[0,S]$, and the angular perturbation
is uniformly $O(e^{-2R})$.
By smoothness and a Gronwall estimate,
\begin{equation}
\sup_{0\le s\le S}
d_{\Sn}(\bm\xi(s),\bm\zeta(s))
\le C_S e^{-2R}.
\label{eq:three_phase_shadowing}
\end{equation}
Choosing $R$ sufficiently large, the full angular
trajectory therefore enters $\mathcal A(\bm\xi^*)$
within finite rescaled time, while remaining at
large radius.

We now prove the three parts.

\paragraph{Proof of part (a): large-radius excursion.}
At $T_0$, $\mu(\bm\xi(T_0),y)>0$.
From \eqref{eq:three_phase_flow}, for sufficiently
large $r(T_0)$, the positive leading term dominates
the remainder, giving $\dot r(T_0)>0$.
Thus the trajectory initially moves outward.

Moreover, smoothness of $\mathcal L^{\mathrm{BS}}$
and compactness of $\Sn$ imply that its spherical
gradient is uniformly bounded. Hence
\[
\|\dot{\xhat}(t)\|
=
O(\sinh^{-2}r(t))
=
O(e^{-2r(t)})
\]
throughout the large-radius regime.
Intermediate drift reversals are allowed.
This proves part~(a).

\paragraph{Proof of part (b): angular alignment and drift reversal.}
Let $T_2$ be the first entrance time into
$\mathcal A(\bm\xi^*)$.
The finite-time approximation above shows that
$T_2<\infty$ and, since $s(T_2)\le S$,
\[
r(T_2)\ge 2R-MS>R
\]
for sufficiently large $R$.

By \eqref{eq:three_phase_mu_negative} and the
definition of $\mathcal A(\bm\xi^*)$,
$\mu(\bm\xi(T_2),y)\le-m<0$.
Together with $\mu(\bm\xi(T_0),y)>0$,
continuity gives at least one zero crossing
on $(T_0,T_2)$.
Let $T_1$ be the last such crossing.
Then $T_0<T_1<T_2$ and
$\mu(\bm\xi(t),y)<0$ for $t\in(T_1,T_2]$.

Furthermore, positive definiteness of the local
Busemann Hessian and
\eqref{eq:three_phase_rescaled} yield
\begin{equation}
\frac{d}{ds}
d_{\Sn}(\bm\xi,\bm\xi^*)
\le
-\lambda d_{\Sn}(\bm\xi,\bm\xi^*)
+Ce^{-2r}.
\label{eq:three_phase_contraction}
\end{equation}
At zero distance, the inequality is understood
in the upper Dini derivative sense.
Choosing $R$ so that
$Ce^{-2R}<\lambda\theta_*$ makes
$\mathcal A(\bm\xi^*)$ forward invariant while
$r\ge R$.
Therefore $\mu<0$ persists throughout the
subsequent large-radius return.
This proves part~(b).

\paragraph{Proof of part (c): slow radial return.}
Within $\mathcal A(\bm\xi^*)$,
$\mu(\bm\xi,y)\le-m<0$.
From \eqref{eq:three_phase_rescaled}, for sufficiently
large $R$,
\[
-M\le\frac{dr}{ds}\le-\frac m4<0
\]
whenever $r\ge R$.
Thus the radius decreases strictly and reaches
$R$ in finite rescaled time.
Let $T_3$ denote the first hitting time.
Since $r$ remains bounded on $[T_2,T_3]$,
the corresponding original time is also finite.
Moreover, \eqref{eq:three_phase_flow} gives
$|\dot r(t)|=O(e^{-2r(t)})$.

It remains to establish the angular convergence.
Since $r$ decreases on $[T_2,T_3]$,
$e^{-2r(t)}\le e^{-2R}$.
Integrating \eqref{eq:three_phase_contraction}
therefore gives
\begin{equation}
d_{\Sn}(\bm\xi(t),\bm\xi^*)
\le
\theta_*e^{-\lambda(s(t)-s(T_2))}
+\frac{C}{\lambda}e^{-2R}.
\label{eq:three_phase_alignment_bound}
\end{equation}
On the other hand, the lower bound on the radial
velocity and $r(T_2)\ge2R-MS$ imply
\[
s(T_3)-s(T_2)
\ge
\frac{r(T_2)-R}{M}
\ge
\frac{R-MS}{M}.
\]
Substituting into
\eqref{eq:three_phase_alignment_bound} yields
\begin{equation}
d_{\Sn}(\bm\xi(T_3),\bm\xi^*)
\le
\theta_*e^{\lambda S}e^{-\lambda R/M}
+\frac{C}{\lambda}e^{-2R}
\le C'e^{-\kappa R},
\label{eq:three_phase_epsilon_limit}
\end{equation}
where $\kappa=\min\{\lambda/M,2\}>0$
and $C'>0$ is independent of $R$.

Consequently, for any $\varepsilon>0$,
choosing $R$ sufficiently large gives
\[
d_{\Sn}(\bm\xi(T_3),\bm\xi^*)
\le\varepsilon.
\]
Thus the returning trajectory approaches the
Busemann local minimum to arbitrary accuracy
in finite time, despite the decreasing radius.
This completes the proof.
\end{proof}
\newpage

\subsection{Proof of Theorem~\ref{thm:angular}}
\label{proof: 5.6}

\textbf{Theorem~\ref{thm:angular}}\quad
(Angular convergence to Busemann minimizers in the escaping regime)\;
\textit{Under
Assumptions~(\ref{assump:perm}) and (\ref{assump:ideal-prototypes}),
the Riemannian gradient flow satisfies
\begin{equation*}
  r(t)\longrightarrow\infty,
  \qquad
  \bm{\xi}(t)\longrightarrow\bm{\xi}^*
  ,
  \quad \text{where}\quad
  \bm{\xi}^*\in
  \arg\min_{\bm{\xi}\in\partial\Hn}
  \mathcal{L}^{\mathrm{BS}}(\bm{\xi},y).
\end{equation*}
 If Assumption~
(\ref{assump:ideal-prototypes}) holds for all $i\in[N]$,
then $(\bm{\xi}_1^*,\ldots,\bm{\xi}_N^*)$ is a
\textbf{global minimizer} of the Busemann risk
$\mathcal{R}^{\mathrm{BS}}_{\mathbf{W}^\infty}
  =\frac{1}{N}\sum_{i=1}^{N}
  \mathcal{L}^{\mathrm{BS}}(\bm{\xi}_i,y_i)$
on~$(\partial\Hn)^N$.}

\begin{proof}
We first consider a single sample with label $y$.
Under Assumption~(\ref{assump:ideal-prototypes}),
the class scores are
$v_c(\mathbf x)=-b_{w_c^\infty}(\mathbf x)$,
where the ideal prototypes are pairwise distinct.
For each ideal prototype
$w_c^\infty=(1,\hat w_c)$, consider
$w_c(R)=(\cosh R,\sinh R\,\hat w_c)\in\Hn$.
By the radial Busemann expansion,
\[
d_{\Hn}(x,w_c(R))
=R+b_{w_c^\infty}(x)+O(e^{-2R}).
\]
Since all prototypes have the same radius $R$,
the common term $-R$ cancels in the relative
margins of any PERM loss. Consequently,
the limiting class scores are
\[
v_c(x)=-b_{w_c^\infty}(x).
\]

Recall that $\mathbf x=(x_0,\bar{\mathbf x})\in\Hn$.
By a spatial rotation, we may assume that the
true-class ideal prototype is
$w_y^\infty=(1,1,0,\ldots,0)$.
Write $\bar{\mathbf x}=(x_1,\mathbf u)$ and
$s=x_0-x_1>0$. The hyperboloid constraint
$\langle\mathbf x,\mathbf x\rangle_M=-1$
then gives
$x_0+x_1=(1+\|\mathbf u\|^2)/s$.

For each $c\ne y$, write
$w_c^\infty=(1,\hat w_{c,1},
\hat{\mathbf w}_{c,\perp})$ and define
$a_c=\hat{\mathbf w}_{c,\perp}/
(1-\hat w_{c,1})$.
Since the ideal prototypes are distinct,
$1-\hat w_{c,1}>0$.
Using the Minkowski inner product, we obtain
$-\langle\mathbf x,w_y^\infty\rangle_M=s$
and
\begin{equation}\label{eq:ideal_margins}
m_c(\mathbf x)
=
b_{w_c^\infty}(\mathbf x)
-b_{w_y^\infty}(\mathbf x)
=
\log\frac{1+\|\mathbf u-sa_c\|^2}
{s^2(1+\|a_c\|^2)},
\qquad c\ne y.
\end{equation}

By Assumption~(\ref{assump:perm}),
$\ell(\mathbf x,y)=\psi(m)$.
Let $q_c=-\partial_c\psi(m)>0$.
For any ideal prototype $w_c^\infty$,
the Riemannian gradient of its Busemann
function is
$\grad_{\Hn}b_{w_c^\infty}(\mathbf x)
=\mathbf x-
w_c^\infty/(-\langle\mathbf x,
w_c^\infty\rangle_M)$.
Consequently, the full Riemannian
gradient flow satisfies
\begin{equation}\label{eq:ideal_gradient_flow}
\begin{aligned}
\dot s
&=-2s\sum_{c\ne y}
\frac{q_c}{1+\|\mathbf u-sa_c\|^2}<0,\\
\dot{\mathbf u}
&=-2s\sum_{c\ne y}
\frac{q_ca_c}{1+\|\mathbf u-sa_c\|^2}.
\end{aligned}
\end{equation}
Thus $s(t)$ is strictly decreasing.

We first verify global existence.
Since $0<s(t)\le s(0)$, the margins
in~\eqref{eq:ideal_margins} are uniformly
bounded below. By convexity,
nonnegativity, and coordinatewise
monotonicity of $\psi$, we have
$q_c\le\psi(m-e_c)$, where $e_c$
is the corresponding coordinate vector.
The lower bounds on the margins therefore
imply that all $q_c$ remain bounded.
Moreover, since each Busemann function
has unit Riemannian gradient norm,
$\|\grad_{\Hn}m_c\|\le2$.
Hence $\|\grad_{\Hn}\ell\|$ is uniformly
bounded. By completeness of $\Hn$,
the gradient flow exists for all $t\ge0$.

We next establish radial divergence.
Dividing the two equations
in~\eqref{eq:ideal_gradient_flow} yields
\begin{equation}\label{eq:ideal_ratio}
\frac{d\mathbf u}{ds}
=
\frac{\displaystyle\sum_{c\ne y}
\frac{q_ca_c}{1+\|\mathbf u-sa_c\|^2}}
{\displaystyle\sum_{c\ne y}
\frac{q_c}{1+\|\mathbf u-sa_c\|^2}}.
\end{equation}
The right-hand side is a weighted
average of the fixed vectors $a_c$,
and is therefore uniformly bounded.
Since $s(t)$ is positive and decreasing,
equation~\eqref{eq:ideal_ratio} implies
that $\mathbf u(t)$ remains bounded
and converges to a finite limit.

Suppose, for contradiction, that
$s(t)\to s_\infty>0$.
Then $s$ and $\mathbf u$ remain in a
compact region, and so does $\mathbf x$
by the hyperboloid constraint.
Since $q_c>0$ and the vector field is
continuous, the first equation
in~\eqref{eq:ideal_gradient_flow}
has a strictly negative upper bound
on this compact region.
This contradicts the convergence of
$s(t)$ to a positive limit.
Therefore $s(t)\to0$.

By the hyperboloid constraint,
$x_0+x_1=(1+\|\mathbf u\|^2)/s\to\infty$,
while $x_0-x_1=s\to0$.
Thus $x_0\to\infty$,
$x_1/x_0\to1$, and
$\mathbf u/x_0\to0$.
Recalling $x_0=\cosh r$ and
$\bar{\mathbf x}=\sinh r\,\xhat$,
we conclude that
\begin{equation}\label{eq:ideal_convergence}
r(t)\longrightarrow\infty,
\qquad
\bm\xi(t)\longrightarrow
\bm\xi^*\in\partial\Hn,
\quad t\to\infty.
\end{equation}

It remains to establish the minimizing
property of $\bm\xi^*$.
Recall that the Busemann loss retains
the original PERM form
$\mathcal L^{\mathrm{BS}}(\bm\xi,y)
=\psi(\Upsilon_yD\mathbf b(\bm\xi))$.
For directions distinct from the ideal
prototypes, the relative margins are
$\log((1-\xhat^\top\hat w_c)/
(1-\xhat^\top\hat w_y))$, $c\ne y$.

As $\bm\xi$ approaches the true-class
ideal direction, all relative margins
diverge to $+\infty$.
Since $\psi$ is nonnegative and strictly
decreasing in each coordinate,
$L_*=\lim_{a\to\infty}\psi(a\mathbf1)$
exists and equals the infimum of $\psi$.
The Busemann loss therefore extends
continuously to $\bm\xi^*$ with value
$L_*$.

For any direction with finite relative
margins, strict coordinatewise
monotonicity gives
$\mathcal L^{\mathrm{BS}}(\bm\xi,y)>L_*$.
At any other ideal prototype direction,
at least one relative margin tends to
$-\infty$, while the remaining margins
stay bounded above. Monotonicity then
also excludes the value $L_*$.
Consequently,
\begin{equation}
\bm\xi^*\in
\argmin_{\bm\xi\in\partial\Hn}
\mathcal L^{\mathrm{BS}}(\bm\xi,y).
\end{equation}

Finally, restoring the sample index $i$,
the same argument applies independently
to every trajectory. Since the Busemann
risk is separable across samples,
$\mathcal R_{\mathbf W^\infty}^{\mathrm{BS}}
=\frac1N\sum_{i=1}^N
\mathcal L^{\mathrm{BS}}(\bm\xi_i,y_i)$,
the joint limit
$(\bm\xi_1^*,\ldots,\bm\xi_N^*)$
is a global minimizer on
$(\partial\Hn)^N$.
This completes the proof.
\end{proof}

\newpage

\subsection{Proof of Theorem~\ref{thm:boundary}}
\label{proof: 5.7}

\textbf{Theorem~\ref{thm:boundary}}
\textit{Under Assumption~(\ref{assump:perm}), let
\(\{\mathbf w_c\}_{c=1}^C\) be fixed prototypes,
either all in \(\Hn\) or all in \(\partial\Hn\),
with distance or Busemann scores, respectively.
For a test point
\(\mathbf x=(\cosh r,\sinh r\,\xhat)\) with
\(\bm\xi=(1,\xhat)\in\partial\Hn\), define
\(
  \hat y(\mathbf x)
  :=
  \argmin_{y\in[C]}
  \psi\!\left(\Upsilon_yD\,\mathbf v(\mathbf x)\right).
\)
If \(\argmin_y\mathcal L^{\mathrm{BS}}(\bm\xi,y)\) is unique, then for all
sufficiently large \(r\),
\(
  \hat y(\mathbf x)
  =
  \argmin_{y\in[C]}\mathcal L^{\mathrm{BS}}(\bm\xi,y)
  =
  \argmin_{c\in[C]} b_{\bm\xi}(\mathbf w_c).
\)
For ideal prototypes, this equality holds for all \(r>0\).}

\textit{Moreover, the decision boundary between
classes $c$ and $c'$ on $\partial\Hn\cong\Sn$ is
\begin{equation*}
\begin{aligned}
\mathcal B_{c,c'}=\bigl\{\xhat\in\Sn:\;&
\xhat^\top(\wbar_c-\wbar_{c'})
=w_{c,0}-w_{c',0},\quad b_{\bm\xi}(\mathbf w_c)
\le b_{\bm\xi}(\mathbf w_k),
\quad \forall k\in[C]\bigr\},
\end{aligned}
\end{equation*}
which is a subset of the intersection of $\Sn$
with an affine hyperplane in $\R^n$.
For ideal prototypes, $w_{c,0}=1$, so the
corresponding spherical sections are great subspheres.}

\begin{proof}
We first consider the case of finite prototypes.
Throughout this part, fix a boundary
direction
\(
  \bm{\xi}=(1,\xhat)\in\partial\Hn,
 ~
  \xhat\in\Sn,
\)
with
\(
  \mathbf{x}=(\cosh r,\sinh r\,\xhat).
\)
We write
\begin{equation}
  \beta_c := b_{\bm{\xi}}(\mathbf w_c),
  \qquad c\in[C],
\end{equation}
for the Busemann value of the \(c\)-th prototype.

For each candidate label \(y\in[C]\), the PERM loss at the test point is
\begin{equation}
  \psi\!\left(\Upsilon_yD\mathbf v(\mathbf x)\right),
  \qquad
  v_c(\mathbf x)=-d_{\HH}(\mathbf x,\mathbf w_c).
\end{equation}
By the radial Busemann expansion~\eqref{eq:dist_bus_exp}, for each prototype \(\mathbf w_c\),
\begin{equation}
  d_{\HH}(\mathbf x,\mathbf w_c)
  =
  r+b_{\bm{\xi}}(\mathbf w_c)+O(e^{-2r})
  =
  r+\beta_c+O(e^{-2r}).
\end{equation}
By the definition of the Busemann loss,
\begin{equation}
  \mathcal L^{\mathrm{BS}}(\bm{\xi},y)
  =
  \psi\!\left(
  \Upsilon_yD
  \begin{pmatrix}
    -\beta_1\\
    \vdots\\
    -\beta_C
  \end{pmatrix}
  \right).
\end{equation}
Since \(\psi\) is smooth and the number of labels is finite, we have
\begin{equation}
  \psi\!\left(\Upsilon_yD\mathbf v(\mathbf x)\right)
  =
  \mathcal L^{\mathrm{BS}}(\bm{\xi},y)
  +
  O(e^{-2r}),
\end{equation}
uniformly over \(y\in[C]\). Thus, there exist constants \(C_0>0\) and
\(R_0>0\) such that, for all \(r\ge R_0\) and all \(y\in[C]\),
\begin{equation}\label{eq:uniform_boundary_loss_error}
  \left|
  \psi\!\left(\Upsilon_yD\mathbf v(\mathbf x)\right)
  -
  \mathcal L^{\mathrm{BS}}(\bm{\xi},y)
  \right|
  \le C_0e^{-2r}.
\end{equation}

Let
\begin{equation}
  y^*
  =
  \argmin_{y\in[C]}
  \mathcal L^{\mathrm{BS}}(\bm{\xi},y).
\end{equation}
By assumption, this minimizer is unique. Therefore, the gap between the best
label and the second best label is strictly positive:
\begin{equation}
  \Delta
  :=
  \min_{y\neq y^*}
  \left[
    \mathcal L^{\mathrm{BS}}(\bm{\xi},y)
    -
    \mathcal L^{\mathrm{BS}}(\bm{\xi},y^*)
  \right]
  >0.
\end{equation}
Choose \(R_1\ge R_0\) so large that
\begin{equation}
  2C_0e^{-2r}<\Delta
  \qquad
  \text{for all } r\ge R_1.
\end{equation}
Then, for any \(y\neq y^*\) and any \(r\ge R_1\),
\begin{equation}
\begin{aligned}
  &\psi\!\left(\Upsilon_yD\mathbf v(\mathbf x)\right)
  -
  \psi\!\left(\Upsilon_{y^*}D\mathbf v(\mathbf x)\right)
  \\
  &\quad \ge
  \mathcal L^{\mathrm{BS}}(\bm{\xi},y)
  -
  \mathcal L^{\mathrm{BS}}(\bm{\xi},y^*)
  -
  2C_0e^{-2r}
  \\
  &\quad \ge
  \Delta-2C_0e^{-2r}
  >0.
\end{aligned}
\end{equation}
Thus, for all sufficiently large \(r\),
\begin{equation}\label{eq:finite_perm_to_bs_prediction}
  \argmin_{y\in[C]}
  \psi\!\left(\Upsilon_yD\mathbf v(\mathbf x)\right)
  =
  \argmin_{y\in[C]}
  \mathcal L^{\mathrm{BS}}(\bm{\xi},y).
\end{equation}

We now prove that
\begin{equation}\label{eq:busemann_loss_minimum_beta}
  \argmin_{y\in[C]}
  \mathcal L^{\mathrm{BS}}(\bm{\xi},y)
  =
  \argmin_{c\in[C]}
  b_{\bm{\xi}}(\mathbf w_c).
\end{equation}

Recall that \(\beta_c=b_{\bm{\xi}}(\mathbf w_c)\). For a candidate label
\(y\), the vector of Busemann margins has the form
\begin{equation}
  \mathbf m^y
  =
  \bigl(\beta_c-\beta_y\bigr)_{c\neq y}
  \in \mathbb R^{C-1},
\end{equation}
up to a permutation of coordinates. Hence
\begin{equation}
  \mathcal L^{\mathrm{BS}}(\bm{\xi},y)
  =
  \psi(\mathbf m^y).
\end{equation}

Let
\begin{equation}
  y^*
  \in
  \argmin_{c\in[C]}\beta_c.
\end{equation}
We compare the loss of \(y^*\) with the loss of any other label \(y\neq y^*\).

If \(\beta_y>\beta_{y^*}\), then the margin vector for \(y^*\) is strictly
larger than the margin vector for \(y\), after a possible permutation of
coordinates. Indeed, for every \(c\neq y,y^*\),
\begin{equation}
  \beta_c-\beta_{y^*}
  >
  \beta_c-\beta_y,
\end{equation}
because \(\beta_y>\beta_{y^*}\). For the remaining coordinate pair, we have
\begin{equation}
  \beta_y-\beta_{y^*}
  >
  \beta_{y^*}-\beta_y.
\end{equation}
Thus \(\mathbf m^{y^*}\) is strictly larger than \(\mathbf m^y\)
componentwise, up to permutation.

Since \(\psi\) is a PERM template, it is invariant under permutations of its
coordinates. Since the loss is totally regular, \(\nabla\psi\prec0\), so
\(\psi\) is strictly decreasing in each coordinate. Therefore,
\begin{equation}
  \psi(\mathbf m^{y^*})<\psi(\mathbf m^y),
\end{equation}
or equivalently,
\begin{equation}
  \mathcal L^{\mathrm{BS}}(\bm{\xi},y^*)
  <
  \mathcal L^{\mathrm{BS}}(\bm{\xi},y).
\end{equation}
Hence any label \(y\) with \(\beta_y>\beta_{y^*}\) cannot minimize the
Busemann loss.

If instead \(\beta_y=\beta_{y^*}\), then the two margin vectors
\(\mathbf m^y\) and \(\mathbf m^{y^*}\) differ only by a permutation of
coordinates. By permutation invariance of \(\psi\),
\begin{equation}
  \psi(\mathbf m^y)=\psi(\mathbf m^{y^*}).
\end{equation}
Therefore the minimizers of the Busemann loss are exactly the labels whose
Busemann value \(\beta_y\) is minimal. This proves
\eqref{eq:busemann_loss_minimum_beta}. In particular, uniqueness of
\(\argmin_y\mathcal L^{\mathrm{BS}}(\bm{\xi},y)\) is equivalent to uniqueness
of \(\argmin_c b_{\bm{\xi}}(\mathbf w_c)\).

Combining \eqref{eq:finite_perm_to_bs_prediction} and
\eqref{eq:busemann_loss_minimum_beta}, we obtain, for all sufficiently large
\(r\),
\begin{equation}
  \hat y(\mathbf x)
  =
  \argmin_{y\in[C]}
  \mathcal L^{\mathrm{BS}}(\bm{\xi},y)
  =
  \argmin_{c\in[C]}
  b_{\bm{\xi}}(\mathbf w_c).
\end{equation}

Let
\begin{equation}
  c^*
  =
  \argmin_{c\in[C]}
  b_{\bm{\xi}}(\mathbf w_c).
\end{equation}
By uniqueness, the Busemann gap
\begin{equation}
  \Gamma
  :=
  \min_{c\neq c^*}
  \left[
    b_{\bm{\xi}}(\mathbf w_c)
    -
    b_{\bm{\xi}}(\mathbf w_{c^*})
  \right]
\end{equation}
is strictly positive.
Again by the radial Busemann expansion,
\(
  d_{\HH}(\mathbf x,\mathbf w_c)
  =
  r+b_{\bm{\xi}}(\mathbf w_c)+O(e^{-2r})
\)
for every \(c\in[C]\). Since there are only finitely many prototypes, the
error can be bounded uniformly over \(c\). Hence there exist constants
\(C_1>0\) and \(R_2>0\) such that, for all \(r\ge R_2\),
\begin{equation}
  \left|
  d_{\HH}(\mathbf x,\mathbf w_c)
  -
  r
  -
  b_{\bm{\xi}}(\mathbf w_c)
  \right|
  \le C_1e^{-2r}
  \qquad
  \text{for all } c\in[C].
\end{equation}
Choose \(R_3\ge R_2\) such that
\begin{equation}
  2C_1e^{-2r}<\Gamma
  \qquad
  \text{for all } r\ge R_3.
\end{equation}
Then, for any \(c\neq c^*\) and any \(r\ge R_3\),
\begin{equation}
\begin{aligned}
  d_{\HH}(\mathbf x,\mathbf w_c)
  -
  d_{\HH}(\mathbf x,\mathbf w_{c^*})
  &\ge
  b_{\bm{\xi}}(\mathbf w_c)
  -
  b_{\bm{\xi}}(\mathbf w_{c^*})
  -
  2C_1e^{-2r}
  \\
  &\ge
  \Gamma-2C_1e^{-2r}
  >0.
\end{aligned}
\end{equation}
Thus
\begin{equation}
  \argmin_{c\in[C]}
  d_{\HH}(\mathbf x,\mathbf w_c)
  =
  \argmin_{c\in[C]}
  b_{\bm{\xi}}(\mathbf w_c)
\end{equation}
for all sufficiently large \(r\). Therefore the asymptotic PERM prediction
also coincides with the nearest prototype prediction.

We next consider ideal prototypes
$\mathbf w_c=w_c^\infty\in\partial\Hn$.
Their class scores satisfy
\begin{equation}
v_c(\mathbf x)
=
-b_{w_c^\infty}(\mathbf x)
=
-\log\!\left(
\cosh r-\sinh r\,\xhat^\top\wbar_c
\right).
\end{equation}
For every $r>0$, the ordering of these scores
is determined exactly by
$\xhat^\top\wbar_c$. The same PERM comparison
argument above therefore gives
$\hat y(\mathbf x)=
\argmax_{c\in[C]}\xhat^\top\wbar_c$.
Moreover, the extended Busemann values satisfy
$b_{\bm\xi}(w_c^\infty)
=\log(1-\xhat^\top\wbar_c)$
away from the ideal prototype directions.
Their ordering is identical, and hence
\begin{equation}
\hat y(\mathbf x)
=
\argmin_{y\in[C]}
\mathcal L^{\mathrm{BS}}(\bm\xi,y)
=
\argmin_{c\in[C]}
b_{\bm\xi}(w_c^\infty),
\qquad r>0.
\end{equation}
The equality at ideal prototype directions
follows from the continuous extension of
the Busemann loss. Thus, unlike the finite
case, the classification rule is exact
for every positive radius.

Finally, we characterize the decision boundary.
For either prototype configuration, a boundary
direction separating classes $c$ and $c'$ must
satisfy
\begin{equation}
  b_{\bm{\xi}}(\mathbf w_c)
  =
  b_{\bm{\xi}}(\mathbf w_{c'}),
\end{equation}
with the common value minimal among all
classes. Using the explicit formula
\begin{equation}
  b_{\bm{\xi}}(\mathbf w_c)
  =
  \log\!\bigl(w_{c,0}-\xhat\cdot\wbar_c\bigr),
\end{equation}
the equality becomes
\begin{equation}
  \log\!\bigl(w_{c,0}-\xhat\cdot\wbar_c\bigr)
  =
  \log\!\bigl(w_{c',0}-\xhat\cdot\wbar_{c'}\bigr).
\end{equation}
The quantities inside the logarithms are
positive for finite prototypes and for
ideal prototypes away from their own
directions. Since distinct ideal prototypes
cannot be tied at either endpoint, the
logarithms are well defined at any
pairwise boundary. Their injectivity gives
\begin{equation}
  w_{c,0}-\xhat\cdot\wbar_c
  =
  w_{c',0}-\xhat\cdot\wbar_{c'}.
\end{equation}
Rearranging gives
\begin{equation}
  \xhat\cdot(\wbar_c-\wbar_{c'})
  =
  w_{c,0}-w_{c',0}.
\end{equation}
The additional minimum condition requires
$b_{\bm\xi}(\mathbf w_c)
\le b_{\bm\xi}(\mathbf w_k)$
for every $k\in[C]$. Therefore,
\begin{equation}
\begin{aligned}
  \mathcal B_{c,c'}
  =
  \bigl\{\xhat\in\Sn:\;&
  \xhat\cdot(\wbar_c-\wbar_{c'})
  =
  w_{c,0}-w_{c',0},~\
  b_{\bm\xi}(\mathbf w_c)
  \le b_{\bm\xi}(\mathbf w_k),
  \quad \forall k\in[C]\bigr\}.
\end{aligned}
\end{equation}
This is a subset of the intersection
of $\Sn$ with the affine hyperplane
\begin{equation}
  \left\{
  z\in\mathbb R^n:
  z\cdot(\wbar_c-\wbar_{c'})
  =
  w_{c,0}-w_{c',0}
  \right\}.
\end{equation}
For ideal prototypes, $w_{c,0}=w_{c',0}=1$,
so the affine hyperplane passes through
the origin and its spherical section
is a great subsphere.

Then we completed the proof.
\end{proof}

\newpage

\section{Experiments}\label{sec:experiments(appendix)}

In this appendix, we present the numerical setup,
integration procedures, ablations, and additional
diagnostics for the radial and angular dynamics.
We first integrate the radial and limiting Busemann
angular equations separately to examine their
asymptotic behavior. We then study the full coupled
Riemannian gradient flow, focusing on the evolution
of the drift coefficient $\mu$ and the three-phase
dynamics characterized by
Theorem~\ref{thm:three_phase}.
Finally, we report classification results from
neural network training.

\subsection{Setup}
\label{exp:setup}

Throughout the numerical experiments, we use the
hyperboloid model
\begin{equation}
    \mathbb{H}^n
    =
    \left\{
    \mathbf x\in\mathbb R^{n+1}:
    \langle\mathbf x,\mathbf x\rangle_M=-1,
    \ x_0>0
    \right\},
\end{equation}
with $n=2$ and constant sectional curvature $-1$.
We represent each point in polar coordinates
$\mathbf x=(\cosh r,\sinh r\,\xhat)$,
where $r\ge0$ and $\xhat\in\Sn$,
and visualize the trajectories on the Poincar\'e
disk through $\tanh(r/2)\xhat$.

For the canonical prototype configuration,
we use $C=3$ fixed prototypes on the geodesic
sphere of radius $\rho=1.0$, equally spaced
at angles $\phi_c=2\pi c/C$ for
$c\in\{0,1,2\}$. The prototypes remain
fixed throughout the numerical integration,
while the sample embeddings evolve according
to the corresponding gradient dynamics.

\paragraph{Dataset.}
The dataset consists of $N=15$ samples, with
five samples per class assigned by
$\mathrm{label}_i=i\bmod C$.
For each sample, the initial radius is drawn
from $r_0\sim\mathrm{Uniform}[1.0,2.0]$,
and the initial angle from
$\phi_0\sim\mathrm{Uniform}[0,2\pi)$,
with $\xhat_0=(\cos\phi_0,\sin\phi_0)$.
We use rejection sampling with angular
exclusion intervals of $\pm0.10$ radians
around the prototype directions.
All random initializations use
\texttt{np.random.default\_rng(seed=11)}.

The default loss is the PERM cross-entropy
$\psi_{\mathrm{CE}}(u)
=\log(1+\sum_j e^{-u_j})$,
where the relative margins are
$u_k=v_y-v_{j_k}$ with
$v_c=-d_{\Hn}(\mathbf x,\mathbf w_c)$
and $j_k\in[C]\setminus\{y\}$.

\subsection{Separate Radial and Angular Flows}
\label{exp:implementations}

We first examine the radial and angular
components of the large-radius dynamics
through separate numerical integration.
The radial trajectories characterize the
dependence of radial motion on the drift
coefficient $\mu$, while the angular
trajectories characterize optimization
of the Busemann loss on $\partial\Hn$.
The resulting behavior is visualized in
Fig.~\ref{fig:experiment_summary}.

\paragraph{Radial flow.}
We fix the angular direction
$\xhat_i(t)=\xhat_i(0)$ and integrate
$\dot r=-\partial_r\ell$.
Introducing the logarithmic time coordinate
$\tau=\log t$ allows us to examine the
asymptotic radial scaling
$r(t)=\frac12\log t+O(1)$ predicted
by Theorem~\ref{thm:radial}.

We sample $\tau\in[-2,18]$ at 300 points
using \texttt{scipy.integrate.solve\_ivp}
with LSODA, relative tolerance $10^{-7}$,
and absolute tolerance $10^{-9}$.
Integration terminates when $r\ge14$
or $r\le0.05$.
We partition the trajectories according
to the sign of the initial drift
$\mu(\bm\xi_i(0),y_i)$ and estimate
their late-time slopes using linear
regression over the final 30\% of the
integration interval.

\paragraph{Busemann angular flow.}
We integrate the limiting angular dynamics
on $\Sn$:
\begin{equation}
\frac{d\xhat}{ds}
=
-\grad_{\Sn}
\mathcal L^{\mathrm{BS}}(\bm\xi,y).
\end{equation}
The Busemann loss is evaluated as
$\mathcal L^{\mathrm{BS}}(\bm\xi,y)
=\psi((b_{\bm\xi}(\mathbf w_c)
-b_{\bm\xi}(\mathbf w_y))_{c\ne y})$,
and its spherical gradient is computed
analytically.

We sample $s\in[0,40]$ at 120 linearly
spaced points using
\texttt{scipy.integrate.solve\_ivp}
with RK45, relative tolerance $10^{-8}$,
absolute tolerance $10^{-10}$,
and maximum step size $0.5$.
The angular directions are normalized
during integration.
We evaluate convergence through the
limiting directions and the spherical
gradient norm
$\|\grad_{\Sn}\mathcal L^{\mathrm{BS}}\|$.

\paragraph{Numerical implementation.}
We evaluate the Minkowski inner product
using the identity
\begin{equation}
-\langle\mathbf x,\mathbf w\rangle_M
=
\frac{A}{2}e^r+\frac{B}{2}e^{-r},
\end{equation}
where $A=w_0-\xhat^\top\wbar$ and
$B=w_0+\xhat^\top\wbar$.
Hyperbolic distances are evaluated through
$\operatorname{arccosh}(
\max\{-\langle\mathbf x,\mathbf w\rangle_M,1\})$,
with analytic derivatives in the radial
and angular coordinates.
For the drift coefficient, we use
\begin{equation}
\alpha_{\bm\xi}(\mathbf w_c)
=
\frac{
\|\wbar_c\|^2-(\xhat^\top\wbar_c)^2
}{
(w_{c,0}-\xhat^\top\wbar_c)^2
}
=
\frac{\|\wbar_c^\perp\|^2}
{\langle\mathbf w_c,\bm\xi\rangle_M^2}.
\end{equation}
We evaluate the PERM losses and their
derivatives using the same numerical
expressions across prototype configurations.

\subsection{Coupled Gradient Flow and Three-Phase Dynamics}

We next integrate the full Riemannian gradient flow, allowing both
the radius and angular direction to evolve:
\begin{equation}
\dot r=-\frac{\partial\ell}{\partial r},
\qquad
\dot{\hat{\bm\xi}}
=-\frac{1}{\sinh^2 r}
\operatorname{grad}_{\Sn}\ell.
\end{equation}
This experiment examines the interaction between radial drift and
angular alignment under the dynamics of Theorem~6.4.

\paragraph{Initialization and evaluation.}
We consider prototype counts \(C\in\{2,3,5\}\) with equally spaced
prototypes at radius \(\rho=1\), and compare cross-entropy,
multiclass exponential, and sum-of-softplus PERM losses.
For the nondegenerate configurations, initial directions are selected
within the attraction basin of a Busemann local minimum
\(\bm\xi^*\), with \(\mu(\bm\xi(0),y)>0\).
We initialize trajectories at large radii and integrate the
coupled radial and angular equations simultaneously.

Along each trajectory, we record the radius \(r(t)\), drift
coefficient \(\mu(\bm\xi(t),y)\), and spherical distance
\(d_{\Sn}(\bm\xi(t),\bm\xi^*)\).
We also evaluate the radial velocity \(\dot r(t)\) and its
large-radius approximation
\(2\mu(\bm\xi(t),y)e^{-2r(t)}\).
Figure~\ref{fig:comparison} presents the trajectories across
prototype counts and PERM losses, with shaded regions indicating
standard deviations across seeds.

\paragraph{Sufficiently aligned regions.}
For nondegenerate configurations, we examine the sufficiently
aligned region defined in Definition~6.3:
\[
\mathcal A(\bm\xi^*)
=
\left\{
\bm\xi\in\partial\Hn:
d_{\Sn}(\bm\xi,\bm\xi^*)\le\theta_*
\right\},
\]
where \(\theta_*>0\) is chosen such that
\(\mu(\bm\xi,y)<0\) and the Busemann Hessian is
positive definite throughout the region.
We record the first entry time
\[
T_2=\inf\{t:\bm\xi(t)\in\mathcal A(\bm\xi^*)\},
\]
indicated by vertical dashed lines in
Figure~\ref{fig:comparison}.
The inset in the drift panel magnifies the near-zero behavior
of \(\mu\) for \(C=3\) and \(C=5\).

\paragraph{Three-phase evolution.}
The coupled trajectories exhibit three successive phases.
Initially, positive drift produces outward radial motion while
the directions evolve toward Busemann local minima.
As angular alignment proceeds, the drift decreases and
crosses zero, reversing the leading radial motion.
The trajectories subsequently enter the sufficiently aligned
regions, where negative drift induces slow radial return.

Figure~\ref{fig:comparison} illustrates these phases across
the nondegenerate configurations, while the degenerate case
exhibits weaker drift reversal.
The angular distances decrease toward zero across the tested
configurations, providing numerical evidence of the Busemann
directional bias.

\paragraph{Radial and angular coupling.}
The large-radius radial dynamics satisfy
\[
\dot r
=
2\mu(\bm\xi,y)e^{-2r}+O(e^{-4r}).
\]
The numerical trajectories illustrate how angular alignment
changes the sign of the leading radial term, connecting
outward excursions with subsequent slow return.
Together, these observations support the coupled radial and
angular behavior characterized by Theorem~6.4.

\subsection{Ablations}
\label{exp:ablations}

We examine the radial and Busemann
angular dynamics across three factors:
the number of prototypes, the PERM loss
template, and the geometric configuration
of the prototypes. The results are
presented column-wise in
Fig.~\ref{fig:experiment_summary}.

\paragraph{(i) Number of prototypes $C$.}
We consider $C\in\{2,3,5\}$, with five
samples per class, giving
$N\in\{10,15,25\}$.
The prototype radius is fixed at
$\rho=1.0$, and the loss is cross-entropy.
All configurations use seed 11.
This experiment examines the radial
dichotomy and the structure of angular
attractors as the number of classes
increases.

\paragraph{(ii) PERM loss template $\psi$.}
We compare three PERM loss templates:
\begin{enumerate}
    \item Cross-entropy:
    $\psi_{\mathrm{CE}}(u)
    =\log(1+\sum_j e^{-u_j})$.
    \item Multiclass exponential:
    $\psi_{\mathrm{EXP}}(u)
    =\sum_j e^{-u_j}$.
    \item Sum-of-softplus:
    $\psi_{\mathrm{SS}}(u)
    =\sum_j\log(1+e^{-u_j})$.
\end{enumerate}
We use numerically stable implementations
of these functions, including
\texttt{logsumexp} for cross-entropy.
Their gradients are checked against
\texttt{scipy.optimize.approx\_fprime},
with relative error below $10^{-6}$.
All three losses are evaluated on the
same canonical configuration
($C=3$, $\rho=1.0$, seed 11).

\paragraph{(iii) Prototype configuration.}
We consider three geometric layouts:
\begin{enumerate}
    \item Canonical: equal radii
    $\rho_c=1.0$, with
    $\phi_c=2\pi c/3$.
    \item Unequal radii:
    $\rho_c\in\{0.7,1.0,1.5\}$,
    with equally spaced angles.
    \item Clustered angles:
    equal radii $\rho_c=1.0$,
    with $\phi_c\in\{0,\pi/3,2\pi/3\}$.
\end{enumerate}
We use cross-entropy and seed 11
throughout these configurations.

\paragraph{Evaluation.}
For each panel of
Fig.~\ref{fig:experiment_summary},
we fit a linear regression to
$r_i(\tau)$ over the final 30\% of
the logarithmic time interval.
We report the late-time slope with
bootstrap confidence intervals.
The predicted positive-drift slope
is $1/2$.

We further group the radial trajectories
by $\mu(\bm\xi_i(0),y_i)>0$ and
$\mu(\bm\xi_i(0),y_i)\le0$,
corresponding to the two radial
behaviors. Angular convergence is
evaluated through the spherical
gradient norm and the limiting
Busemann directions.

Across the tested configurations,
the positive-drift radial trajectories
exhibit the predicted logarithmic
scaling, while the negative-drift
trajectories move inward.
The angular trajectories converge to
Busemann critical points, with
$\|\grad_{\Sn}\mathcal L^{\mathrm{BS}}\|
\le10^{-6}$ across the tested settings.
For the canonical configuration,
the limiting gradient norms reach
approximately $10^{-16}$.

\subsection{Additional Diagnostics}
\label{exp:add_diag}

We further examine the drift coefficient
at Busemann angular attractors, visualize
the trajectories on the Poincar\'e disk,
and characterize the angular attraction
basins. These diagnostics complement
the radial and angular results in
Fig.~\ref{fig:experiment_summary}
and the coupled three-phase analysis.

\paragraph{Drift coefficient at angular attractors.}
For each panel, we cluster the limiting
angular directions into bins of width
$0.05$ radians and evaluate the drift
coefficient at each cluster centroid:
\begin{equation*}
\mu(\bm\xi^*,y)
=
-\left\langle
\nabla\psi(\Upsilon_yD\mathbf b(\bm\xi^*)),
\Upsilon_yD\boldsymbol\alpha(\bm\xi^*)
\right\rangle.
\end{equation*}
In nine of the ten panels, all observed
angular attractors have negative drift.
For the canonical configuration
($C=3$, $\rho=1$, cross-entropy),
we obtain six attractors with
$\mu\in[-0.965,-0.059]$.
For $C=5$, the nine observed attractors
satisfy $\mu\in[-1.033,-0.303]$.

In the $C=2$ configuration, the
degenerate angular structure produces
$\mu\approx10^{-15}$ at the corresponding
antipodal directions.
The distribution of drift coefficients
illustrates the negative-drift structure
of nondegenerate Busemann attractors
and its connection to the radial return
mechanism.

\begin{figure}[h]
    \centering
    \includegraphics[width=\linewidth]{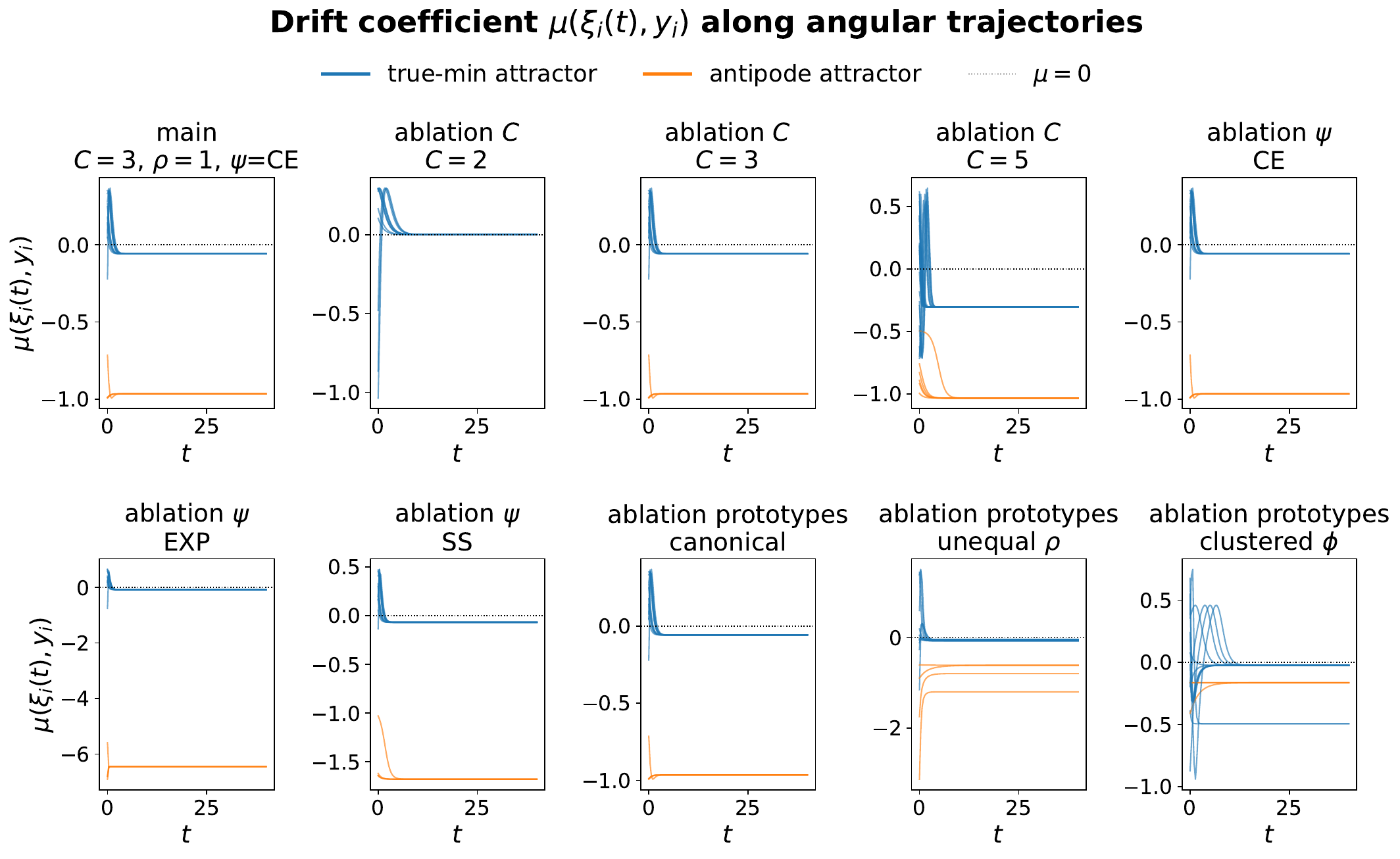}
    \caption{Evolution of the drift coefficient
    $\mu(\bm\xi(t),y)$ along the limiting
    Busemann angular trajectories.
    Blue: class-aligned attractors;
    orange: antipodal attractors;
    dashed line: $\mu=0$.
    The observed nondegenerate attractors
    have negative drift, illustrating
    the angular alignment mechanism
    underlying radial return.}
    \label{fig:app1_mu_traces}
\end{figure}

\paragraph{Poincar\'e disk visualizations.}
Fig.~\ref{fig:app2_disks} presents the
same ten configurations as
Fig.~\ref{fig:experiment_summary},
arranged in a $4\times5$ grid.
The upper two rows show the radial
trajectories, while the lower two rows
show the Busemann angular trajectories.

We use $\mathbf X$ to denote prototype
locations, $\circ$ for initial
directions, and $\star$ for the
limiting angular directions.
For the angular trajectories,
we apply a radial visualization
coordinate in $[0.55,0.97]$ to
display their evolution on $\mathbb S^1$
within the Poincar\'e disk.

\begin{figure}[h]
    \centering
    \includegraphics[width=\linewidth]{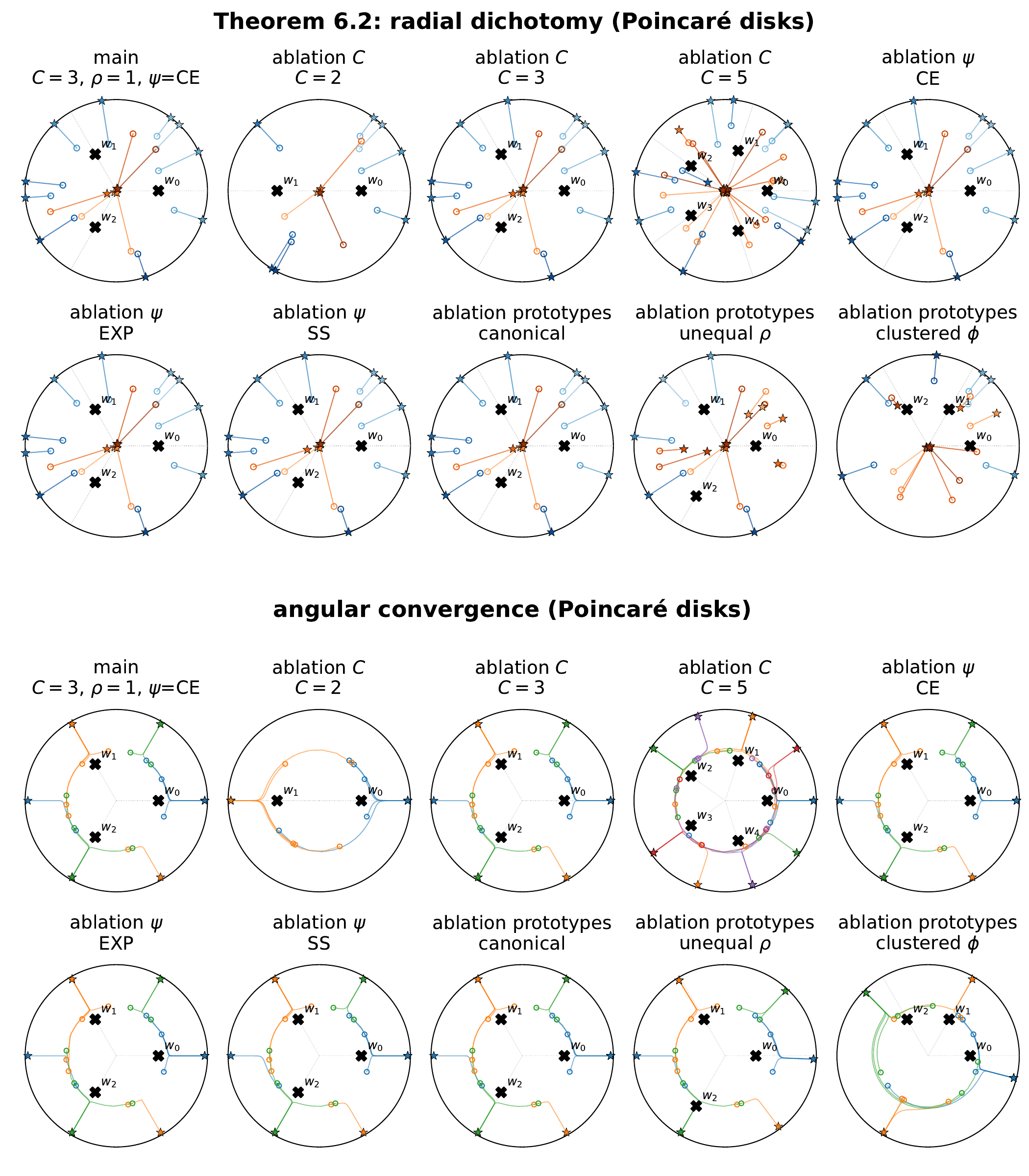}
    \caption{Poincar\'e disk visualization
    of the radial and Busemann angular
    trajectories across ten experimental
    configurations. Top two rows:
    fixed-direction radial dynamics.
    Bottom two rows: limiting angular
    dynamics, displayed with radial
    separation for visualization.
    Prototype locations, initial
    directions, and limiting directions
    are marked by $\mathbf X$,
    $\circ$, and $\star$, respectively.}
    \label{fig:app2_disks}
\end{figure}

\paragraph{Initialization basin map.}
We consider the canonical configuration
with $C=3$, $\rho=1.0$, cross-entropy,
and target class $y=0$.
We sample 200 equally spaced initial
directions $\theta_0\in[-\pi,\pi)$
and integrate the Busemann angular flow
up to $s=40$.

Each limiting direction is classified
as class-aligned when
$b_{\bm\xi^*}(\mathbf w_y)
<\min_{c\ne y}b_{\bm\xi^*}(\mathbf w_c)$,
and as antipodal otherwise.
We obtain 145 class-aligned trajectories
and 55 antipodal trajectories,
corresponding to 72.5\% and 27.5\%
of the initial directions.
The basin boundary occurs at
approximately
$|\theta_0|=132.5^\circ$,
giving the antipodal basin a width
of approximately $95^\circ$.

These results characterize how the
initial boundary direction determines
the limiting Busemann attractor and,
through $\mu$, influences the
subsequent radial dynamics.

\paragraph{Software, hardware, and reproducibility.}
The numerical integration experiments
use Python 3.9, NumPy 2.0, SciPy 1.13,
and matplotlib 3.9 on a CPU.
The original separated-flow figures,
tables, and CSV outputs are generated
in approximately 30 seconds.
The implementation uses fixed random
seeds and analytic gradients for
the hyperbolic dynamics.
The experimental code is anonymously
released for review.

\subsection{Neural Network Training}
\label{exp:network}

We further investigate Busemann asymptotic behavior
under shared neural network parameterization.
We first compare stationary solutions of the exact
hyperbolic loss and Busemann risk using a shared
angular encoder, then examine their gradients
on ResNet-18 checkpoints trained on CIFAR-10.

\paragraph{Shared-encoder stationarity.}
We consider a three-parameter shared linear
angular encoder with $900$ samples, cross-entropy
loss, and fixed prototypes. We use canonical and
irregular prototype configurations, five independent
seeds, and seven radii
$R\in\{3,3.5,4,4.5,5,5.5,6\}$,
yielding $70$ stationary-point checkpoints.

From each exact hyperbolic stationary point
$\theta_R^*$, we recover a nearby Busemann
stationary point $\theta_{\mathrm{BS}}^*$ by solving
$\nabla_\theta\mathcal R_{\mathbf W}^{\mathrm{BS}}
(\bm\Xi_\theta)=0$
using full-batch float64 L-BFGS and
step-restricted Newton iterations.
All $70$ recovered points have gradient
residuals below $1.11\times10^{-13}$.

We measure the angular proximity between
their representations:
\begin{equation}
D_R=
\left[
\frac1N\sum_{i=1}^N
d_{\Sn}\left(
\hat{\bm\xi}_i(\theta_R^*),
\hat{\bm\xi}_i(\theta_{\mathrm{BS}}^*)
\right)^2
\right]^{1/2}.
\end{equation}
Motivated by the large-radius expansion,
we examine the scaling $D_R=O(e^{-2R})$.
We fit the slope of $\log D_R$ against $R$
for each seed and compute $95\%$ confidence
intervals from the five seed-level slopes.

\begin{table}[t]
\centering
\small
\begin{tabular}{llcc}
\toprule
Prototypes & Checkpoints &
Slope of $\log D_R$ & 95\% CI \\
\midrule
Canonical & All
& $-1.989442$ & $[-2.017487,-1.961398]$ \\
Irregular & All
& $-2.000482$ & $[-2.001438,-1.999526]$ \\
Canonical & Converged
& $-2.001996$ & $[-2.007563,-1.996428]$ \\
Irregular & Converged
& $-2.000017$ & $[-2.000328,-1.999706]$ \\
\bottomrule
\end{tabular}
\caption{Angular proximity between stationary
representations of the exact hyperbolic loss
and Busemann risk. The converged subset satisfies
$\|\nabla_\theta\mathcal L_{\mathbf W}\|<10^{-8}$.
The observed slopes closely follow $e^{-2R}$ scaling.}
\label{tab:shared-critical-proximity}
\end{table}

Table~\ref{tab:shared-critical-proximity}
shows slopes close to $-2$ in both configurations.
The corresponding slopes of
$\log\|\nabla_\theta
\mathcal R_{\mathbf W}^{\mathrm{BS}}\|$
are $-1.990407$ and $-2.000416$
for all checkpoints, and $-2.001859$
and $-2.000007$ for converged subsets.

We also examine the local first-order
stationary-point displacement. The recovered
Busemann Hessians have smallest absolute
eigenvalues between $0.424545$ and $0.558863$,
with maximum condition number $2.13155$.
Median relative errors of the first-order
parameter and angular displacement approximations
are $0.3374\%$ and $0.3375\%$, respectively.
The recovered Busemann critical parameters
agree across radii within $1.72\times10^{-13}$.
These observations provide additional numerical
evidence for $O(e^{-2R})$ stationary-point
proximity under shared parameterization.

\paragraph{ResNet-18 training.}
We additionally evaluate saved ResNet-18
checkpoints trained on CIFAR-10 with freely
varying hyperbolic embedding radii.
We compute exact hyperbolic and Busemann
risk parameter gradients on a fixed
unaugmented $1024$-sample training probe
using float64 arithmetic.

\begin{table}[t]
\centering
\small
\begin{tabular}{ccccc}
\toprule
Seed & $R_\theta$ & Mean radius &
$\|\nabla_\theta\mathcal L_{\mathbf W}\|$ &
$\|\nabla_\theta\mathcal R_{\mathbf W}^{\mathrm{BS}}\|$ \\
\midrule
0 & $3.0677$ & $3.1514$ & $0.096812$ & $0.095768$ \\
1 & $3.0591$ & $3.1493$ & $0.435698$ & $0.429801$ \\
2 & $3.0258$ & $3.0369$ & $0.461785$ & $0.457161$ \\
\bottomrule
\end{tabular}
\caption{Exact hyperbolic and Busemann parameter
gradient norms at CIFAR-10 ResNet-18 checkpoints,
evaluated on a fixed $1024$-sample training probe.}
\label{tab:resnet-probe}
\end{table}

Table~\ref{tab:resnet-probe} shows closely
matched gradient norms across all three seeds.
Together with the shared-encoder results,
these observations provide complementary
numerical evidence for Busemann asymptotic
behavior under neural network parameterization.

\section{Societal Impact}
\label{sec:social}
This work is primarily theoretical and studies the implicit bias of
Riemannian gradient flow in hyperbolic representation learning. It does not
collect user data, involve human subjects, release pretrained models, or
deploy a system. The potential positive impact is indirect: a better
understanding of boundary saturation, near boundary clustering, and
asymptotic decision boundaries may improve the stability, interpretability,
and reliability of hyperbolic models for structured or hierarchical data.
Potential negative impacts are also indirect. More effective representation
learning methods could be used in sensitive domains such as recommendation,
ranking, social network analysis, or decision support, where insufficient
validation, fairness evaluation, or privacy safeguards may lead to harmful
outcomes. Since our experiments are controlled synthetic simulations and our
contribution is theoretical, we do not identify direct risks beyond the
broader responsible use concerns associated with representation learning.

\newpage

\newpage
\section*{NeurIPS Paper Checklist}

\begin{enumerate}

\item {\bf Claims}
    \item[] Question: Do the main claims made in the abstract and introduction accurately reflect the paper's contributions and scope?
    \item[] Answer: \answerYes{} 
    \item[] Justification: Yes, we claim it clearly.
    \item[] Guidelines:
    \begin{itemize}
        \item The answer \answerNA{} means that the abstract and introduction do not include the claims made in the paper.
        \item The abstract and/or introduction should clearly state the claims made, including the contributions made in the paper and important assumptions and limitations. A \answerNo{} or \answerNA{} answer to this question will not be perceived well by the reviewers. 
        \item The claims made should match theoretical and experimental results, and reflect how much the results can be expected to generalize to other settings. 
        \item It is fine to include aspirational goals as motivation as long as it is clear that these goals are not attained by the paper. 
    \end{itemize}

\item {\bf Limitations}
    \item[] Question: Does the paper discuss the limitations of the work performed by the authors?
    \item[] Answer:\answerYes{}
    \item[] Justification: We discuss our limitations in Appendix~[\ref{sec:discussion}].
    \item[] Guidelines:
    \begin{itemize}
        \item The answer \answerNA{} means that the paper has no limitation while the answer \answerNo{} means that the paper has limitations, but those are not discussed in the paper. 
        \item The authors are encouraged to create a separate ``Limitations'' section in their paper.
        \item The paper should point out any strong assumptions and how robust the results are to violations of these assumptions (e.g., independence assumptions, noiseless settings, model well-specification, asymptotic approximations only holding locally). The authors should reflect on how these assumptions might be violated in practice and what the implications would be.
        \item The authors should reflect on the scope of the claims made, e.g., if the approach was only tested on a few datasets or with a few runs. In general, empirical results often depend on implicit assumptions, which should be articulated.
        \item The authors should reflect on the factors that influence the performance of the approach. For example, a facial recognition algorithm may perform poorly when image resolution is low or images are taken in low lighting. Or a speech-to-text system might not be used reliably to provide closed captions for online lectures because it fails to handle technical jargon.
        \item The authors should discuss the computational efficiency of the proposed algorithms and how they scale with dataset size.
        \item If applicable, the authors should discuss possible limitations of their approach to address problems of privacy and fairness.
        \item While the authors might fear that complete honesty about limitations might be used by reviewers as grounds for rejection, a worse outcome might be that reviewers discover limitations that aren't acknowledged in the paper. The authors should use their best judgment and recognize that individual actions in favor of transparency play an important role in developing norms that preserve the integrity of the community. Reviewers will be specifically instructed to not penalize honesty concerning limitations.
    \end{itemize}

\item {\bf Theory assumptions and proofs}
    \item[] Question: For each theoretical result, does the paper provide the full set of assumptions and a complete (and correct) proof?
    \item[] Answer: \answerYes{}
    \item[] Justification: We provide all the detailed proofs in Appendix~[\ref{proof:4.1}],[\ref{proof:4.3}],[\ref{proof: 5.2}],[\ref{proof: 5.6}], and [\ref{proof: 5.7}].
    \item[] Guidelines:
    \begin{itemize}
        \item The answer \answerNA{} means that the paper does not include theoretical results. 
        \item All the theorems, formulas, and proofs in the paper should be numbered and cross-referenced.
        \item All assumptions should be clearly stated or referenced in the statement of any theorems.
        \item The proofs can either appear in the main paper or the supplemental material, but if they appear in the supplemental material, the authors are encouraged to provide a short proof sketch to provide intuition. 
        \item Inversely, any informal proof provided in the core of the paper should be complemented by formal proofs provided in appendix or supplemental material.
        \item Theorems and Lemmas that the proof relies upon should be properly referenced. 
    \end{itemize}

    \item {\bf Experimental result reproducibility}
    \item[] Question: Does the paper fully disclose all the information needed to reproduce the main experimental results of the paper to the extent that it affects the main claims and/or conclusions of the paper (regardless of whether the code and data are provided or not)?
    \item[] Answer: \answerYes{}
    \item[] Justification: We provide all the details of experiments in Appendix~[\ref{sec:experiments(appendix)}].
    \item[] Guidelines:
    \begin{itemize}
        \item The answer \answerNA{} means that the paper does not include experiments.
        \item If the paper includes experiments, a \answerNo{} answer to this question will not be perceived well by the reviewers: Making the paper reproducible is important, regardless of whether the code and data are provided or not.
        \item If the contribution is a dataset and\slash or model, the authors should describe the steps taken to make their results reproducible or verifiable. 
        \item Depending on the contribution, reproducibility can be accomplished in various ways. For example, if the contribution is a novel architecture, describing the architecture fully might suffice, or if the contribution is a specific model and empirical evaluation, it may be necessary to either make it possible for others to replicate the model with the same dataset, or provide access to the model. In general. releasing code and data is often one good way to accomplish this, but reproducibility can also be provided via detailed instructions for how to replicate the results, access to a hosted model (e.g., in the case of a large language model), releasing of a model checkpoint, or other means that are appropriate to the research performed.
        \item While NeurIPS does not require releasing code, the conference does require all submissions to provide some reasonable avenue for reproducibility, which may depend on the nature of the contribution. For example
        \begin{enumerate}
            \item If the contribution is primarily a new algorithm, the paper should make it clear how to reproduce that algorithm.
            \item If the contribution is primarily a new model architecture, the paper should describe the architecture clearly and fully.
            \item If the contribution is a new model (e.g., a large language model), then there should either be a way to access this model for reproducing the results or a way to reproduce the model (e.g., with an open-source dataset or instructions for how to construct the dataset).
            \item We recognize that reproducibility may be tricky in some cases, in which case authors are welcome to describe the particular way they provide for reproducibility. In the case of closed-source models, it may be that access to the model is limited in some way (e.g., to registered users), but it should be possible for other researchers to have some path to reproducing or verifying the results.
        \end{enumerate}
    \end{itemize}

\item {\bf Open access to data and code}
    \item[] Question: Does the paper provide open access to the data and code, with sufficient instructions to faithfully reproduce the main experimental results, as described in supplemental material?
    \item[] Answer: \answerYes{}
    \item[] Justification: We provide them in the supplementary materials.
    \item[] Guidelines:
    \begin{itemize}
        \item The answer \answerNA{} means that paper does not include experiments requiring code.
        \item Please see the NeurIPS code and data submission guidelines (\url{https://neurips.cc/public/guides/CodeSubmissionPolicy}) for more details.
        \item While we encourage the release of code and data, we understand that this might not be possible, so \answerNo{} is an acceptable answer. Papers cannot be rejected simply for not including code, unless this is central to the contribution (e.g., for a new open-source benchmark).
        \item The instructions should contain the exact command and environment needed to run to reproduce the results. See the NeurIPS code and data submission guidelines (\url{https://neurips.cc/public/guides/CodeSubmissionPolicy}) for more details.
        \item The authors should provide instructions on data access and preparation, including how to access the raw data, preprocessed data, intermediate data, and generated data, etc.
        \item The authors should provide scripts to reproduce all experimental results for the new proposed method and baselines. If only a subset of experiments are reproducible, they should state which ones are omitted from the script and why.
        \item At submission time, to preserve anonymity, the authors should release anonymized versions (if applicable).
        \item Providing as much information as possible in supplemental material (appended to the paper) is recommended, but including URLs to data and code is permitted.
    \end{itemize}

\item {\bf Experimental setting/details}
    \item[] Question: Does the paper specify all the training and test details (e.g., data splits, hyperparameters, how they were chosen, type of optimizer) necessary to understand the results?
    \item[] Answer: \answerYes{}
    \item[] Justification: We provide all the details of experiments in Appendix~[\ref{sec:experiments(appendix)}].
    \item[] Guidelines:
    \begin{itemize}
        \item The answer \answerNA{} means that the paper does not include experiments.
        \item The experimental setting should be presented in the core of the paper to a level of detail that is necessary to appreciate the results and make sense of them.
        \item The full details can be provided either with the code, in appendix, or as supplemental material.
    \end{itemize}

\item {\bf Experiment statistical significance}
    \item[] Question: Does the paper report error bars suitably and correctly defined or other appropriate information about the statistical significance of the experiments?
    \item[] Answer: \answerYes{}
    \item[] Justification: 
We report inter-seed standard deviations for coupled gradient flow
trajectories and 95\% confidence intervals for neural network
stationary point scaling estimates. These statistics quantify the
variability of our numerical verification of theoretical predictions,
rather than statistical superiority over baselines.

    \item[] Guidelines:
    \begin{itemize}
        \item The answer \answerNA{} means that the paper does not include experiments.
        \item The authors should answer \answerYes{} if the results are accompanied by error bars, confidence intervals, or statistical significance tests, at least for the experiments that support the main claims of the paper.
        \item The factors of variability that the error bars are capturing should be clearly stated (for example, train/test split, initialization, random drawing of some parameter, or overall run with given experimental conditions).
        \item The method for calculating the error bars should be explained (closed form formula, call to a library function, bootstrap, etc.)
        \item The assumptions made should be given (e.g., Normally distributed errors).
        \item It should be clear whether the error bar is the standard deviation or the standard error of the mean.
        \item It is OK to report 1-sigma error bars, but one should state it. The authors should preferably report a 2-sigma error bar than state that they have a 96\% CI, if the hypothesis of Normality of errors is not verified.
        \item For asymmetric distributions, the authors should be careful not to show in tables or figures symmetric error bars that would yield results that are out of range (e.g., negative error rates).
        \item If error bars are reported in tables or plots, the authors should explain in the text how they were calculated and reference the corresponding figures or tables in the text.
    \end{itemize}

\item {\bf Experiments compute resources}
    \item[] Question: For each experiment, does the paper provide sufficient information on the computer resources (type of compute workers, memory, time of execution) needed to reproduce the experiments?
    \item[] Answer: \answerYes{}
    \item[] Justification: We provide all the details of experiments in Appendix~[\ref{sec:experiments(appendix)}].
    \item[] Guidelines:
    \begin{itemize}
        \item The answer \answerNA{} means that the paper does not include experiments.
        \item The paper should indicate the type of compute workers CPU or GPU, internal cluster, or cloud provider, including relevant memory and storage.
        \item The paper should provide the amount of compute required for each of the individual experimental runs as well as estimate the total compute. 
        \item The paper should disclose whether the full research project required more compute than the experiments reported in the paper (e.g., preliminary or failed experiments that didn't make it into the paper). 
    \end{itemize}
    
\item {\bf Code of ethics}
    \item[] Question: Does the research conducted in the paper conform, in every respect, with the NeurIPS Code of Ethics \url{https://neurips.cc/public/EthicsGuidelines}?
    \item[] Answer: \answerYes{}
    \item[] Justification: Yes. Our research is primarily a
theoretical analysis of optimization dynamics in hyperbolic classification,
supported by controlled synthetic experiments. 
    \item[] Guidelines:
    \begin{itemize}
        \item The answer \answerNA{} means that the authors have not reviewed the NeurIPS Code of Ethics.
        \item If the authors answer \answerNo, they should explain the special circumstances that require a deviation from the Code of Ethics.
        \item The authors should make sure to preserve anonymity (e.g., if there is a special consideration due to laws or regulations in their jurisdiction).
    \end{itemize}

\item {\bf Broader impacts}
    \item[] Question: Does the paper discuss both potential positive societal impacts and negative societal impacts of the work performed?
    \item[] Answer: \answerYes{}
    \item[] Justification: We discuss our impacts in Appendix~[\ref{sec:social}]
    \item[] Guidelines:
    \begin{itemize}
        \item The answer \answerNA{} means that there is no societal impact of the work performed.
        \item If the authors answer \answerNA{} or \answerNo, they should explain why their work has no societal impact or why the paper does not address societal impact.
        \item Examples of negative societal impacts include potential malicious or unintended uses (e.g., disinformation, generating fake profiles, surveillance), fairness considerations (e.g., deployment of technologies that could make decisions that unfairly impact specific groups), privacy considerations, and security considerations.
        \item The conference expects that many papers will be foundational research and not tied to particular applications, let alone deployments. However, if there is a direct path to any negative applications, the authors should point it out. For example, it is legitimate to point out that an improvement in the quality of generative models could be used to generate Deepfakes for disinformation. On the other hand, it is not needed to point out that a generic algorithm for optimizing neural networks could enable people to train models that generate Deepfakes faster.
        \item The authors should consider possible harms that could arise when the technology is being used as intended and functioning correctly, harms that could arise when the technology is being used as intended but gives incorrect results, and harms following from (intentional or unintentional) misuse of the technology.
        \item If there are negative societal impacts, the authors could also discuss possible mitigation strategies (e.g., gated release of models, providing defenses in addition to attacks, mechanisms for monitoring misuse, mechanisms to monitor how a system learns from feedback over time, improving the efficiency and accessibility of ML).
    \end{itemize}
    
\item {\bf Safeguards}
    \item[] Question: Does the paper describe safeguards that have been put in place for responsible release of data or models that have a high risk for misuse (e.g., pre-trained language models, image generators, or scraped datasets)?
    \item[] Answer: \answerNA{}
    \item[] Justification: The experiments use only
controlled synthetic data for numerical verification of theoretical results,
and no sensitive, private, or user collected data are involved.
    \item[] Guidelines:
    \begin{itemize}
        \item The answer \answerNA{} means that the paper poses no such risks.
        \item Released models that have a high risk for misuse or dual-use should be released with necessary safeguards to allow for controlled use of the model, for example by requiring that users adhere to usage guidelines or restrictions to access the model or implementing safety filters. 
        \item Datasets that have been scraped from the Internet could pose safety risks. The authors should describe how they avoided releasing unsafe images.
        \item We recognize that providing effective safeguards is challenging, and many papers do not require this, but we encourage authors to take this into account and make a best faith effort.
    \end{itemize}

\item {\bf Licenses for existing assets}
    \item[] Question: Are the creators or original owners of assets (e.g., code, data, models), used in the paper, properly credited and are the license and terms of use explicitly mentioned and properly respected?
    \item[] Answer: \answerYes{}
    \item[] Justification:  We properly cites prior work and theoretical frameworks used in
the analysis. The numerical experiments use synthetic data generated by the
authors and do not use external datasets, pretrained models, or third party
assets beyond standard open source scientific computing libraries.
    \item[] Guidelines:
    \begin{itemize}
        \item The answer \answerNA{} means that the paper does not use existing assets.
        \item The authors should cite the original paper that produced the code package or dataset.
        \item The authors should state which version of the asset is used and, if possible, include a URL.
        \item The name of the license (e.g., CC-BY 4.0) should be included for each asset.
        \item For scraped data from a particular source (e.g., website), the copyright and terms of service of that source should be provided.
        \item If assets are released, the license, copyright information, and terms of use in the package should be provided. For popular datasets, \url{paperswithcode.com/datasets} has curated licenses for some datasets. Their licensing guide can help determine the license of a dataset.
        \item For existing datasets that are re-packaged, both the original license and the license of the derived asset (if it has changed) should be provided.
        \item If this information is not available online, the authors are encouraged to reach out to the asset's creators.
    \end{itemize}

\item {\bf New assets}
    \item[] Question: Are new assets introduced in the paper well documented and is the documentation provided alongside the assets?
    \item[] Answer: \answerNA{}
    \item[] Justification: The paper does not release new datasets, pretrained models, or other
standalone assets. 
    \item[] Guidelines:
    \begin{itemize}
        \item The answer \answerNA{} means that the paper does not release new assets.
        \item Researchers should communicate the details of the dataset\slash code\slash model as part of their submissions via structured templates. This includes details about training, license, limitations, etc. 
        \item The paper should discuss whether and how consent was obtained from people whose asset is used.
        \item At submission time, remember to anonymize your assets (if applicable). You can either create an anonymized URL or include an anonymized zip file.
    \end{itemize}

\item {\bf Crowdsourcing and research with human subjects}
    \item[] Question: For crowdsourcing experiments and research with human subjects, does the paper include the full text of instructions given to participants and screenshots, if applicable, as well as details about compensation (if any)? 
    \item[] Answer: \answerNA{}
    \item[] Justification: The paper does not involve crowdsourcing experiments, human subjects,
participant instructions, screenshots for user studies, or participant
compensation. The experiments are controlled synthetic numerical simulations.
    \item[] Guidelines:
    \begin{itemize}
        \item The answer \answerNA{} means that the paper does not involve crowdsourcing nor research with human subjects.
        \item Including this information in the supplemental material is fine, but if the main contribution of the paper involves human subjects, then as much detail as possible should be included in the main paper. 
        \item According to the NeurIPS Code of Ethics, workers involved in data collection, curation, or other labor should be paid at least the minimum wage in the country of the data collector. 
    \end{itemize}

\item {\bf Institutional review board (IRB) approvals or equivalent for research with human subjects}
    \item[] Question: Does the paper describe potential risks incurred by study participants, whether such risks were disclosed to the subjects, and whether Institutional Review Board (IRB) approvals (or an equivalent approval/review based on the requirements of your country or institution) were obtained?
    \item[] Answer:  \answerNA{}
    \item[] Justification: The paper does not involve human subjects, user studies,
or participant data. Therefore, there are no study participant risks to
disclose and no IRB or equivalent human subjects approval is required.
    \item[] Guidelines:
    \begin{itemize}
        \item The answer \answerNA{} means that the paper does not involve crowdsourcing nor research with human subjects.
        \item Depending on the country in which research is conducted, IRB approval (or equivalent) may be required for any human subjects research. If you obtained IRB approval, you should clearly state this in the paper. 
        \item We recognize that the procedures for this may vary significantly between institutions and locations, and we expect authors to adhere to the NeurIPS Code of Ethics and the guidelines for their institution. 
        \item For initial submissions, do not include any information that would break anonymity (if applicable), such as the institution conducting the review.
    \end{itemize}

\item {\bf Declaration of LLM usage}
    \item[] Question: Does the paper describe the usage of LLMs if it is an important, original, or non-standard component of the core methods in this research? Note that if the LLM is used only for writing, editing, or formatting purposes and does \emph{not} impact the core methodology, scientific rigor, or originality of the research, declaration is not required.
    \item[] Answer:\answerNA{}
    \item[] Justification: The research does not use LLMs as part of the core methodology. Any LLM usage
was limited to writing, editing, and experimental assistance.
    \item[] Guidelines:
    \begin{itemize}
        \item The answer \answerNA{} means that the core method development in this research does not involve LLMs as any important, original, or non-standard components.
        \item Please refer to our LLM policy in the NeurIPS handbook for what should or should not be described.
    \end{itemize}

\end{enumerate}

\end{document}